\documentclass[onecolumn,conference,A4]{IEEEtran}
\usepackage[left=1in,right=1in,top=1in,bottom=1in]{geometry}
\usepackage{moreverb,url}
\usepackage{times} 
\usepackage{amsmath,amsthm} 

\usepackage{amssymb}  
\usepackage{relsize}
\usepackage{cite}
\usepackage{latexsym,amsfonts,mathtools}
\usepackage[colorlinks,bookmarksopen,bookmarksnumbered,citecolor=black,urlcolor=black]{hyperref}
\usepackage[inline]{enumitem}
\usepackage[ruled,linesnumbered]{algorithm2e}
\usepackage{subcaption}
\usepackage{stmaryrd}
\usepackage{lscape}
\usepackage{rotating}
\SetSymbolFont{stmry}{bold}{U}{stmry}{m}{n}
\usepackage{algorithmic}
\usepackage{booktabs}
\usepackage{epstopdf}
\usepackage{hyperref}

\theoremstyle{plain}
\newtheorem{theorem}{Theorem}
\newtheorem{lemma}{Lemma}

\theoremstyle{remark}
\newtheorem{remark}{Remark}

\theoremstyle{definition}
\newtheorem{definition}{Definition}
\newtheorem*{problem*}{Problem Statement}

\newtheorem*{objective*}{Objective}

\newcommand{\vect}[1]{\boldsymbol{#1}}
\newcommand{\dsqrbrac}[1]{\llbracket{#1}\rrbracket}
\newcommand{\vectnot}[3]{\vect{#1}_{#2}^{(#3)}}
\newcommand{\scalnot}[3]{{#1}_{#2}^{(#3)}}
\newcommand{\weight}[2]{{w}_{#1}^{(#2)}}

\newcommand{\particle}[2]{{\Theta}_{#1}^{(#2)}}
\newcommand{\tb}[1]{\textcolor{black}{#1}}

\newcommand\BibTeX{{\rmfamily B\kern-.05em \textsc{i\kern-.025em b}\kern-.08em
T\kern-.1667em\lower.7ex\hbox{E}\kern-.125emX}}

\begin{document}

\title{Mr.MSTE: Multi-robot Multi-Source Term Estimation with Wind-Aware Coverage Control}


\author{
    \IEEEauthorblockN{Rohit V. Nanavati$^{1,2}$, Tim J. Glover$^{2,*}$, Matthew J. Coombes$^{2}$, and Cunjia Liu$^{2}$}
    \IEEEauthorblockA{$^{1}$Department of Aerospace Engineering, Indian Institute of Technology Bombay, Powai, Mumbai 400076 India} 
    \IEEEauthorblockA{$^{2}$Aeronautical and Automotive Engineering Department, Loughborough University, Loughborough, LE113TU, UK, }
    \IEEEauthorblockA{${*}$ now at Defence Science and Technology Laboratory (Dstl), UK}
}

\maketitle


\begin{abstract}
This paper presents a Multi-Robot Multi-Source Term Estimation (MRMSTE) framework that enables teams of mobile robots to collaboratively sample gas concentrations and infer the parameters of an unknown number of airborne releases. The framework is built on a hybrid Bayesian inference scheme that represents the joint multi-source probability density and incorporates physics-informed state transitions, including source birth, removal, and merging induced by atmospheric dispersion. A superposition-based measurement model is naturally accommodated, allowing sparse concentration measurements to be exploited efficiently. To guide robot deployment, we introduce a wind-aware coverage control (WCC) strategy that integrates the evolving multi-source belief with local wind information to prioritize regions of high detection likelihood. Unlike conventional coverage control or information-theoretic planners, WCC explicitly accounts for anisotropic plume transport when modelling sensor performance, leading to more effective sensor placement for multi-source estimation. Monte Carlo studies demonstrate faster convergence and improved separation of individual source beliefs compared to traditional coverage-based strategies and small-scale static sensor networks. Real-world experiments with $\mathrm{CO}_2$ releases using TurtleBot platforms further validate the proposed approach, demonstrating its practicality for scalable multi-robot gas-sensing applications.
\end{abstract}


\section{Introduction}   
    Emergency response to accidental or deliberate releases of chemical, biological, radiological, and nuclear (CBRN) contaminants into the atmosphere requires efficient characterization of atmospheric dispersion by estimating the key source-term parameters, such as release location, emission rate, and transport properties \cite{HUTCHINSON2017130,BURGUES2020141172}. Traditional approaches to source term estimation (STE) rely on deploying arrays of static sensors within the area of interest to measure contaminant concentrations and infer the underlying release parameters. However, such solutions are only practical when potential incidents are predictable and sensor infrastructure can be pre-deployed, such as at oil and gas sites, chemical warehouses, and nuclear facilities. To overcome these limitations, there has been growing interest in leveraging mobile robotics for source localization and estimation \cite{Francis2022}, offering significant advantages in scalability and flexibility. Mobile robots are particularly well suited to complex and hazardous scenarios, such as localizing unknown release sources in cluttered environments or responding to rapidly evolving dispersion processes in large-scale urban settings.

    Mobile sensing platforms require advanced path planning functions to deploy or guide robots in the search domain. Many studies focus on intelligent search algorithms for a single robot to perform source term estimation or directly seeking the source. Early approaches rely on concentration gradient-based techniques, commonly referred to as chemotaxis, which mimic biological responses to chemical signals \cite{Azuma2012,WeiLi2006}. Alternatively, information-theoretic approaches—often referred to as infotaxis or cognitive search—select sampling locations to maximise the expected reduction in uncertainty of the estimated source term \cite{RISTIC20161,voges2014reactive,hutchinson2018information,Park2021,Callum2022,rhodes2023autonomous}. The reduction of uncertainty is normally quantified using information measures (e.g., entropy, mutual information) within a sequential Bayesian estimation framework, which processes the noisy sensor measurements to establish the belief of the source term. Compared to chemotaxis-based methods, information-theoretic algorithms exhibit superior performance, particularly in turbulent environments where concentration gradients of CBRN particulates are erratic, and sensor measurements are intermittent \cite{martin2010effectiveness}. This trend is further supported by recent experimental validation in robotic source search studies such as \cite{hutchinson2018information,Park2021,Callum2022,Ercolani2022}. However, these methods are predominantly designed for single-source scenarios and do not naturally extend to multi-source settings, where inference is complicated by measurement superposition and ambiguity of the source number.

    \begin{figure}
        \centering
        \includegraphics[width=0.5\linewidth]{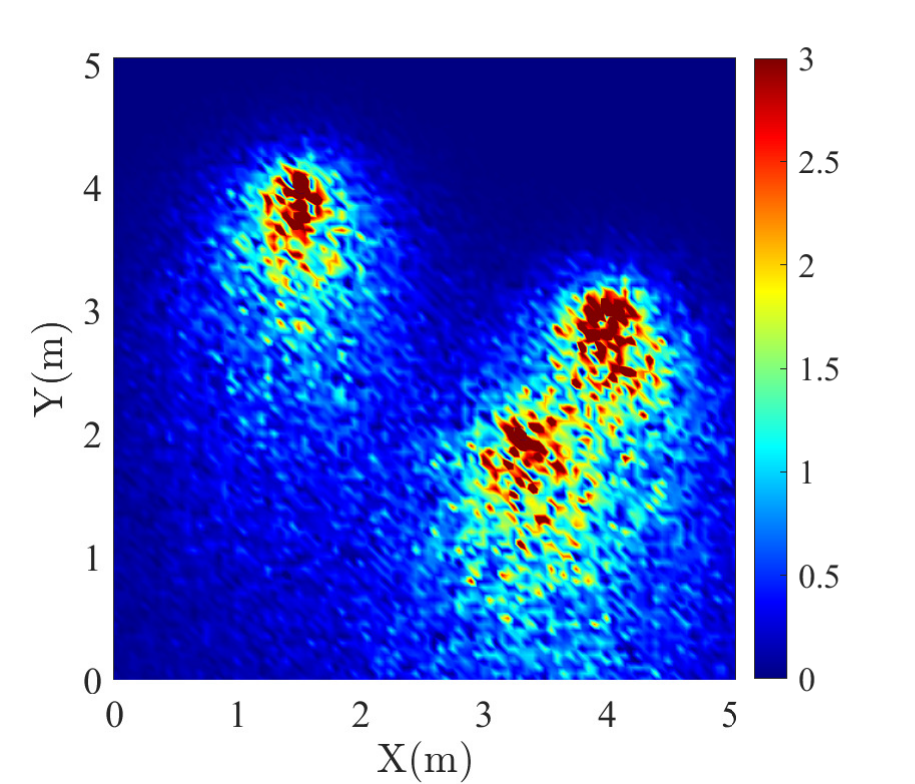}
        \caption{Example plot of gas concentrations distribution at different locations resulting from superposition of dispersion from three gas sources subject to a probability of detection}
        \label{fig:dummygroundtruth}
    \end{figure}
    
    The efficiency of source search and estimation tasks can be significantly improved by employing multiple sensing platforms that operate collaboratively. Multi-robot source localization has been actively investigated over the past decade, primarily focusing on coordination and information-sharing strategies, with most efforts limited to single-source scenarios. Coordinating multiple robots offers the advantage of covering larger areas rapidly and overcoming challenges that often arise in single-robot searches, such as getting stuck in the local plume or failing to identify additional release sources. However, such a potential for more complex multi-source scenarios remains largely unexplored.
     
    Extending current multi-robot search algorithms to handle multiple airborne sources is non-trivial. The primary challenge lies in designing robust inference algorithms capable of characterizing multi-source dispersion under a superposition-based measurement model. Specifically, the contaminant concentration measured at any location is a cumulative sum of emissions from all sources in the environment, modulated by their relative positions and emission characteristics, as shown in Fig. \ref{fig:dummygroundtruth}. Because the individual contributions of each source to the measured concentration are unknown, solving the inverse inference problem to obtain the joint multi-source probability density (JMPD) of the source term becomes highly challenging. Moreover, the path planning and coordination policies need to incorporate the evolving JMPD belief to guide the sensing actions. As a result, directly applying existing multi-robot search or coordination strategies, typically designed for single-source inference, is unlikely to be effective in multi-source settings without substantial reformulation. In the next subsection, we review related work in multi-source inference and multi-robot coordination, and identify the critical gaps addressed in this paper.
    
	\subsection{Related Literature}
    In the literature, many algorithms have been developed for multi-source term estimation using static sensor networks \cite{yee2010validation,Yee2012,WADE201345,SINGH2015402,DANIYAN2024107198}. These works primarily leverage the relatively large amounts of measurements provided by sensor arrays to identify the number of sources and estimate parameters associated with each source. Different Bayesian estimation techniques, such as reversible-jump Markov chain Monte Carlo (RJ-MCMC) or random finite set methods, are frequently employed to derive the multi-source probability density. However, as previously discussed, deploying a large static network within the search domain is often impractical, particularly in dynamic or unpredictable scenarios. 
    
    Shifting to single-robot-based searches for multi-sources, while offering flexibility, is constrained by the sparse measurements available per sampling instant, leading to extended mission times and a heavier reliance on computationally expensive informative path-planning strategies. In fact, single-robot multi-source term estimation solutions have predominantly been designed for radiological estimation, e.g., \cite{RISTIC20101225,BAI2023104529,LAZNA2024}, where dispersion patterns and sensor readings are less affected by atmospheric turbulence compared to the gas or particulate dispersion. To address these limitations, integrating the sampling flexibility of mobile sensing with the spatial diversity of measurements from multi-robot systems offers a promising approach.

    Several multi-robot planning strategies have been developed for the tasks of source term estimation or source seeking, such as \cite{PARK202072, RISTIC202013, Ristic2017, SONG2020103414, HAJIEGHRARY20161698, song2019multi, Ercolani2023, JANG2023120033, Bourne2020, Zhang10250987, cheng2024sniffysquad}. These strategies focus on collaborative decision making to achieve either centralised or distributed multi-robot sensing, tracking, and exploration. For instance, \cite{SONG2020103414, song2019multi} implemented a weighted social Bayesian estimation protocol to update each robot's belief about the source term. Following this, each robot selected a traditional Infotaxis reward-based control action from a limited set of discrete options. Similarly, \cite{RISTIC202013} proposed a decentralized formation control strategy, where each robot computed its own Infotaxis reward-based control action locally and then performed consensus with other robots to maintain the networked formation. In another approach, \cite{PARK202072, JANG2023120033} introduced a range of active and passive coordination strategies utilizing the Infotaxis-based reward function or its variants to determine the control actions for each robot within the network. By optimizing the overall estimated information gain, these strategies enabled robots to negotiate their individual actions sequentially to achieve collaboration. Also focusing on multi-robot collaboration, \cite{Bourne2020} proposed a multi-agent information theoretic path planning strategy called DeMAIT, achieving an efficient source term estimate by collectively maximizing the mutual information gain of the group at the next sensing instance.

     The above-mentioned research demonstrates that increasingly collaborative search strategies can be designed within an information-theoretic framework for multi-robot single-source scenarios. While promising, implementation of these approaches can be both computationally and communicationally demanding. Informative path planners (IPP) typically compute the expected information gain for each candidate action (or sampling location) from a finite set of options, requiring multiple runs of the Bayesian estimator to predict posterior beliefs at new locations. Incorporating a non-myopic planning horizon further increases computational complexity, particularly for multi-robot systems, where the augmented action set introduces additional overhead.
     
     To address these limitations, alternative approaches for multi-robot source seeking or STE have been explored. For example, \cite{Zhang10250987} proposed a gradient-based method where robots move along the gradient descent directions of an information-based loss function defined as the trace of the inverse of the Fisher information matrix. However, this method relies on an Extended Kalman Filter and an isotropic measurement model, which assumes that measurement values depend only on the source–sensor distance. In a recent work, \cite{NANAVATI2024102503} developed a coverage control based coordination policy, which considers the information rewards along the paths generated by a underlying coverage control based spatial planner. These paths drive the robots collectively to high likelihood regions to contain a single source, but the use of a traditional coverage control based spatial planner resulted in the necessity of maximizing the information gain over a time horizon to determine the sampling locations. Although this approach intends to reduce the dimensionality of the action set for multi-robots, its effectiveness is constrained by computational cost depending on the size of time horizon, size of search region, and robot speeds.

    Together, the current literature highlights diverse and innovative methods in which multi-robot systems can address the challenges of collaborative single-source term estimation or source seeking. However, significant gaps remain in developing robust multi-robot coordination strategies and Bayesian inference algorithms for complex multi-source scenarios. Notably, the only work found in this domain is \cite{TRAN2023665}, which proposed a collaborative sequential Monte Carlo (CSMC) estimation framework for multi-source localization and mapping using flocking robots. While innovative, this framework is primarily focused on gas distribution mapping based on Gaussian mixture models rather than physics-grounded atmospheric transport and dispersion models. Furthermore, the approach has not been rigorously validated using real gas release data, limiting its applicability in real-world scenarios. These limitations highlight that multi-robot and multi-source STE remains an unsolved problem and require further development of effective inference and coordination algorithms capable of practical applications.

	\subsection{Contributions}
    \tb{To address the challenges associated with information-theoretic planning and traditional coverage control identified in the previous subsection, this paper proposes a wind-aware coverage control(WCC) -based multi-robot multi-source term estimation (MRMSTE) framework. This approach enables the estimation of multiple source terms of bounded unknown number of release sources, using multiple sensing platforms. The main contributions of this work are threefold.}

    \tb{\begin{enumerate}
        \item \textbf{Hybrid Bayesian Inference for Multi-Source Term Estimation:} 
        We first propose a hybrid inference framework that jointly estimates multiple airborne source terms, and the number of sources. The framework is based on a JMPD formulation, which incorporates physics-informed state transition models for source birth, death, and merge events—an approach that extends beyond traditional multi-target tracking filters (e.g., \cite{Kreucher2005,DANIYAN2024107198}).
        \item \textbf{Wind-aware coverage control for multi-robot sensing:}
        To actively guide sensing, we then introduce a WCC strategy that integrates the evolving JMPD belief with average wind information over the search domain to prioritise regions of high detection likelihood. By embedding an anisotropic, plume-informed sensing performance function into a coverage control framework, WCC enables efficient multi-robot deployment and data collection without the combinatorial complexity of the sensing actions as in information-theoretic path planning. Additionally, the use of a coverage control-based design inherits collision-avoidance, eliminating the need for explicit collision avoidance mechanisms as in other path-planning strategies.
        \item \textbf{Comprehensive evaluation and real-world validation}:
        The proposed MRMSTE framework is evaluated through extensive Monte Carlo studies, including comparisons with traditional coverage control strategies and static sensor networks. We further validate the approach through real-world experiments using TurtleBot platforms equipped with CO$_2$ sensors, demonstrating practical multi-source estimation under realistic dispersion conditions and illustrating the clear advantages of the proposed WCC path planner over traditional coverage control solutions. To the best of our knowledge, this is the first real-world experimental demonstration of a working MRMSTE solution.  
    \end{enumerate}}

    \tb{Together, these contributions establish a unified inference-and-planning framework for scalable multi-robot estimation of multiple airborne sources under realistic atmospheric dispersion dynamics, establishing a foundation for future real-world applications.}
	
	\section{Preliminaries and Problem Formulation}
    The notation $\partial \mathcal{A}$ is used to refer to the boundary of a closed set $\mathcal{A}$. Additionally, all vectors are expressed using a ``boldsymbol". The notation ``$\vect{a}\cdot\vect{b}$" refers to the dot product between vectors $\vect{a}$ and $\vect{b}$. The notation $\dsqrbrac{1,L}$ denotes set of all the natural number from $1$ to $L$. Moreover, $\Xi_L$ denotes a set of $L!$ permutations of the set $\dsqrbrac{1,L}$. We now include the definition of the metric $\mathcal{O}$ which will be used to analyse the performance of the proposed multi-robot multi-source estimation algorithm in the upcoming sections.

    \begin{definition} \cite{Rahmathullah2017}\label{def:GOSPA}
            \tb{Consider two sets $X = [\scalnot{x}{}{1},...,\scalnot{x}{}{|X|}]$ and $Y = [\scalnot{y}{}{1},...,\scalnot{y}{}{|Y|}]$. The generalized optimal sub-pattern assignment (GOSPA) metric between $X$ and $Y$, denoted by $\mathcal{O}_{X,Y}$, for $|X|<|Y|$ as
        \begin{align}\label{eq:GOSPA}
            \mathcal{O}_{X,Y} = \left[\min_{\xi\in\Xi_{|Y|}}
                 \sum\limits_{i=1}^{|X|}d_c(\scalnot{x}{}{i},\scalnot{y}{}{\xi_i})^2+ c^2(|Y|-|X|)\right]^{\frac{1}{2}}\,,
        \end{align}
        where $d_c(\scalnot{x}{}{i},\scalnot{y}{}{\xi_i}) = \min(\left|\left|\scalnot{x}{}{{\xi_m}}-\scalnot{y}{}{m}\right|\right|,c)$ is a cut-off metric \cite{mahler2014} with $c>0$ denoting some cut-off threshold, $\xi$ is a candidate permutation sequence from $\Xi_{|Y|}$ and $\xi_i$ is the $i^{th}$ element of $\xi$ permutation.}
    \end{definition}

    \subsection{Motion and Measurement Model}\label{sec:measurementModel}
	Consider a set of $n$ homogeneous robots in a closed operating region denoted by $\mathcal{D} \subset \mathbb{R}^2$, whose boundaries are denoted by $\partial \mathcal{D}$. The operating region is assumed to be obstacle free. Let each robot be governed by a differential drive dynamics given by
        \begin{align}\label{eq:dynamics}
            \dot{\vect{p}}^{(i)} &= v^{(i)}\vec{e}_{\gamma^{(i)}},~~~
            \dot{\gamma}^{(i)} = \omega^{(i)}
        \end{align}
    where $\vect{p}^{(i)}$ denote the robot position vector, $\gamma^{(i)}\in[-\pi,\pi)$ denote the heading of the robot body X-axis as shown in Fig. \ref{fig:robot_dynamics}, and $\vec{e}_{\gamma^{(i)}} = \left[\cos\gamma^{(i)}~~\sin\gamma^{(i)}\right]$. The linear speed, $v^{(i)}$ and angular speed $\omega^{(i)}$ corresponding to the $i^{th}$ robot are its control inputs.    
    \begin{figure}[!ht]
        \centering
        \includegraphics[width=0.5\linewidth]{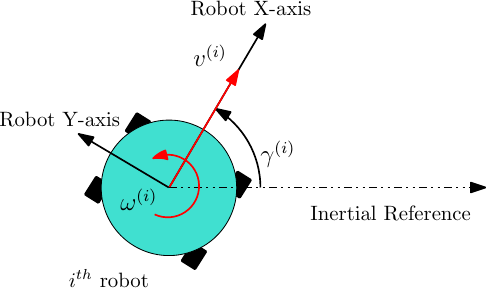}
        \caption{Robot and sensor heading angles.}
        \label{fig:robot_dynamics}
    \end{figure}

	\tb{For $M$ sources of gas release in the environment, let the concentration measurement obtained by the $i^{th}$ robot at $k^{th}$ instant, denoted by $\scalnot{z}{k}{i}:=\scalnot{z}{k}{i}(\vectnot{p}{k}{i})~\forall~i\in\dsqrbrac{1,n}$, be modelled as \cite{hutchinson2018information,Hutchinson2019JFR}}
    \begin{align} \label{eq:sensor}
        \scalnot{z}{k}{i} = \begin{cases}
            C_{M}(\vectnot{p}{k}{i}) + \nu_k & \text{if detection event} \\
            0 & \text{if non-detection event}
        \end{cases}\,.
    \end{align}
    The function $C(\vect{p}^{(i)})$ is the ground truth concentration at $\vect{p}^{(i)}$ modelled as
    \begin{align}
        C_{M}(\vect{p}^{(i)}) =
        \sum\limits_{m=1}^{M}&\left[\frac{Q_m}{2\pi D} \exp{\left(\frac{\vect{v}_w\cdot(\vect{p}^{(i)}-\vect{s}_m)}{2D}\right)} K_0\left(\frac{||\vect{p}^{(i)}-\vect{s}_m||_2}{\lambda}\right)\right]\,,
    \end{align}
    with source location $\vect{s}_{m}$, release rate $Q_m$ , average particle lifetime $\tau$, gas diffusivity $D$, wind speed $v_w$, wind direction $\psi$, $K_0$ denoting zero-order modified Bessel's equation, and 
    $$\vect{v}_w = v_w\vec{e}_{\psi} = v_w[\cos\psi~~\sin\psi]\,,~~\lambda = 2D\sqrt{\frac{\tau}{4D+v_w^2\tau}}\,.$$ Moreover, $\nu_k\sim \mathcal{N}(0,\sigma^2)$ is the measurement noise independent and identically distributed across space and time. For brevity and ease of the reader, the subscript $\star_{k}$ has been dropped at some places and an absence of the subscript refers to the current time step. Furthermore, as discussed in \cite{hutchinson2018information}, a successful gas detection by the onboard sensor is identified using a threshold concentration denoted by $z_{thr}$. Additionally, the probability of successfully detecting gas is denoted by $P_d$.
    
	\subsection{Coverage Control Preliminaries}
	The region $\mathcal{D}$ can be partitioned into $n$ cells (or dominance regions), denoted by $\mathcal{W} = \{\mathcal{W}_1,\ldots,\mathcal{W}_n\}$, corresponding to each robot's positions such that $\mathcal{D} = \bigcup_{i=1}^{n}\mathcal{W}^{(i)}$ and $\mathcal{W}^{(i)}\cap\mathcal{W}^{(j)} = \partial\mathcal{W}^{(i)}\cap\partial\mathcal{W}^{(j)}$ for $i\neq j$, that is, $\mathcal{W}$ is a set cells that are collectively exhaustive and mutually exclusive except for boundaries. Each robot is required to sample the environment at points within their cells such that the accuracy of the posterior estimate of the gas distribution is improved. To achieve a collective action, we choose to minimize a locational optimisation objective function as
	\begin{align}\label{eq:obj1}
        \begin{split}
            \min_{\substack{\vectnot{p}{}{i},\mathcal{W}}}\mathcal{H}(\vect{p}^{(i)},\mathcal{W}) = \min_{\substack{\vectnot{p}{}{i},\mathcal{W}}},\sum_{i=1}^{n}\int_{\mathcal{W}^{(i)}}f_i(\vect{q})\phi(\vect{q})d\vect{q}
        \end{split}
	\end{align}
	where $f_i(\vect{q})\triangleq f(\vect{p}^{(i)},\vect{q}):\mathbb{R}_+\to\mathbb{R}_+$ is a non-decreasing differentiable function that quantifies the ``sensing'' performance \cite{Cortes2004}. Furthermore, each point $\vect{q}\in\mathcal{D}$ is mapped to a density function $\phi$ defined as $\phi(.):\mathcal{D}\to\mathbb{R}_+$.  Generally, the density function is known \textit{apriori} in traditional coverage control. However, for the application of multi-source estimation, the density function relates to the current estimate of the source locations. It can be noted that minimizing \eqref{eq:obj1}, with $\phi$ representing the spatial distribution of potential source locations, would effectively drive robots towards regions of their respective partitions likely to contain a source. Similar design choices have been demonstrated to perform satisfactorily for a single-source scenario in \cite{NANAVATI2024102503}.

    \subsection{Problem Statement}
    In traditional coverage control designs, the function $f$ is typically selected as $f = || \vect{p}^{(i)} -  \vect{q}||^2$ in order to account for the radially symmetric decline in sensing performance as the distance from the robot increases. However, for a multi-source term estimation scenario, the dispersion of airborne containments in $\mathcal{D}$ is not radially symmetric, rather it is subject to prevailing wind speed and direction relative to the sensing platform. Therefore, further investigation in the design of an appropriate ``sensing performance'' function and associated Voronoi partition is necessary for minimizing \eqref{eq:obj1} in a meaningful manner. Therefore, we now formally state the objectives of this work:
    \begin{enumerate}[label=(\roman*)]
        \item Design a label-free belief update procedure that efficiently localizes the source of a pre-specified number of unknown airborne releases in the environment along with environmental parameters;\label{objective:1}
        \item Design a wind-aware ``sensing performance'' function and the resulting generalised Voronoi partitioning;
        \item Design linear and angular speed commands that minimize \eqref{eq:obj1} and drive $n$ differential-drive robots to sampling locations that aids effective estimation by the proposed estimation technique in \ref{objective:1}.
    \end{enumerate}
    \begin{remark}
        \tb{In this paper, we assume a centralized architecture for the inference and path planning framework to avoid additional design complications involving decentralized framework.}
    \end{remark}
    
    \section{Sequential Bayesian MSTE}\label{sec:BayesianEstimation}

    \subsection{Multi-source state}

    \tb{In this work, we assume that the sources within the region of interest $\mathcal{D}$ are sufficiently separated spatially, and that at least one active source presets with a known upper bound $M_\mathrm{max}$ for the true source count. Environmental parameters, such as wind speed and direction, particle lifetime, and diffusivity, are treated as known \textit{apriori}. In contrast, the individual source locations, emission rates, and the actual number of active sources are all unknown and must be jointly estimated from measurements.}
    
    \tb{We denote the $m$-th source term at time $k$ by $\vect{\theta}_{k,m} = [\vect{s}_{k,m}^{\intercal},\,Q_{k,m}]^\intercal \in \mathbb{R}^{3}$, $m\in \{1, \ldots, M_{\mathrm{max}} \}$. The collection of active sources is represented as an unordered finite set $\Theta_{k,M_{k}}=\{ \vect{\theta}_{k,1}, \ldots, \vect{\theta}_{k, M_{k}} \}$, where the number of sources, $M_{k} \in \{1, \ldots, M_{\mathrm{max}} \}$, is an unknown discrete variable to be estimated together with other source terms. The resulting system state is therefore the hybrid variable $\vect{X}_{k} = (\Theta_{k,M_{k}}, M_{k})$, whose continuous component $\Theta_{k,M_k}$ changes dimension according to the discrete mode $M_k$. Each element $\vect{\theta}_{k,m}$ can be interpreted as a component of a variable-structure state vector. Given the hybrid state $\vect{X}_{k}$, the goal of the multi-source estimation problem is to compute the posterior JMPD,
    \[
    p(\vect{X}_{k}|\vect{Z}_{1:k}) = p\left(\Theta_{k,M_{k}},M_{k}|\vect{Z}_{1:k}\right),
    \]
    where $\vect{Z}_{1:k}$ denotes all measurements collected up to time $k$ by the sensing robots. }
    
    \begin{remark}
    \tb{Representing $\Theta_{k,M_k}$ as an unordered set is mathematically equivalent to the symmetrised vector representation used in \cite{Kreucher2005}, and corresponds directly to the finite-set representation adopted in Random Finite Set (RFS) theory \cite{mahler2007statistical}. This formulation enforces permutation invariance and avoids artificial labelling of sources, which not only aligns naturally with the superposition-based observation model but is also essential for implementing the proposed merge and removal mechanisms in the state transition. }
    \end{remark}

    \tb{For completeness, we note that although $\Theta_{k,M_k}$ is formally an unordered set, many expressions in subsequent sections use an ordered vector representation (e.g., $\theta$, $\theta'$) when writing densities or transition kernels. These representations are implicitly symmetrised, so that all probability densities remain permutation-invariant. We also drop the cardinality subscript $M_k$ in $\Theta_{k,M_k}$ whenever the meaning is clear from context. }
    
    \subsection{State transition modelling} \label{sec:transition}

    \tb{With the hybrid state defined as $\vect{X}_k = (\Theta_k, M_k)$, the state transition density factorises as
    \begin{equation}
    \begin{aligned}
    p(\vect{X}_k \mid \vect{X}_{k-1})
    &= p(\Theta_k \mid M_k, \Theta_{k-1}, M_{k-1}) p(M_k \mid \Theta_{k-1}, M_{k-1}),
    \end{aligned}
    \label{eq:transition_factorisation}
    \end{equation}
    where the discrete kernel governs the state-dependent cardinality changes and the continuous kernel describes source term evolution conditional on the new cardinality. In this work, these transition kernels are carefully designed by integrating deterministic rules inspired by atmospheric dispersion physics (e.g., source indistinguishability and negligible emissions) with stochastic birth--death dynamics. The proposed transition models go beyond the standard RFS/JMPD frameworks that rely primarily on measurement-driven interactions \cite{DANIYAN2024107198}, to ensure a physically plausible dynamic representation for multi-source terms.}

    \tb{The discrete variable $M_k$ is constrained to change by at most one per step, $M_k \in \{M_{k-1}-1, M_{k-1}, M_{k-1}+1\}$, clipped to $1 \leq M_k \leq M_{\max}$. Its state-dependent kernel is
    \begin{equation}
        g_{k}(m \mid m', \theta') := p(M_k = m \mid \Theta_{k-1}=\theta', M_{k-1}=m'),\label{eq:transition}
    \end{equation}
    which is defined hierarchically as follows. First, we prioritize deterministic transition, such that $g_{k}(m \mid m', \theta') =1$ (others zero) if a merge occurs ($\min_{i\neq j} \|\vect{s}_{i}-\vect{s}_{j}\| < d_{\merge}$, reflecting spatial plume overlap) or, absent merge, a removal ($\min_i Q_{k-1,i} < Q_{\min}$, culling weak sources to align with sensor thresholds). Otherwise, we introduce stochastic birth/death events to form a Markov-jump transition model, in which $M_{k-1}$ evolves according to the tri-diagonal transition matrix $\mathbf{\Pi} \in \mathbb{R}^{M_{\max} \times M_{\max}}$ with entries defined as follow
    \begin{equation}
    \mathbf{\Pi} = 
    \begin{bmatrix}
    1-p_{\mathrm{b}} & p_{\mathrm{b}} & 0 & 0 & 0 \\
    p_{\mathrm{d}} & p_{\mathrm{s}} & p_{\mathrm{b}} & 0 & 0 \\
    0 & \ddots & \ddots & \ddots & 0 \\
    0 & 0 & p_{\mathrm{d}} & p_{\mathrm{s}} & p_{\mathrm{b}} \\
    0 & 0 & 0 & p_{\mathrm{d}} & 1-p_{\mathrm{d}}
    \end{bmatrix},
    \label{eq:Pi_example}
    \end{equation}
    where the $p_{\mathrm{b}}$, $p_{\mathrm{d}}$ and $p_{\mathrm{s}}=1-p_{\mathrm{b}}-p_{\mathrm{d}}$ are probabilities of birth, death and survival, respectively. In this case, the transition kernel reduces to $g_{k}(m \mid m', \theta') = \mathbf{\Pi}_{(m',m)}$.}
    
    \tb{This design is novel in airborne multi-source term estimation. Specifically, deterministic rules reflect physical processes (plume overlap and sensor-limited detection), while the underlying Markov jump provides controlled stochastic exploration of the unknown source count.}

    \tb{For the continuous state, given $M_{k-1}=m'$ and $M_k=m$, the transition density can be written as 
    \begin{equation}
    \begin{aligned} \label{eq:cont_transition}
    h_{k}(\theta \mid m, \theta',m') :=
    &\; p(\Theta_k = \theta \mid 
         M_k = m,\,
         \Theta_{k-1} = \theta', M_{k-1} = m').
    \end{aligned}
    \end{equation}
    This continuos transition density is designed as follows:
    \begin{enumerate}
        \item[(i)] Survival ($m=m'$): each of the $m$ sources evolves independently via Gaussian diffusion,
      \[
      h_{k}(\theta \mid m, \theta',m')  = \prod_{i=1}^{m} \mathcal{N} \bigl(\vect{\theta}_{k,i} ; \vect{\theta}'_{k-1,i}, \mathbf{\Sigma}_k\bigr),
      \]
      where $\mathbf{\Sigma}_k = \operatorname{diag}(\sigma_x^2,\sigma_y^2,\sigma_Q^2)$.
      \item[(ii)] Birth ($m=m'+1$): the $m'$ survival sources diffuse as above and one new source is born from the prior
      \[
      b(\vect{\theta}_{k,m}) = \mathcal{U}_{\mathcal{A}}(\vect{s})\cdot \Gamma(Q_0,{Q_{1}}),
      \]
      where $\Gamma(\cdot)$ represents Gamma distribution. This yields
      \[
      h_{k}(\theta \mid m, \theta',m')  = b(\vect{\theta}_{k,m})\prod_{i=1}^{m'} \mathcal{N} \bigl(\vect{\theta}_{k,i} ; \vect{\theta}'_{k-1,i}, \Sigma_k\bigr).
      \]
      \item[(iii)] Deterministic merge or removal ($m=m'-1$): first form the reduced set $\tilde{\theta}' \in \mathbb{R}^{3(m-1)}$ from $\theta'$ in the following hierarchical order: 1) replace the two closest sources $i,j$ (the pair achieving $d_{\min}<d_{\merge}$) by a single source with mass-weighted centroid
        \[
        \tilde{\vect{\theta}}_{j} = \Bigl( \frac{Q'_i \vect{s}'_i + Q'_j \vect{s}'_j}{Q'_i + Q'_j},\; Q'_i + Q'_j \Bigr).
        \]
        or, 2) remove the source with smallest $Q_{k-1,i}$ (or the one selected by removal rule). When merge or removal is triggered, the reduced set $\tilde{\theta}'$ replaces $\theta'$ to go through sources diffuse as in Step (i). Thus, in this case we have
          \[
          h_{k}(\theta \mid m, \theta', m') = \prod_{i=1}^{m-1} \mathcal{N} \bigl(\vect{\theta}_{k,i} ; \vect{\theta}'_{k-1,i}, \mathbf{\Sigma}_k\bigr).
          \]
        \end{enumerate}}

    \tb{This continuous transition preserves total emitted strength during merging, ensures smooth dynamics, and maintains permutation invariance, which is critical for particle filter performance in unordered-set representations.}

    \subsection{Observation model}
    
    \tb{The measurement model, introduced in Section~\ref{sec:measurementModel}, is based on superposition of individual plume contributions at one robot location. The likelihood of the current measurement set $\vect{Z}_{k} = \{z_{k}^{(j)}\}_{j=1}^{n}$ collected at known sensor positions $\{\vect{p}_{k}^{(j)}\}$ is
    \begin{align} 
     \mathcal{L}_{k}(\theta,m) &:= p(\vect{Z}_k \mid \Theta_k=\theta, M_k = m) \prod_{j=1}^{n} p(z_{k}^{(j)} \mid \theta, m),
      \label{eq:likelihood_full}
    \end{align}
where the single-sensor likelihood probability of a new measurement based on the current belief can be formulated as \cite{Hutchinson2019JFR,NANAVATI2024102503},
    \begin{align}\label{eq:likelihood}
        p(z_{k}^{(j)} \mid \theta, m) = 
        \begin{cases}
            p_s(z_{k}^{(j)} ,{\theta}, m) & \text{if } z_{k}^{(j)} \geq z_{\rm thr}\\
            p_{us}(\vect{p}_{k}^{(j)},{\theta},m) & \text{if } z_{k}^{(j)} < z_{\rm thr}
        \end{cases}
    \end{align}
    where
    \begin{align}
        p_s(\star) &= \mathcal{N}\left(\scalnot{z}{k}{j}-C_{m}(\vectnot{p}{k}{j};\theta);~~\sigma_\nu^2\right)\,,\\
        p_{us}(\star) &= 1- \dfrac{P_d}{2} +\dfrac{P_d}{2}\mathsf{erf}\left[\frac{z_{\rm thr} - C_{m}(\vectnot{p}{k}{j};\theta)}{\sqrt{2}\sigma_{\nu}}\right]\,.
    \end{align} }
    
    \tb{Both the transition kernels \eqref{eq:transition} and \eqref{eq:cont_transition} and the likelihood \eqref{eq:likelihood_full} are permutation-invariant with respect to the elements in $\theta$ or $\theta'$, ensuring the posterior JMPD inherits this symmetry of the unordered source set.}

    \subsection{Recursive Bayesian Estimation}

    \tb{The objective of Recursive Bayesian Estimation is to compute the posterior distribution over the hybrid state $\vect{X}_k = (\Theta_k, M_k)$ given all measurements up to time $k$. This can be represented by the posterior density 
    \begin{equation}
        f_{k}^{(m)}(\theta) := p(\Theta_{k}=\theta, M_{k} = m \mid \vect{Z}_{1:k}).
    \end{equation}}
    
    \tb{The exact Bayesian filter consists of the usual prediction and update steps applied to the hybrid state. The predicted hybrid density is obtained by marginalising over all possible previous cardinalities and configurations:
    \begin{equation}
    \begin{aligned} \label{eq:prediction_mode}
    &f_{k|k-1}^{(m)}(\theta)
    \nonumber\\ &= \sum_{m'=1}^{M_{\max}} 
          \int h_{k}(\theta \mid m, \theta', m')\,
               g_{k}(m \mid m', \theta') \,f_{k-1}^{(m')}(\theta')\, d \theta'.
    \end{aligned}
    \end{equation}
    The update step follows the Bayes' rule producing the mode-conditioned posteriors such that
    \begin{equation}
      f_k^{(m)}(\theta) = \frac{\mathcal{L}_{k}(\theta,m)\,f_{k|k-1}^{(m)}(\theta)}
                               {p(\vect{Z}_{k} \mid \vect{Z}_{1:k-1})}.
      \label{eq:posterior_mode}
    \end{equation}
    Equations~\eqref{eq:prediction_mode}--\eqref{eq:posterior_mode} define the exact recursive Bayesian estimator for the hybrid multi-source state. Closed-form solutions are intractable due to the variable dimension and non-linear/non-Gaussian nature of the model; we therefore resort to sequential Monte Carlo approximation in the following subsection.}
    
    \subsection{Particle Filter Implementation}

    \tb{Using the Monte Carlo representation of a probability distribution function, the posterior $p(\vect{X}_{k}\mid \vect{Z}_{1:k})$ can be approximated as 
     \begin{align}\label{eq:posteriorbelief}
         p({\Theta}_{k}=\theta, M_{k}=m|\vect{Z}_{1:k}) \approx \sum_{p=1}^{N_p}\weight{k}{p} \delta_{M_{k}^{(p)}}(m) \delta_{\particle{k}{p}}(\theta),
     \end{align}
     where $\{(M_{k}^{(p)},\particle{k}{p}), \weight{k}{p}\}_{p=1}^{N_p}$ is a set of $N_p$ weighted samples with weight $\sum_{p=1}^{N_p} \weight{k}{p} = 1$, $\delta_{\particle{k}{p}}(\cdot)$ the Dirac delta function located at $\particle{k}{p}$, and $\delta_{M_{k}^{(p)}}(m)$ the Kronecker delta that takes $1$ if and only if $m=M_{k}^{(p)}$. The particle filter recursively computes the set of new particles at time instant $k$, given the particle set $\{(M_{k-1}^{(p)},\particle{k-1}{p}), \weight{k-1}{p}\}_{p=1}^{N_p}$ at $k-1$ and most recent measurement $\vect{Z}_{k}$. }

     \tb{We adopt a bootstrap proposal in which particles are propagated through the hybrid state-transition model:
    \[
    \vect{X}_k^{(p)} \sim p(\vect{X}_k \mid \vect{X}_{k-1}^{(p)}).
    \]
    Concretely, for each particle,
    \begin{align}
        M_k^{(p)} & \sim p(M_k \mid \Theta_{k-1}^{(p)}, M_{k-1}^{(p)}),  \\
        \Theta_k^{(p)} & \sim p(\Theta_{k} \mid M_k^{(p)}, \Theta_{k-1}^{(p)}, M_{k-1}^{(p)}),
    \end{align}
    using the merge/removal/birth–death mechanism for the discrete cardinality and the corresponding continuous transitions defined in Section~\ref{sec:transition}.}

    \tb{Because the proposal equals the transition density, the importance weights update according to
    \begin{equation}
    \label{eq:bootstrap_update}
    \tilde w_k^{(p)}
    \propto
    w_{k-1}^{(p)} \,
    \mathcal{L}_k \big(\Theta_k^{(p)}, M_k^{(p)}\big),
    \end{equation}
    where $\mathcal{L}_k$ is the likelihood function defined in \eqref{eq:likelihood_full}. The updated weights are normalised as $w_k^{(p)} = \tilde w_k^{(p)} / \sum_p \tilde w_k^{(p)}$. The overall filtering algorithm is now summarized in Algorithm \ref{alg:one_step_pf} including the optional resampling step.}

\begin{algorithm}[t]
  \caption{Proposed SIR PF for multi-sources}
  \label{alg:one_step_pf}
  \begin{algorithmic}[1]
    \REQUIRE
      $\big\{ (\Theta_{k-1}^{(p)}, M_{k-1}^{(p)}),\, w_{k-1}^{(p)} \big\}_{p=1}^{N_p}$; $\vect{Z}_k$;

    \ENSURE Updated particle set at time $k$:
      $\big\{ (\Theta_{k}^{(p)}, M_{k}^{(p)}),\, w_{k}^{(p)} \big\}_{p=1}^{N_p}$.
    %
    \FOR{$p = 1$ \TO $N_p$}
      \STATE Sample new cardinality
        $M_k^{(p)} \sim g_k\big(m \mid M_{k-1}^{(p)}, \Theta_{k-1}^{(p)}\big)$ \\
        according to the merge/removal/birth--death rules.
      \STATE Sample new source set
        $\Theta_k^{(p)} \sim h_k\big(\theta \mid M_k^{(p)}, \Theta_{k-1}^{(p)}, M_{k-1}^{(p)}\big)$ \\
        with the corresponding diffusion/birth/merge dynamics.
      \STATE Set $\vect{X}_k^{(p)} = \big(\Theta_k^{(p)}, M_k^{(p)}\big)$.
    \ENDFOR
    %
    \FOR{$p = 1$ \TO $N_p$}
      \STATE Compute unnormalised weight
        $\tilde w_k^{(p)} = w_{k-1}^{(p)}\, \mathcal{L}_k\big(\Theta_k^{(p)}, M_k^{(p)}\big)$.
    \ENDFOR
    \STATE Normalise weights:
      $w_k^{(p)} = \tilde w_k^{(p)} \Big/ \sum_{q=1}^{N_p} \tilde w_k^{(q)}$.
    %
    \STATE Compute effective sample size
      $N_{\mathrm{eff}} = 1 \big/ \sum_{p=1}^{N_p} \big(w_k^{(p)}\big)^2$.
    \IF{$N_{\mathrm{eff}} < N_{\mathrm{thr}}$}
      \STATE Resample to obtain an unweighted set
        $\big\{ (\Theta_k^{(p)}, M_k^{(p)}),\, w_k^{(p)} = 1/N_p \big\}_{p=1}^{N_p}$.
    \ENDIF
  \end{algorithmic} 
\end{algorithm}

    \subsection{Multi-source estimate extraction}\label{sec:estimatecompute}
    \tb{After the particle filter update at time $k$, the posterior JMPD of the multi-source state is represented by a weighted but unordered set of particles of varying cardinality
     $\{\scalnot{\Theta}{k}{p},\scalnot{M}{k}{p},\weight{k}{p}\}_{p=1}^{N_p}$. Since the source components may inconsistently ordered across particles, direct averaging to extract the estimates is not meaningful. Exploiting the permutation-invariant feature of JMPD, we therefore impose a canonical labelling of particle components prior to computing point estimates.}

    \tb{First, we aggregate all source components from all particles into a single collection of source terms defined as $$\mathcal{Y}_k = \left\{\left.\scalnot{\theta}{k,i}{p}\in\mathbb{R}_{1\times 3}~\right|~\forall~p\in\dsqrbrac{1,N_p}\text{ and }~i\in\dsqrbrac{1,\scalnot{M}{k}{p}}\right\}\,.$$ 
    A $K$-means clustering with $K = M_{\max}$ is performed on $\mathcal{Y}_k$, yielding $M_{\max}$ centroids denoted by $\Pi = \{\mu_{k,1},\ldots,\mu_{k,M_{\max}}\}$. These centroids serve as canonical reference labels in the parameter space, thus, providing a consistent geometric coordinate system against which particle components can be matched.}

    \tb{For each particle $p$, we assign its unordered components to the canonical centroids $\{\mu_{k,j}\}_{j=1\to M_{\max}}$ by solving for the following optimization problem 
    $$\xi^p = \arg\min_{\xi\in\Xi_{M_{\max}}}
                 \sum\limits_{i=1}^{\scalnot{M}{k}{p}}||\scalnot{\theta}{k,i}{p}, - \mu_{k,\xi_i}||_2\,.$$ 
    where $\xi^p$ denotes an optimal assignment mapping from $p^{th}$ particle components to centroid indices. Each particle is then reordered in a canonical, label-consistent order, that is, $$\scalnot{\Theta}{k}{p}\to\left(\scalnot{\theta}{k,\xi^p_1}{p},\ldots,\scalnot{\theta}{k,\xi^p_{\scalnot{M}{k}{p}}}{p}\right)$$
    It is important to note that the particles with fewer than $M_{\max}$ components simply leave some labels unassigned.} 
    
    \tb{For each canonical label $j \in \{1,\dots,M_{\max}\}$, the  probability of existence at $k^{th}$ instant is computed as 
    $$P_{e,k}^{(j)} = \sum_{p}\mathsf{b}_{p,j}\weight{k}{p}$$
     where $\mathsf{b}_{p,j}\in\{0,1\}$ is a binary variable which assumes the value 1 only if the $p^{th}$ particle contains a parameter vector with the $j^{th}$ label. Let $\mathcal{J}$ denote the set for each canonical label with $P_{e,k}^{(j)} \geq 0.5$. Only source terms for which $j\in\mathcal{J}$ are reported as valid inferred sources. Therefore, the estimated number of sources at $k^{th}$ instant, denoted by $\bar{M}_k$, is $|\mathcal{J}|$. This avoids reporting spurious birth events and reduces sensitivity to transient particles of low posterior weight.}

    \tb{The parameter estimate ${\bar{\theta}}_{k,j}$ corresponding to the $j^{th}\in\mathcal{J}$ label is the posterior weighted mean computed as
    $$\bar{\theta}_{k,j} = \frac{\sum_{p}\mathsf{b}_{p,j}\weight{k}{p}\scalnot{\theta}{k,j}{p}}{\sum_{p}\mathsf{b}_{p,j}\weight{k}{p}}\,.$$}

    \tb{This yields the estimated location and emission rate of each source that is sufficiently likely to exist under the posterior distribution. The procedure provides a robust, permutation-invariant means of extracting coherent multi-source estimates from hybrid particle representations with varying cardinality.}

    \subsection{Coverage density construction}
    \tb{Note that after each environment sampling instance, the posterior belief of the source term is updated and consequently the density function $\phi(\star)$ should be updated. We can now mathematically express $\phi(.)$ at $\vect{q}$ after the $k^{th}$ belief update can be expressed as
    \begin{align}\label{eq:densityfunction}
        \phi(\vect{q}) = p(\vect{X}_k^{(p)}(x,y)|\vect{Z}_{k})\approx \mathcal{G}\left(\weight{k}{p},\vect{X}_k^{(p)}\right)
    \end{align}
    where $\mathcal{G}\left(\weight{k}{p},\vect{X}_k^{(p)}\right)$ is a Gaussian Mixture model constructed using the weighted samples $\{\weight{k}{p},\vect{X}_k^{(p)}\}$.}

    \tb{In next section, we propose and analyse a novel ``\textit{wind-aware} gas sensing Voronoi generation distance'' and the consequent generalised Voronoi-partition to perform locational optimization in the context of multi-source term estimation.}
    
	\section{Gas Sensing Performance Modelling}
	The sensing performance is usually modelled to decay with an increase in distance from the robot in a radially symmetric fashion for visual or Lidar sensors. A radial symmetric performance function may be true for applications where the primary objective is to maximize the monitoring/surveillance capabilities via range or visibility sensors. However, in this paper for observing the dispersion of airborne contaminates, we design the sensing performance function $f$ to account for an asymmetric plume structure due to the environmental wind condition as follows
    \begin{align}\label{eq:sensingperform}
        f_i(\vect{q}) = ||\vect{p}^{(i)}-\vect{q}||^2 +\alpha||\vect{p}^{(i)}-\vect{q}||\left(\vect{p}^{(i)}-\vect{q}\right)\cdot\vec{e}_{\psi}
    \end{align}
    where $-1<\alpha\leq0$ is a tuning parameter. The second term in \eqref{eq:sensingperform} acts as a additive penalisation based on the relative orientation of $\vect{q}$ w.r.t to the robot location and the average direction of wind flow in the search domain. Such a structure of the sensing function effectively captures the improved likelihood of sensing gas concentration downwind from potential sources. It is important to note that the V-distance proposed in \eqref{eq:sensingperform} is positive-definite function, that is, $f_i(\vect{q})>0~~\forall~\vect{q}\neq\vect{p}^{(i)}$ and satisfies
    $$||\vect{p}^{(i)}-\vect{q}||^2(1-|\alpha|)\leq f_i(\vect{q})\leq||\vect{p}^{(i)}-\vect{q}||^2(1+|\alpha|)$$
    for some $\alpha\in(-1,0]$. The tuning parameter $\alpha$ can be thought of as the relative importance associated with directional sensing degradation as compared to sensing degradation with larger distances. For $\alpha=0$, the proposed sensing function reduces down to the Euclidean distance metric resulting in the traditional Voronoi Diagram. 
    
    We now propose a generalised partition, denoted by $\mathcal{V}(\{\vect{p}^{(i)}\}_{1:n},f,\mathcal{D})\equiv\{\mathcal{V}^{(i)}\}_{1:n}$, on $\mathcal{D}$ using the performance function proposed in \eqref{eq:sensingperform} as a ``$\mathcal{V}$-assignment rule" (or ``$\mathcal{V}$-distance") along with the set of points $\{\vect{p}^{(i)}\}_{i=1}^{n}$ as
    \begin{align}\label{eq:Voronoi}
        \mathcal{V}^{(i)} &= \{\vect{q}~|~f_i(\vect{q})\leq f_j(\vect{q})~\forall~j\neq i,~\vect{q}\in\mathcal{D}\}
    \end{align}
    where $\mathcal{V}^{(i)}$ denotes the generalised dominance region corresponding to $i^{\rm th}$ robot based on the proposed sensing performance function. We now prove that $\mathcal{V}(\{\vect{p}^{(i)}\}_{1:n},f,\mathcal{D})\equiv\{\mathcal{V}^{(i)}\}_{1:n}$ results in a tessellation or a generalised Voronoi diagram on $\mathcal{D}$.
    \begin{lemma}\label{lemma:Voronoiset}
        The generalized partition $\mathcal{V}(\{\vect{p}^{(i)}\}_{1:n},f,\mathcal{D}) = \{\mathcal{V}_1,\ldots,\mathcal{V}_n\}$ for $n\geq2$ defined as per \eqref{eq:Voronoi} are non-empty, collectively exhaustive and mutually exclusive, except for their boundaries, on $\mathcal{D}$.
    \end{lemma}
    \begin{proof}
        It can be trivially seen from \eqref{eq:Voronoi} that each $\vect{q}\in\mathcal{D}$ is assigned to at least one cell as the assignment is performed using a minimum $\mathcal{V}$-distance criterion. Therefore, the resultant cells $\mathcal{V}^{(i)}~\forall~i\in\dsqrbrac{1,n}$ are collectively exhaustive, that is, $\bigcup_{i=1}^{n}\mathcal{V}^{(i)} = \mathcal{D}$. Furthermore, consider the set $\bar{\mathcal{V}}^{(i)}$ defined as
        $$\bar{\mathcal{V}}^{(i)} = \{\vect{q}~|~f_i(\vect{q})\leq\tilde{\epsilon},~\vect{q}\in\mathcal{B}^{(i)}\}$$ where $\tilde{\epsilon}>0$ and $\mathcal{B}^{(i)}$ is a ball of radius $\delta>0$ around $\vect{p}^{(i)}$. For an arbitrarily small $\tilde{\epsilon}>0$, there exist a $\delta$ such that $$\varnothing\neq\bar{\mathcal{V}}^{(i)}=\mathcal{B}^{(i)}\subset\mathcal{V}^{(i)}~\forall~i\in\dsqrbrac{1,n}.$$ Hence, the dominance regions $\mathcal{V}^{(i)}$ are non-empty for all $i\in\dsqrbrac{1,n}$ with a positive Lebesque measure.
        
        Moreover, let $\mathcal{E}_{ij}$ be defined as $\mathcal{E}_{ij} = \mathcal{V}^{(i)}\bigcap\mathcal{V}^{(j)}~\forall~i\neq j$\,. Therefore, it can be noted from the \eqref{eq:Voronoi} that, for $i\neq j$, $$\mathcal{V}^{(i)}\bigcap\mathcal{V}^{(j)}  = \{\vect{q}~|~f_i(\vect{q})= f_j(\vect{q})~\forall~j\neq i,~\vect{q}\in\mathcal{D}\}\,.$$
        Therefore, the Voronoi cells $\mathcal{V}^{(i)}$ are mutually exclusive except their boundaries, that is, $\mathcal{V}^{(i)}\bigcap\mathcal{V}^{(j)} = \partial\mathcal{V}^{(i)}\bigcap\partial\mathcal{V}^{(j)}$ for all $i\neq j$. It is shown in Chapter 3 of \cite{boots2009spatial} that the results of Lemma \ref{lemma:Voronoiset} imply that the proposed partition $\mathcal{V}(\{\vect{p}^{(i)}\}_{1:n},f,\mathcal{D})$ forms a tessellation and a generalized Voronoi diagram on $\mathcal{D}$. This completes the proof.  
    \end{proof}
     An example of the proposed partitioning, $\mathcal{V}$, in contrast to the traditional Voronoi partitions is presented in Fig. \ref{fig:voronoi_example}. The partitions presented in Fig. \ref{fig:Voronoia0} correspond to the traditional Euclidean Voronoi partition, whereas, Figs. \ref{fig:Voronoia50}-\ref{fig:Voronoia99} present different partitioning resulting from varying $\alpha$. As can be seen from fig. \ref{fig:voronoi_example}, increasing the value of $\alpha$ limits and extends each robots dominance region away and along from the wind flow direction, respectively. In the context of the proposed sensing performance function, the dominance region $\mathcal{V}^{(i)}$ is representative of $i^{th}$ robot being the most effective sensing platform in the robot network for a $\vect{q}\in\mathcal{V}^{(i)}$.
    
    \begin{figure}
        \centering
        \captionsetup[subfigure]{justification=centering}
        \begin{subfigure}{0.245\linewidth}
            \centering
            \includegraphics[width=\linewidth]{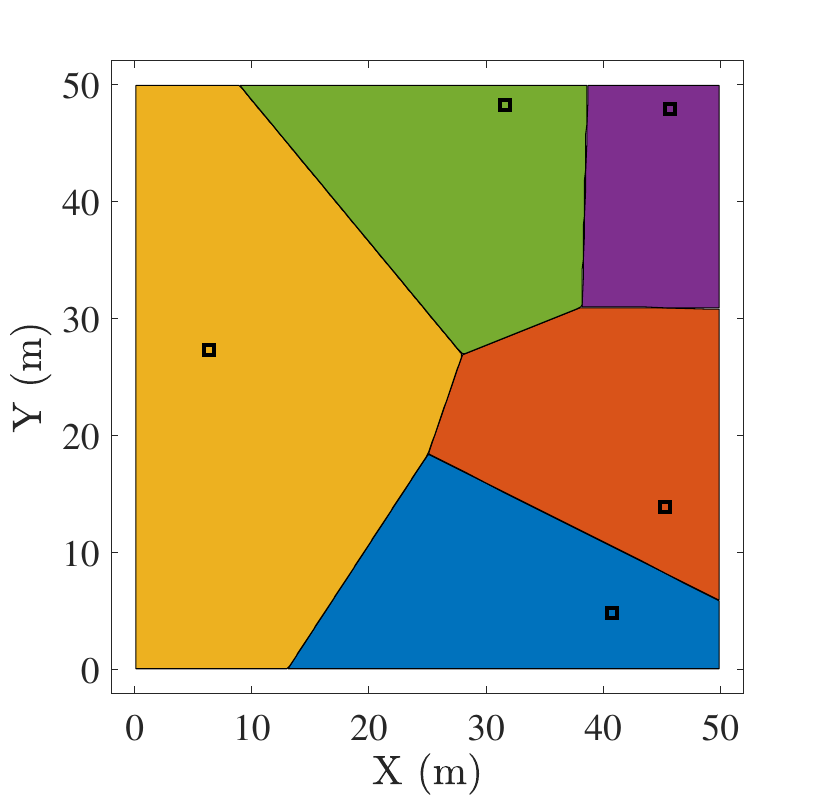}
            \caption{$\alpha=0$.}
            \label{fig:Voronoia0}
        \end{subfigure}
        \begin{subfigure}{0.245\linewidth}
            \centering
            \includegraphics[width=\linewidth]{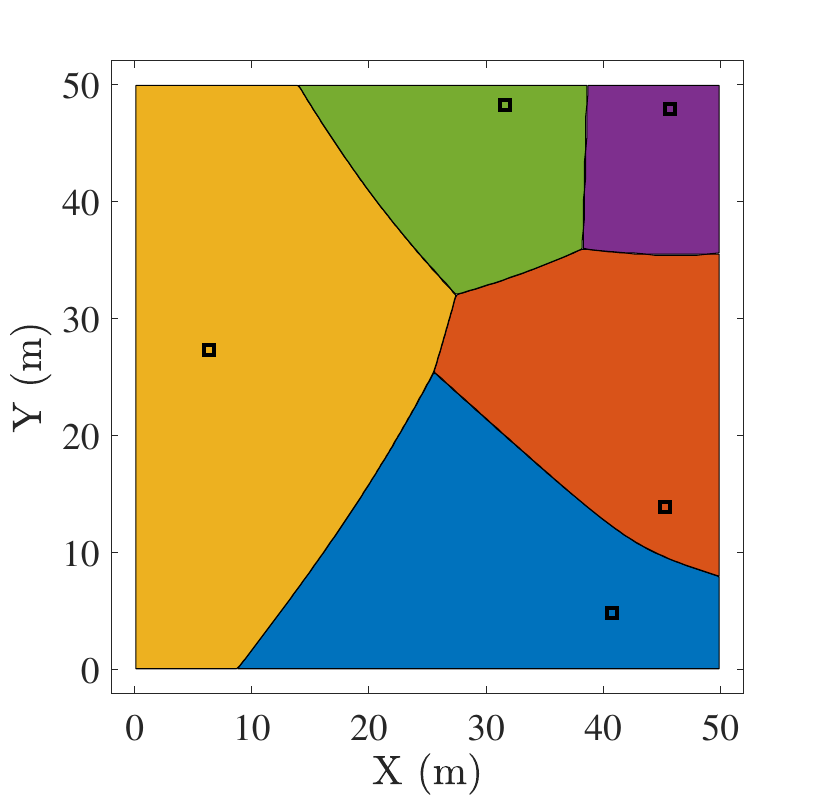}
            \caption{$\alpha=-0.5$.}
            \label{fig:Voronoia50}
        \end{subfigure}
        \begin{subfigure}{0.245\linewidth}
            \centering
            \includegraphics[width=\linewidth]{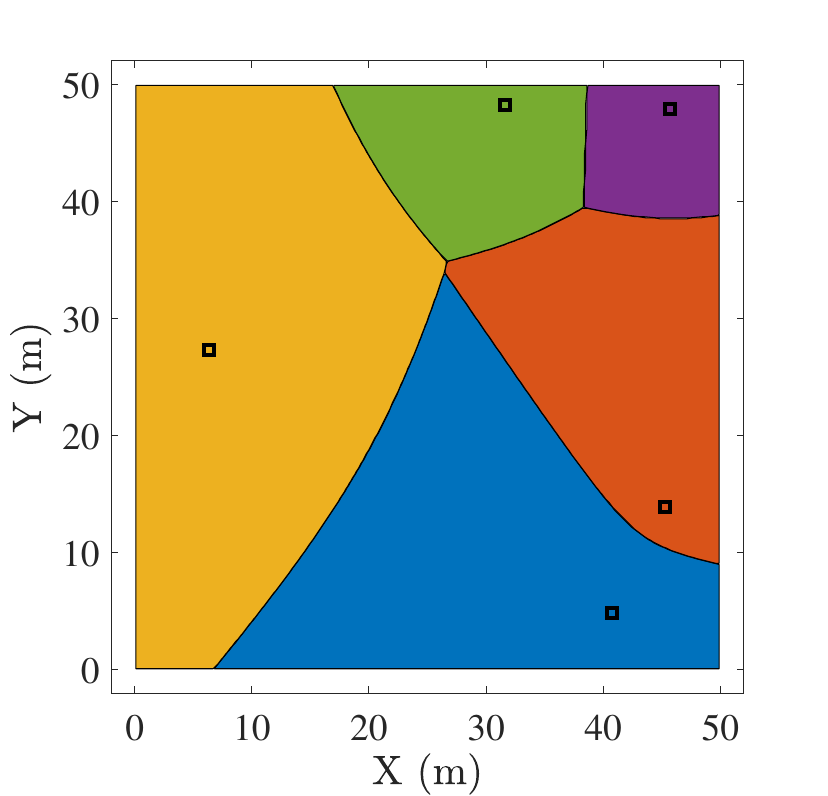}
            \caption{$\alpha=-0.75$.}
            \label{fig:Voronoia75}
        \end{subfigure}
        \begin{subfigure}{0.245\linewidth}
            \centering
            \includegraphics[width=\linewidth]{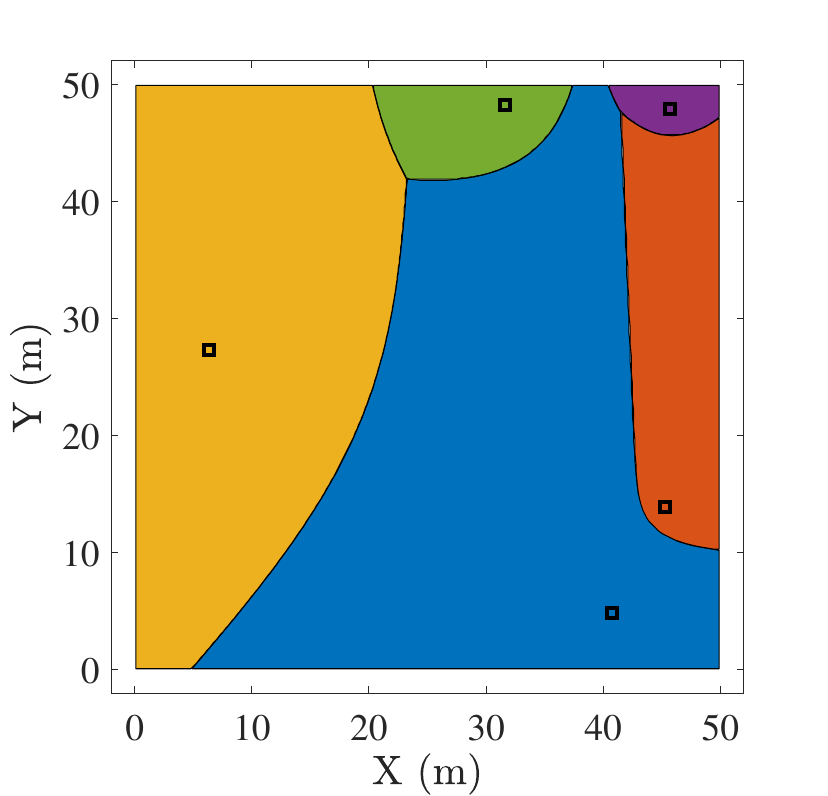}
            \caption{$\alpha=-0.99$.}
            \label{fig:Voronoia99}
        \end{subfigure}
        \caption{Illustration of the proposed generalised Voronoi tessellation for different values of $\alpha$. The wind heading is $\psi=-90^\circ$ and a square marker is used to denote the robot locations.}
        \label{fig:voronoi_example}
    \end{figure}

	\section{Path Planning and Control Policy Design}
	
	In this section, we introduce the multi-robot coordination algorithm that enables collaborative sampling to estimate the unknown parameters of multi-sources in the operational region $\mathcal{D}$. \tb{We assume that the number of available agents is larger than or equal to the maximum number of leaks in the search domain, that is, $n\geq M_{\max}$.}
 
    It should be noted that under the the proposed generalised Voronoi tessellation, $\mathcal{V}(\{\vect{p}^{(i)}\}_{1:n},f,\mathcal{D})$, the dominance regions $\mathcal{V}^{(i)}$ are determined once $\vect{p}^{(i)}~\forall~i\in\dsqrbrac{1:n}$ are fixed. In other words, for $\mathcal{W}=\mathcal{V}$, the locational optimization objective function \eqref{eq:obj1} is only a function of robot position. This result is formally stated and proved in Theorem \ref{thm:proveobj2}.
	\begin{theorem}\label{thm:proveobj2}
	    Let $\mathcal{H}_\mathcal{V}$ be defined as 
     \begin{align}\label{eq:obj2}
         \mathcal{H}_\mathcal{V}(\vect{p}^{(i)}) &= \sum_{i=1}^{n}\int_{\mathcal{V}^{(i)}}f_i(\vect{q})\phi(\vect{q})d\vect{q}\,,
     \end{align}
     where $\mathcal{V}^{(i)}\in\mathcal{V}(\{\vect{p}^{(i)}\}_{1:n},f,\mathcal{D})$ are the regions of dominance defined as per \eqref{eq:Voronoi}. For a given density function $\phi(.)$ defined over $\mathcal{D}$, the objective function defined in \eqref{eq:obj1}, $\mathcal{H}$, and $\mathcal{H}_\mathcal{V}$ have the same minimizer. 
	\end{theorem}
     \begin{proof}
         For some $\vect{p}^{(i)}$'s, let us choose an arbitrary tessellation $\mathcal{W} = \{\mathcal{W}_1,\ldots,\mathcal{W}_n\}$ such that $\mathcal{W}\neq\mathcal{V}$. 
         Consider an arbitrary point $\vect{q}\in\mathcal{D}$ such that it lies dominance region of the $i^{th}$ and $j^{th}$ robot as per the tessellation $\mathcal{V}$ and $\mathcal{W}$, respectively. Therefore, one can write $f_j(\vect{q})\phi(\vect{q})\geq f_i(\vect{q})\phi(\vect{q})\,.$
         Since, $\mathcal{W}$ is not a Voronoi tessellation over $\mathcal{D}$, $f_j(\vect{q})\phi(\vect{q})>f_i(\vect{q})\phi(\vect{q})$ must hold over some measurable set of $\mathcal{D}$. Thus,
         	\begin{align*}
        		\mathcal{H}(\vect{p}^{(i)},\mathcal{W}) &= \sum_{i=1}^{n}\int_{\mathcal{W}^{(i)}}f_i(\vect{q})\phi(\vect{q})d\vect{q} > \mathcal{H}_\mathcal{V}(\vect{p}^{(i)})\,.
        	\end{align*}
         Therefore, it can be concluded that the objective function \eqref{eq:obj1} is minimized with respect to the tessellation $\mathcal{W}$ only if the $\mathcal{W}=\mathcal{V}$. Therefore, minimizing \eqref{eq:obj1} is equivalent to minimizing $\mathcal{H}_\mathcal{V}$ defined in \eqref{eq:obj2}. This completes the proof.
     \end{proof}
    \begin{remark}\label{rem:Hvpostivedef}
        The objective function $\mathcal{H}_\mathcal{V}>0$ for $\vectnot{p}{}{i}\in\mathcal{D}$ as  $\phi>0$ for all $\vect{q}\in\mathcal{D}$ and $f_i(\vect{q})>0$ for all $\vect{q}\neq\vectnot{p}{}{i}$.
    \end{remark}
     We now derive a gradient descent control policy that locally minimizes the objective function $\mathcal{H}_\mathcal{V}$ with respect to the robot locations.
     
	\begin{theorem}\label{thm:criticalpoints}
	Let the terms $\hat{\mathcal{M}}_i$, $\bar{\mathcal{M}}_i$ and $\hat{\mathcal{C}}_i$ be defined as
    \begin{align}
        \cos\eta_i &= \frac{(\vectnot{p}{}{i}-\vect{q})\cdot\vec{e}_{\psi}}{||\vectnot{p}{}{i}-\vect{q}||}\\
        \hat{\mathcal{M}}_i &= \int_{\mathcal{V}_i}[2+\alpha\cos\eta_i]\phi(\vect{q})d\vect{q}\\
        \bar{\mathcal{M}}_i &= \int_{\mathcal{V}_i}||\vect{p}^{(i)} - \vect{q}||\phi(\vect{q})d\vect{q}\\
        \hat{\mathcal{C}}_i &= \frac{1}{\hat{\mathcal{M}}_i}\int_{\mathcal{V}_i}\vect{q}[2+\alpha\cos\eta_i]\phi(\vect{q})d\vect{q}\,.
    \end{align}
    Under the sensing performance function defined in \eqref{eq:sensingperform}, the critical position of the $i^{th}$ robot in the resulting generalised centroidal Voronoi configuration satisfies 
	\begin{align}
    	\vect{p}^{(i)}_\star &= \hat{\mathcal{C}}_{i} - \left[\frac{\alpha\bar{\mathcal{M}}_{i}}{\hat{\mathcal{M}}_{i}}\right]\vec{e}_{\psi}\label{eq:critical_points_p}\,.
    \end{align}
	\end{theorem}
	\begin{proof}
    In order to find the condition of optimality that minimizes the cost function $\mathcal{H}_\mathcal{V}$, we set the partial derivative of \eqref{eq:obj2} with respect to $\vect{p}^{(i)}$ to zero. Differentiating \eqref{eq:obj2} and applying the Leibniz rule \cite{Flanders1973}, we get
	\begin{align}\label{eq:partialder0}
	\frac{\partial \mathcal{H}_\mathcal{V}}{\partial \vect{p}^{(i)}} &= \frac{\partial }{\partial \vect{p}^{(i)}}\int_{\mathcal{V}^{(i)}}f_i(\vect{q})\phi(\vect{q})d\vect{q} + \frac{\partial }{\partial \vect{p}^{(i)}}\sum_{i\neq j}\int_{\mathcal{V}^{(j)}}f_j(\vect{q})\phi(\vect{q})d\vect{q}\,,\nonumber  \\
	     &= \int_{\mathcal{V}^{(i)}}\frac{\partial f_i(\vect{q})}{\partial \vect{p}^{(i)}}\phi(\vect{q})d\vect{q} + \int_{\partial\mathcal{V}^{(i)}} f_i(\vect{q})\phi(\vect{q})\left[\frac{\partial \vect{q}}{\partial \vect{p}^{(i)}}\cdot \tilde{\mathsf{n}}^{(i)}(\vect{q})\right]dl\nonumber\\
      &+ \sum_{j\neq i}\int_{\partial\mathcal{V}^{(j)}} f_j(\vect{q})\phi(\vect{q})\left[\frac{\partial \vect{q}}{\partial \vect{p}^{(i)}}\cdot \tilde{\mathsf{n}}^{(j)}(\vect{q})\right]dl\,,
	\end{align}
    where $dl$ is the infinitesimal length and $\tilde{\mathsf{n}}^{(i)}(\vect{q})$ is the outward normal vector to the curve at $\vect{q}\in\partial\mathcal{V}^{(i)}$. The dot product $\left[\frac{\partial \vect{q}}{\partial \vect{p}^{(i)}}\cdot \tilde{\mathsf{n}}^{(i)}(\vect{q})\right]$ is the velocity of the curve at $\vect{q}\in\partial\mathcal{V}^{(i)}$ along $\tilde{\mathsf{n}}^{(i)}(\vect{q})$. Note that $\partial \mathcal{H}_\mathcal{V}/\partial \vect{p}^{(i)}$ is a column vector of size two. In order to evaluate the second and third term of \eqref{eq:partialder0}, we need to express the curve velocity at different parts of the boundaries $\partial \mathcal{V}^{(i)}$. Therefore, we express $\partial \mathcal{V}^{(i)} = \mathcal{E}_{i\mathcal{D}}\bigcup\mathcal{E}_{ij};$ where $\mathcal{E}_{i\mathcal{D}} = \partial\mathcal{V}^{(i)}\bigcap\partial\mathcal{D}$ and $\mathcal{E}_{ij} = \partial\mathcal{V}^{(i)}\bigcap\partial\mathcal{V}^{(j)}~\forall~j\neq i$. Therefore, we can rewrite \eqref{eq:partialder0} as
    \begin{align}\label{eq:partialder1}
	   \frac{\partial \mathcal{H}_\mathcal{V}}{\partial \vect{p}^{(i)}}
	     &= \int_{\mathcal{V}^{(i)}}\frac{\partial f_i(\vect{q})}{\partial \vect{p}^{(i)}}\phi(\vect{q})d\vect{q} + \int_{\mathcal{E}_{i\mathcal{D}}} f_i(\vect{q})\phi(\vect{q})\left[\frac{\partial \vect{q}}{\partial \vect{p}^{(i)}}\cdot \tilde{\mathsf{n}}^{(i)}(\vect{q})\right]dl\nonumber\\
      &+ \sum_{j\neq i}\int_{\mathcal{E}_{ij}} f_i(\vect{q})\phi(\vect{q})\left[\frac{\partial \vect{q}}{\partial \vect{p}^{(i)}}\cdot \tilde{\mathsf{n}}^{(i)}(\vect{q})\right]dl + \sum_{j\neq i}\int_{\mathcal{E}_{ij}} f_j(\vect{q})\phi(\vect{q})\left[\frac{\partial \vect{q}}{\partial \vect{p}^{(i)}}\cdot \tilde{\mathsf{n}}^{(j)}(\vect{q})\right]dl\,.
	\end{align}
    It should be noted that $\mathcal{E}_{ij} = \varnothing$ for all robots $j$ that are not the Voronoi  neighbours of robot $i$. Furthermore, it can be easily concluded that for all $j$ that are Voronoi neighbours to $i$ as per \eqref{eq:Voronoi}, the following holds $$\left[\frac{\partial \vect{q}}{\partial \vect{p}^{(i)}}\cdot \tilde{\mathsf{n}}^{(i)}(\vect{q})\right] = -\left[\frac{\partial \vect{q}}{\partial \vect{p}^{(i)}}\cdot \tilde{\mathsf{n}}^{(j)}(\vect{q})\right]\,,$$
    since the points $\vect{q}\in\mathcal{E}_{ij}$ move at the same speed in opposite directions. Therefore, the sums in third and fourth term of \eqref{eq:partialder1} cancel out. Moreover, for an infinitesimal motion of the $i^{th}$ robot, it can be concluded that $\partial\vect{q}^{(i)}/\partial\vect{p}^{(i)} = 0~\forall~\vect{q}\in\mathcal{E}_{i\mathcal{D}}$. Therefore, the second term in \eqref{eq:partialder1} also equals to zero. Hence, the partial derivative $\partial\mathcal{H}_\mathcal{V}/\partial \vect{p}^{(i)}$ can be simplified to
    \begin{align}\label{eq:Hpartial_der_p}
	   \frac{\partial \mathcal{H}_\mathcal{V}}{\partial \vect{p}^{(i)}}
	     &= \int_{\mathcal{V}^{(i)}}\frac{\partial f_i(\vect{q})}{\partial \vect{p}^{(i)}}\phi(\vect{q})d\vect{q}\,.
	\end{align}
    The partial derivative of $f$ with respect to $\vect{p}^{(i)}$ can be written, using \eqref{eq:sensingperform}, as 
    \begin{align}\label{eq:fpartial}
        \frac{\partial f_i(\vect{q})}{\partial \vect{p}^{(i)}} &= [2+\alpha\cos\eta_i](\vect{p}^{(i)}-\vect{q}) + \alpha||\vect{p}^{(i)}-\vect{q}||\vec{e}_{\psi}\,.
    \end{align}
    Substituting \eqref{eq:fpartial} in \eqref{eq:Hpartial_der_p} and simplifying, we get
    \begin{align}
        \frac{\partial \mathcal{H}_\mathcal{V}}{\partial \vect{p}^{(i)}}
	     &= \int_{\mathcal{V}^{(i)}}(\vect{p}^{(i)}-\vect{q})\left[2+\alpha\cos\eta_i\right]\phi(\vect{q})d\vect{q}\nonumber\\
      &~~+ \int_{\mathcal{V}^{(i)}}\vec{e}_{\psi}\alpha||\vect{p}^{(i)}-\vect{q}||\phi(\vect{q})d\vect{q}\nonumber\\
      & = \hat{\mathcal{M}}_i(\vect{p}^{(i)}-\hat{\mathcal{C}}_{i}) + \alpha\bar{\mathcal{M}}_i\vec{e}_{\psi}\nonumber\\
      & = \hat{\mathcal{M}}_i\left(\vect{p}^{(i)}-\hat{\mathcal{C}}_{i} + \left[\frac{\alpha\bar{\mathcal{M}}_i}{\hat{\mathcal{M}}_i}\right]\vec{e}_{\psi}\right)\nonumber\\
      & = \hat{\mathcal{M}}_i\left(\vect{p}^{(i)}-\vectnot{p}{\star}{i}\right)\label{eq:Hpartial_der_pfinal}
    \end{align}
    Equating \eqref{eq:Hpartial_der_pfinal} to zero, we can see that the critical position of the $i^{th}$ robot is given by \eqref{eq:critical_points_p}. This completes the proof.
	\end{proof}
    \begin{remark}\label{rem:convexD}
        As the Voronoi cell $\mathcal{V}^{(i)}$ is not convex, the critical point $\vectnot{p}{\star}{i}$ may not lie inside it. However, $\vectnot{p}{\star}{i}\in\mathcal{D}$ throughout the mission as $V^{(i)}\subseteq\mathcal{D}$ and $\mathcal{D}$ is a convex region. 
    \end{remark}
    In order to drive the robots to the respective critical locations, we design the linear and angular speeds of the $i^{th}$ robot as
    \begin{subequations}\label{eq:controlInputs}
    \begin{align}
        v^{(i)} &= \begin{cases}
            0 & \text{if }0<\scalnot{v}{\parallel}{i}\\
            -k_v\scalnot{v}{\parallel}{i} & \text{if }-v_{max}\leq k_v\scalnot{v}{\parallel}{i}\leq 0\\
            v_{max} & \text{if }k_v\scalnot{v}{\parallel}{i}<-v_{max}
        \end{cases}\\[2mm]
        \omega^{(i)} &= \begin{cases}
            k_\omega \tan^{-1}\left(\dfrac{-\scalnot{v}{\perp}{i}}{-\scalnot{v}{\parallel}{i}}\right) & \text{if }||\vectnot{p}{}{i}-\vectnot{p}{\star}{i}||\neq0\\
            0 & \text{if }||\vectnot{p}{}{i}-\vectnot{p}{\star}{i}||=0
        \end{cases}
    \end{align}
    \end{subequations}
    where $\scalnot{v}{\parallel}{i}=\vec{e}_{\scalnot{\gamma}{}{i}}\cdot(\vectnot{p}{}{i}-\vectnot{p}{\star}{i})$, $\scalnot{v}{\perp}{i}=\vec{e}_{\perp\scalnot{\gamma}{}{i}}\cdot(\vectnot{p}{}{i}-\vectnot{p}{\star}{i})$, and the positive constants, $k_v$ and $k_{\omega}$, are proportional gains. 
    
    Consider the sets $\mathcal{Q}$ and $\mathcal{E}$ defined as 
    \begin{align}
        \mathcal{Q}&\triangleq\{\vectnot{p}{}{i}~|~\scalnot{v}{\parallel}{i}\geq0 \text{ and }||\vectnot{p}{}{i}-\vectnot{p}{\star}{i}||\neq0 \}\\
        \mathbb{E}&\triangleq\{\vectnot{p}{\star}{i}~|~\scalnot{v}{}{i},\scalnot{\omega}{}{i}=0\}\,.
    \end{align}
    From \eqref{eq:controlInputs}, $\omega^{(i)}\neq0$ for all $\vectnot{p}{}{i}\in\mathcal{Q}$. Therefore, for all $\vectnot{p}{}{i}\in\mathcal{Q}$, the robot heading (unit vector $\vec{e}_{\scalnot{\gamma}{}{i}}$) rotates over time to eventually cause $\scalnot{v}{\parallel}{i}<0$, consequently, resulting in $\vectnot{p}{}{i}\notin\mathcal{Q}$ after some time. 
    Thus, the set $\mathcal{Q}$ does not represent an invariant set as the robots haven't converged on their respective critical points. Moreover, it can noted that for all $\vectnot{p}{}{i}\in\mathbb{E}$, both $\scalnot{v}{}{i}$ and $\scalnot{\omega}{}{i}$ are zero, resulting in $\mathbb{E}$ being the largest invariant set in $\mathcal{D}$.    

    \begin{theorem}\label{thm:convergence}
	The trajectories resulting from the control commands given by \eqref{eq:controlInputs} asymptotically drive the each robots to their critical point \eqref{eq:critical_points_p} while minimizing $\mathcal{H}_\mathcal{V}$ in a non-increasing manner.
	\end{theorem}
    \begin{proof}
        Consider the Lyapunov function candidate $\mathcal{H}_\mathcal{V}(\vectnot{p}{}{i})$ for all $\vect{q}\in\mathcal{D}$ (refer to Remark \ref{rem:Hvpostivedef}). Thus, the time derivative of $\mathcal{H}_\mathcal{V}(\vectnot{p}{}{i})$ can be expressed as
        \begin{align}
            \frac{d\mathcal{H}_\mathcal{V}}{dt} &= \sum_{i=1}^{n}\left[\frac{\partial \mathcal{H}}{\partial \vectnot{p}{}{i}}\cdot\frac{d\vectnot{p}{}{i}}{dt}\right] = \sum_{i=1}^{n}\hat{\mathcal{M}}_i\scalnot{v}{}{i}\scalnot{v}{\parallel}{i}\,,
        \end{align}
        which can be further simplified using \eqref{eq:controlInputs} as
        \begin{align}\label{eq:HvDot}
            \frac{d\mathcal{H}_\mathcal{V}}{dt} &= \begin{cases}
            0 & \text{if }0<\scalnot{v}{\parallel}{i}\\
            -k_v\hat{\mathcal{M}}_i\left[\scalnot{v}{\parallel}{i}\right]^2 & \text{if }-v_{max}\leq k_v\scalnot{v}{\parallel}{i}\leq 0\\
            -\hat{\mathcal{M}}_i\dfrac{\left[v_{max}\right]^2}{k_v} & \text{if }k_v\scalnot{v}{\parallel}{i}<-v_{max}
        \end{cases}\,.
        \end{align}
        At some time $t$, let $\vectnot{p}{}{i}\in\partial\mathcal{D}$ and heading outward from $\mathcal{D}$, the relative velocity $\scalnot{v}{\parallel}{i}>0$ as $\vectnot{p}{\star}{i}\in\mathcal{D}$ (refer to Remark \eqref{rem:convexD}). Thus, as per \eqref{eq:controlInputs}, $\scalnot{v}{}{i}=0$ for such a scenario. Hence, the region $\mathcal{D}$ is invariant under the proposed control design of \eqref{eq:controlInputs}. Furthermore, it can be shown that $\mathcal{H}_\mathcal{V}$ is continuously differentiable with respect to $\{\vectnot{p}{}{i}\}$. Thus, as $\dot{\mathcal{H}}_{\mathcal{V}}\leq0$ in $\mathcal{D}$ from \eqref{eq:HvDot} and $\vectnot{p}{}{i}\in\mathbb{E}$ is the largest invariant set, it can be concluded using LaSalle's in-variance principle that all $\vectnot{p}{}{i}(0)\in\mathcal{D}$ will asymptotically converge to $\vectnot{p}{\star}{i}$ of the generalised Voronoi partition and minimize $\mathcal{H}_\mathcal{V}$ in a non-increasing manner under the control action \eqref{eq:controlInputs}.
    \end{proof}

    The sensing platforms perform environmental sampling when each robot in the network converges to their corresponding stationary points $\vectnot{p}{\star}{i}$. Moreover, the search is assumed to be concluded when the normed uncertainty of the estimated source term, computed as 
    \begin{align}\label{eq:normedSigma}
        \sigma_{\Theta_k} = \sqrt{\sum_{j\in\mathcal{J}}\sum_{p=1}^{N_p}\mathsf{b}_{p,j}\weight{k}{p}\left|\left|\scalnot{\theta}{k,j}{p} - \bar{\theta}_{k,j}\right|\right|^2}\,,
    \end{align} is below a certain threshold, $\sigma_{th}$. The Overall proposed multi-robot multi-source term estimation strategy is summarised in Algorithm \ref{algo:strategy}.
    
    \begin{algorithm}[!ht]
        \DontPrintSemicolon
        \KwIn{\small
              Initial robot position: $\vectnot{p}{0}{i}$,
              Input Domain: $\mathcal{D}$,
              Sampling instant counter: $k=1$,
              }
              \BlankLine
        \KwOut{\small Source Term Estimate: $\bar{\theta}_{k}$ and $\bar{M}_k$}
        \BlankLine
        Initialise particle filter with $\{\weight{0}{p},\vectnot{X}{0}{p}\}$\;
        Perform environment sampling action for $\scalnot{z}{k}{:}$\;
        Belief Update: $p(\vect{X}_{k}|\vect{Z}_{k}) = \{\weight{k}{p},\vectnot{X}{k}{p}\}$ as per Algorithm \ref{alg:one_step_pf}\;
        \BlankLine
        \While{$\sigma_{\Theta_k}>\sigma_{th}$}{
        Compute the density function $\phi_k$ as per \eqref{eq:densityfunction}\;
        Construct the generalised Voronoi cell $\mathcal{V}_i$ using Eq. \eqref{eq:Voronoi}\;
        Compute $\scalnot{v}{}{1:n},\scalnot{\omega}{}{1:n}$ using \eqref{eq:controlInputs}\;
        \While{$\sum_{i}(\scalnot{v}{}{i}+|\scalnot{\omega}{}{i}|)\approx0$}{
            Construct the generalised Voronoi cell $\mathcal{V}_i$ using Eq. \eqref{eq:Voronoi}\;
            Compute $\scalnot{v}{}{1:n},\scalnot{\omega}{}{1:n}$ using \eqref{eq:controlInputs}\;
            Update robot positions,~~$\vectnot{p}{t}{1:n}$ as per \eqref{eq:dynamics}\;
        }
        Set $k = k+1$\;
        Perform environment sampling action for $\scalnot{z}{k}{:}$\;
        Belief Update: $p(\vect{X}_{k}|\vect{Z}_{k}) = \{\weight{k}{p},\vectnot{X}{k}{p}\}$ as per Algorithm \ref{alg:one_step_pf}\; 
        Compute $\sigma_{\Theta_k}$ using \eqref{eq:normedSigma}\;
        }
        Compute source term estimate as per Section \ref{sec:estimatecompute} \;
        \caption{Multi-robot Multi-Source Term Estimation (MR-MSTE) Path Planner.}
        \label{algo:strategy}
    \end{algorithm}

\begin{figure*}[!ht]
    \centering
    \captionsetup[subfigure]{justification=centering}
    \begin{subfigure}{0.4\linewidth}
        \includegraphics[width=\linewidth]{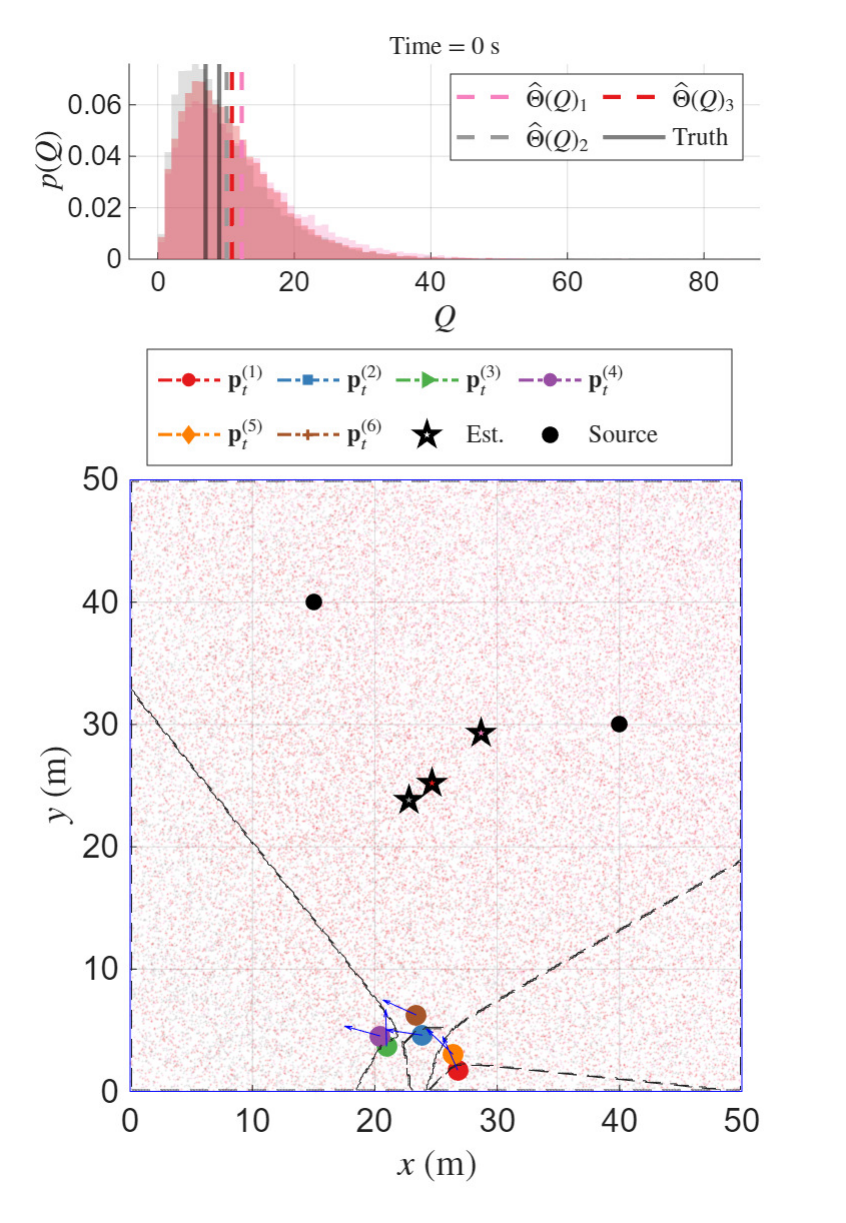}
        \caption{}
        \label{fig:Illustrativefig1}
    \end{subfigure}
    \begin{subfigure}{0.4\linewidth}
        \includegraphics[width=\linewidth]{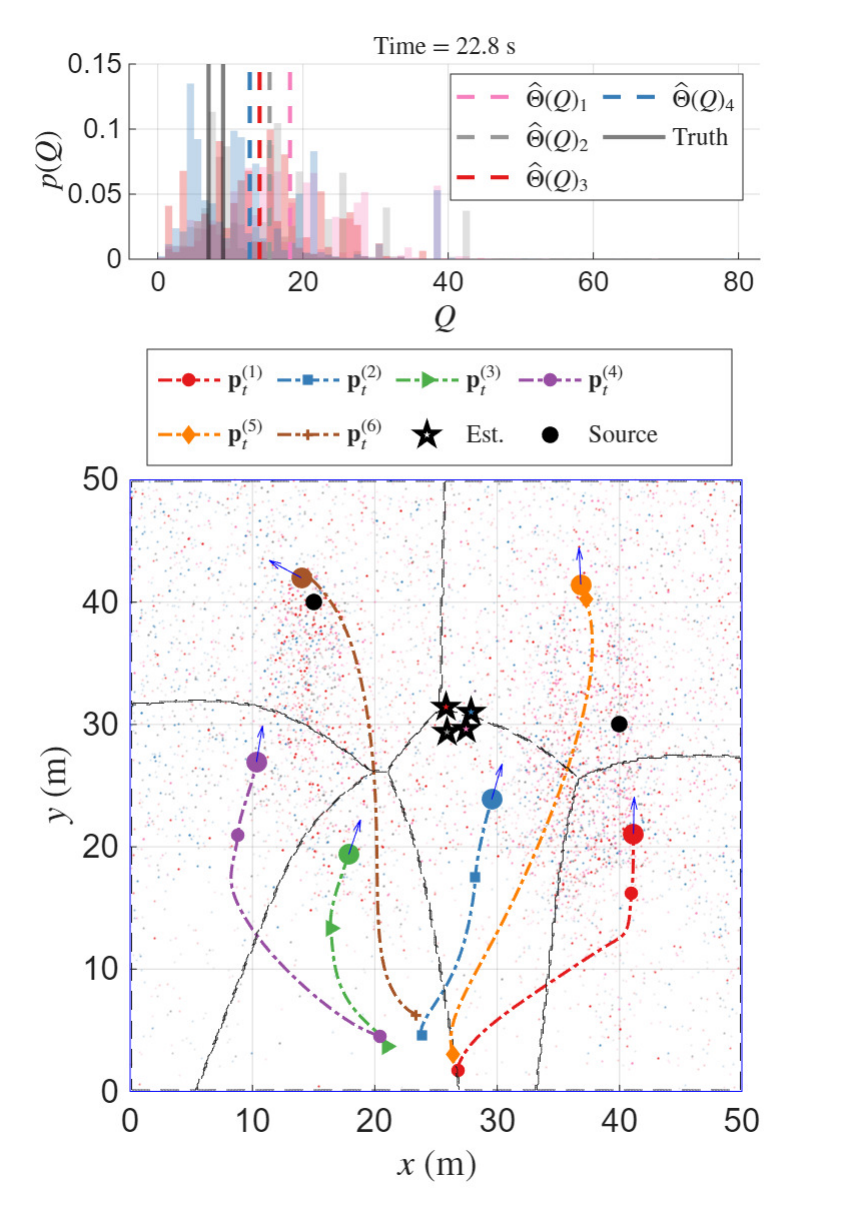}
        \caption{}
        \label{fig:Illustrativefig185}
    \end{subfigure}
    \begin{subfigure}{0.4\linewidth}
        \includegraphics[width=\linewidth]{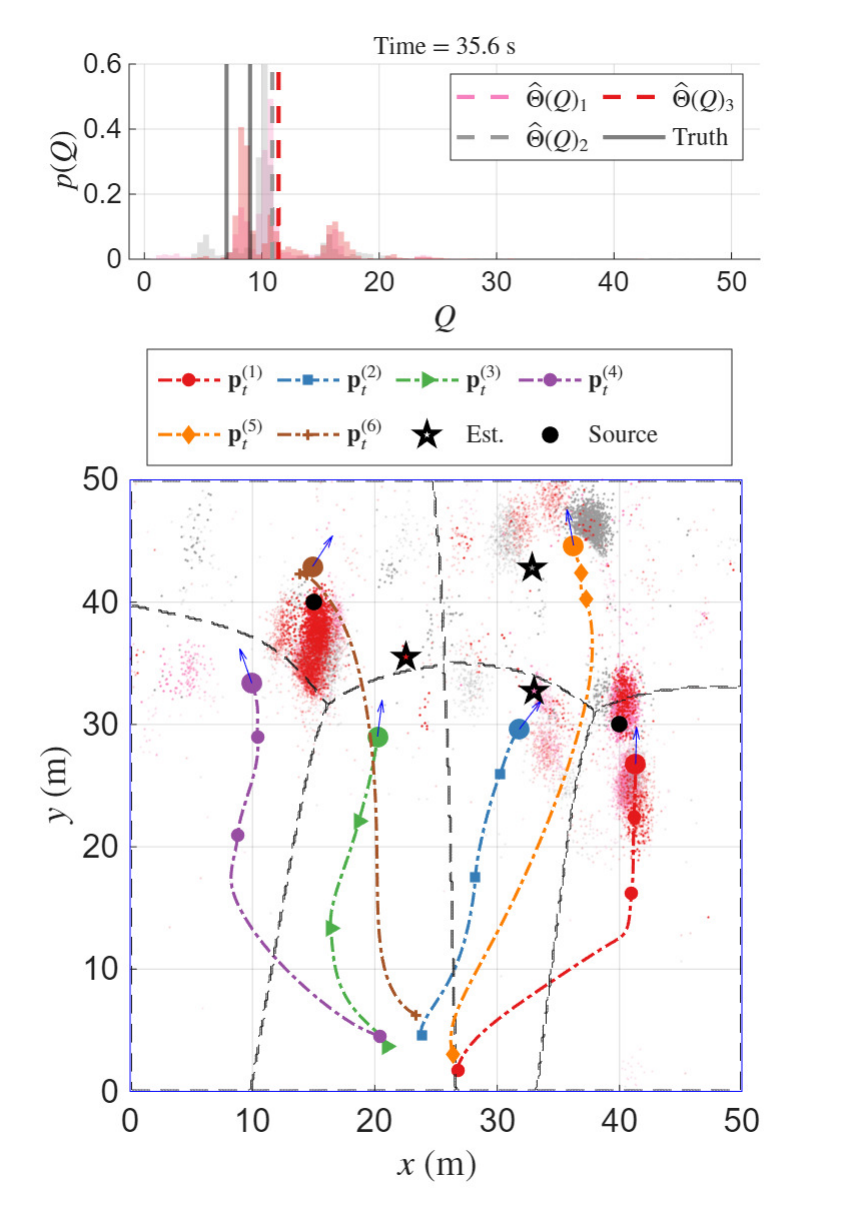}
        \caption{}
        \label{fig:Illustrativefig418}
    \end{subfigure}
    \begin{subfigure}{0.4\linewidth}
        \includegraphics[width=\linewidth]{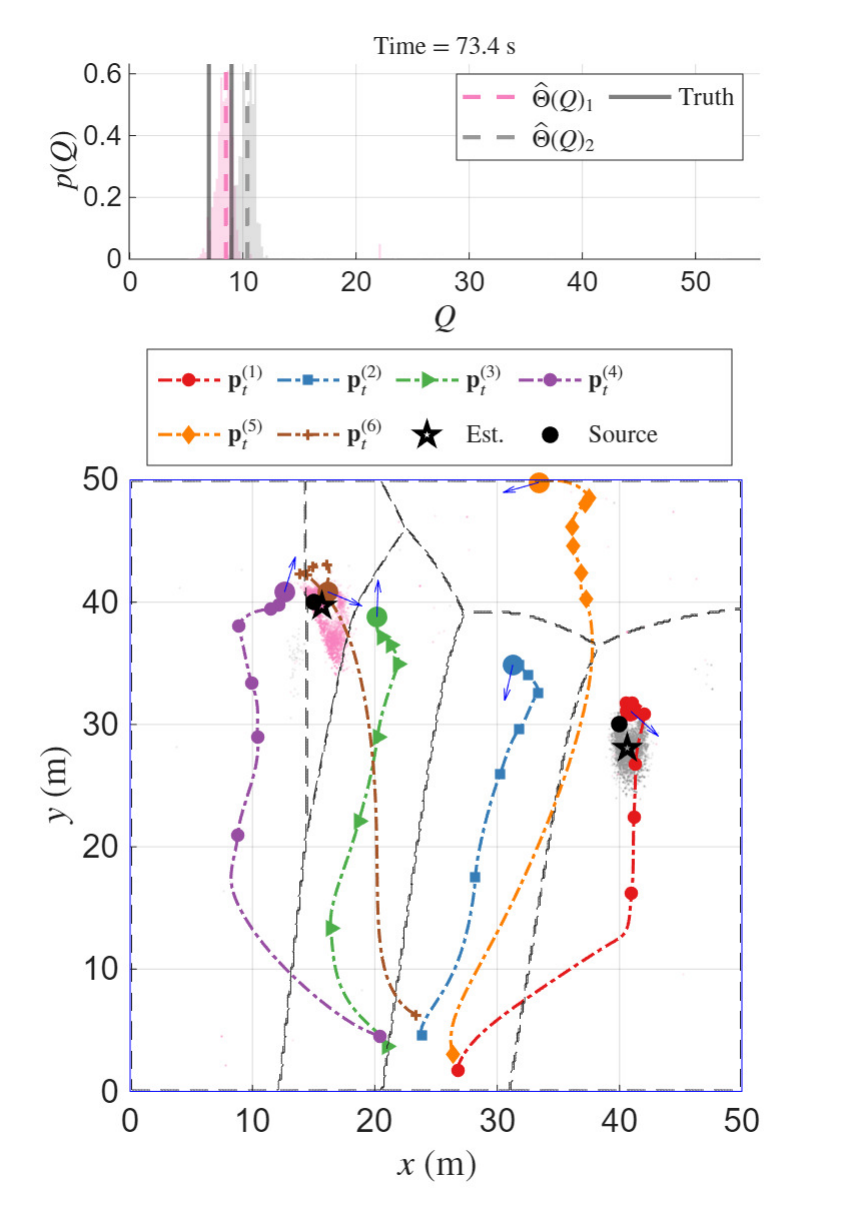}
        \caption{}
        \label{fig:Illustrativefig626}
    \end{subfigure}
    \caption{Illustrative run for the proposed wind-aware coverage control based MR-MSTE algorithm with $n=6$, $M_{\max}=4$, $\alpha=-0.75$, $v_w = 4m/s$ and $\psi=90^\circ$ and the black dashed line representing the generalized Voronoi cells corresponding to each robot.}
    \label{fig:Illustrative_Run}
\end{figure*}

\section{Simulation Study}
\tb{In this section, we first present a illustrative study followed by a comprehensive Monte Carlo simulation study to demonstrate the advantages of the proposed algorithm  over the traditional coverage control for multi-source term estimation tasks. As considered in simulation studies of \cite{Park2021,PARK202072,JANG2023120033}, it is assumed that the environment parameters are known and fixed during each run.}

\begin{figure}
    \centering
    \captionsetup[subfigure]{justification=centering}
    \begin{subfigure}{0.45\linewidth}
        \includegraphics[width=\linewidth]{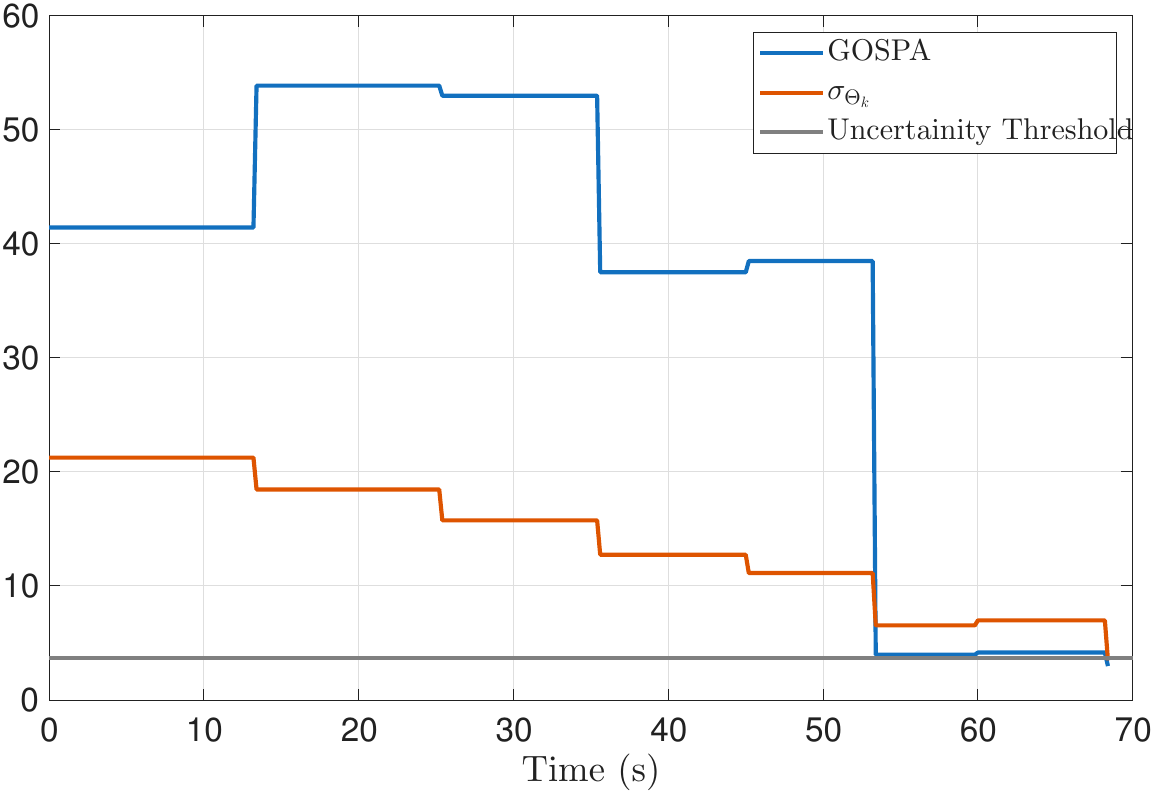}
        \caption{}
        \label{fig:rmse_uncertainity}
    \end{subfigure}
    \begin{subfigure}{0.45\linewidth}
        \includegraphics[width=\linewidth]{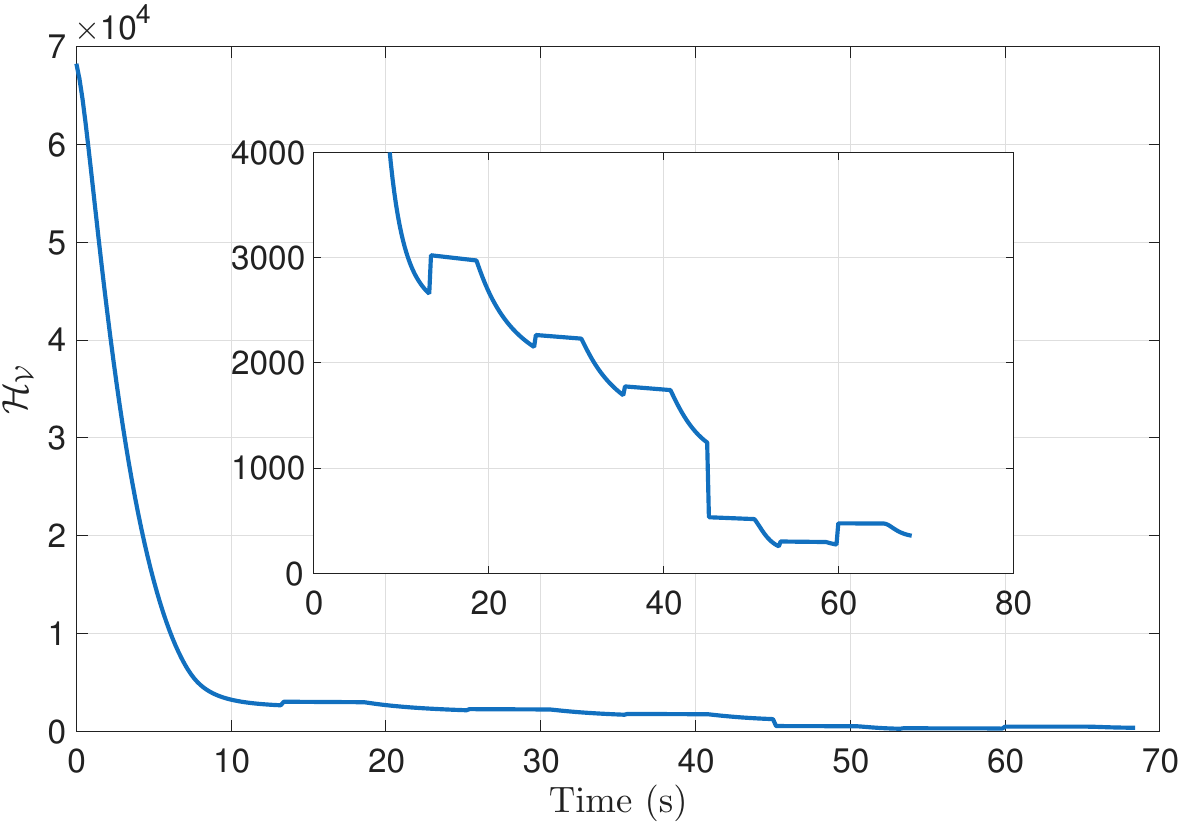}
        \caption{}
        \label{fig:ObjectiveFunction}
    \end{subfigure}
    \caption{Variation of (\subref{fig:rmse_uncertainity}) GOSPA error along with $\sigma_{\Theta_k}$, and, (\subref{fig:ObjectiveFunction}) $\mathcal{H}_{\mathcal{V}}$ for the illustrative run presented in Fig. \ref{fig:Illustrative_Run}. }
    \label{fig:performanceIndex}
\end{figure}

\subsection{Illustrative Study}
\tb{We consider an illustrative study with 6 robots deployed in the lower left corner of an open $50m\times50m$ region. The ground truth consist of two sources located at $15m$ east-$40m$ north and $40m$ east-$30m$ north of the left bottom corner (refereed to as the origin from here on). Moreover, the sources are assumed to be releasing chemical gas at $7g/s$ and $9g/s$ dispersed by a wind flow along the negative Y-axis at a speed of $4m/s$. The diffusivity and particle lifetime of the gas is chosen to be $1.2m^2/s$ and $5s$, respectively.} 

\tb{The prior belief for the individual source position is assumed to be a uniform distribution over $\mathcal{D}$, that is, $\mathcal{U}(\mathcal{D})$. The release rate prior for each source is chosen as a Gamma distribution $\Gamma(2,5)$. The number of particles is set to be $N_p=25000$. Furthermore, the robots are constrained to a maximum linear and angular speed of $4~m/s$ and $2.25~rad/s$, respectively. The sensor threshold $z_{thr}$ is chosen as 0.5$mg/m^3$. The simulation results, presented in Fig. \ref{fig:Illustrative_Run}, consist of trajectory plots for each robot, sampled locations denoted using markers, belief of the source locations, and the estimated source term belief at different time instants. The parameters $\alpha$ and $\sigma_{th}$ are chosen to be -0.75 and 4 units, respectively. Furthermore, a stop-and-sample time of 5 seconds is also considered during each sampling instance. Moreover, the birth and death probabilities are assumed to be 8\% each along with a miss-detection probability of 5\%.}

\tb{It can be seen from Figs. \ref{fig:Illustrative_Run} that the two unknown source locations and release rates are estimated successfully with sufficient accuracy. All robots required 8 collective environmental sampling instance during the search to localise the sources, with an overall mission time of 73.4 seconds. Additionally, it is important to note that the coverage control based design drives the robots to upwind regions within their Voronoi partition that are likely to contain a source. In traditional informative path planning solutions, the design may require explicitly incorporating robot-source assignment at the cost of computational and design complexity, otherwise some robots may be trapped within local plumes. In contrast, under the proposed coverage control based strategy, the robots in the network are inherently assigned to potential source locations as the search progresses. Such an feature is a consequence of the density function design choice, as robots are driven to high density areas based on JMPD. Furthermore, it can be seen from Fig. \ref{fig:ObjectiveFunction} that $\mathcal{H}_\mathcal{V}$, \eqref{eq:obj2}, is minimised as the search progresses. Upon updating the source term belief, the objective function cost may increase as seen from Fig. \ref{fig:ObjectiveFunction}. Recall that the proposed path planning policy only guarantees reducing in $\mathcal{H}_\mathcal{V}$ between sampling instance. However, as more samples are incorporated to improve the source term belief, the standard deviation of source term localisation is likely to reduce, consequently, causing the robots to converge near the estimated source locations (high density regions) towards the end of the search. Thus ensuring $\mathcal{H}_\mathcal{V}$ is reduced over the span of a mission. }

\begin{figure*}[!h]
    \centering
    \captionsetup[subfigure]{justification=centering}
    \begin{subfigure}{0.28\linewidth}\centering
        \includegraphics[width=\linewidth]{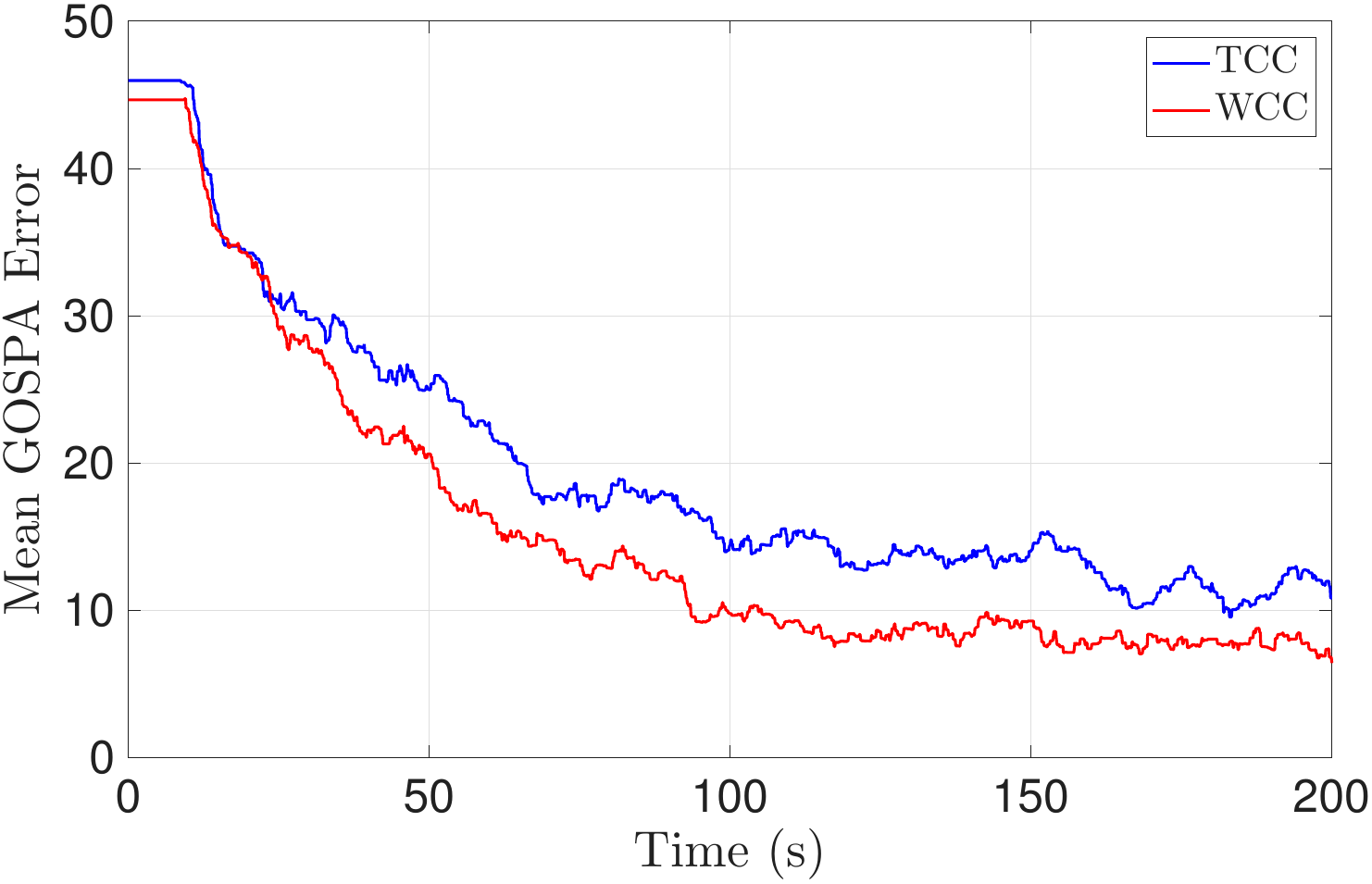}
        \caption{$n=$ 6 and $M_{\max}=4$}
        \label{fig:gospaNoA6Ns4}
    \end{subfigure}
    \begin{subfigure}{0.28\linewidth}\centering
        \includegraphics[width=\linewidth]{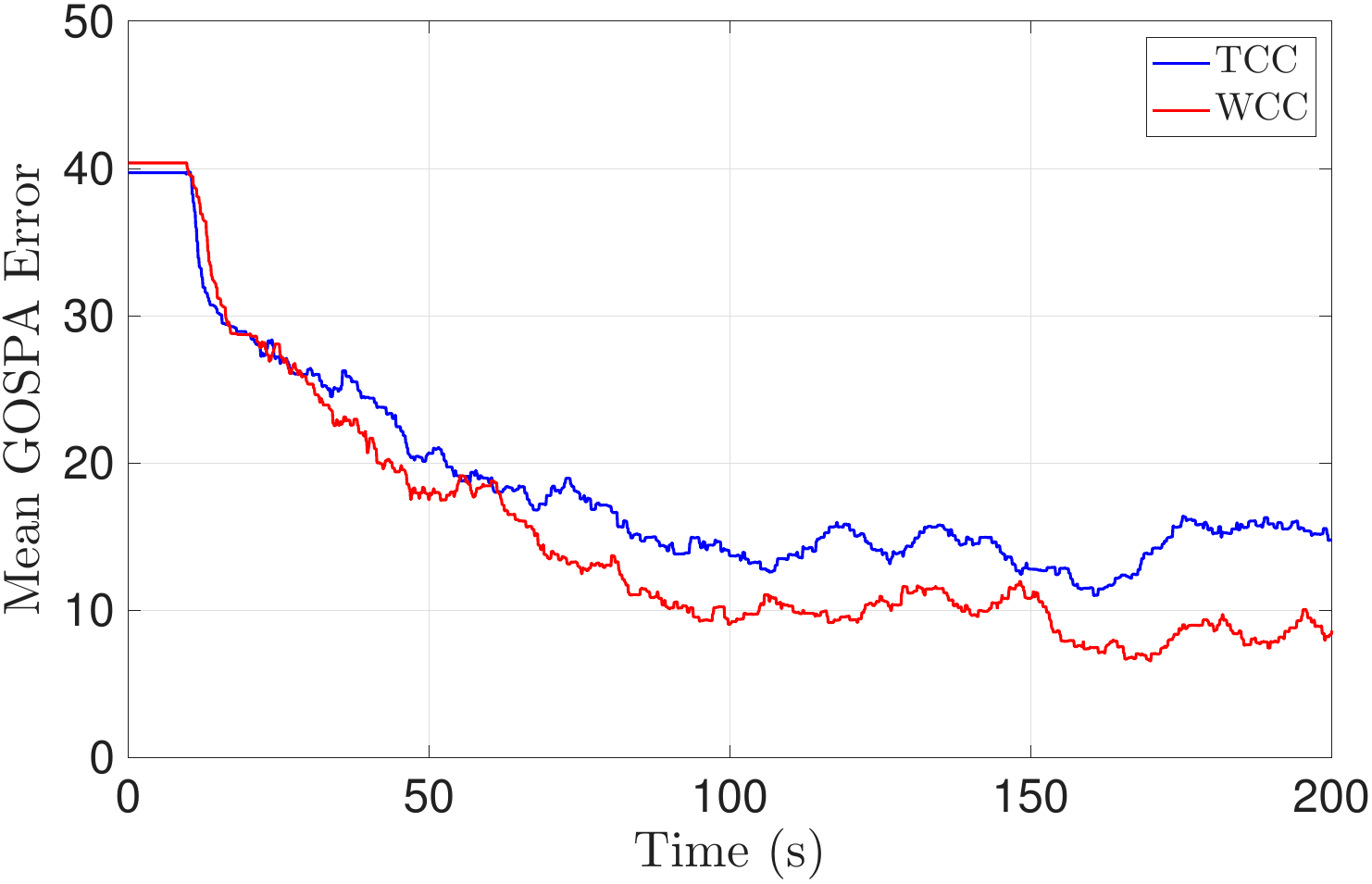}
        \caption{$n=$ 6 and $M_{\max}=3$}
        \label{fig:gospaNoA6Ns3}
    \end{subfigure}
    \begin{subfigure}{0.28\linewidth}\centering
        \includegraphics[width=\linewidth]{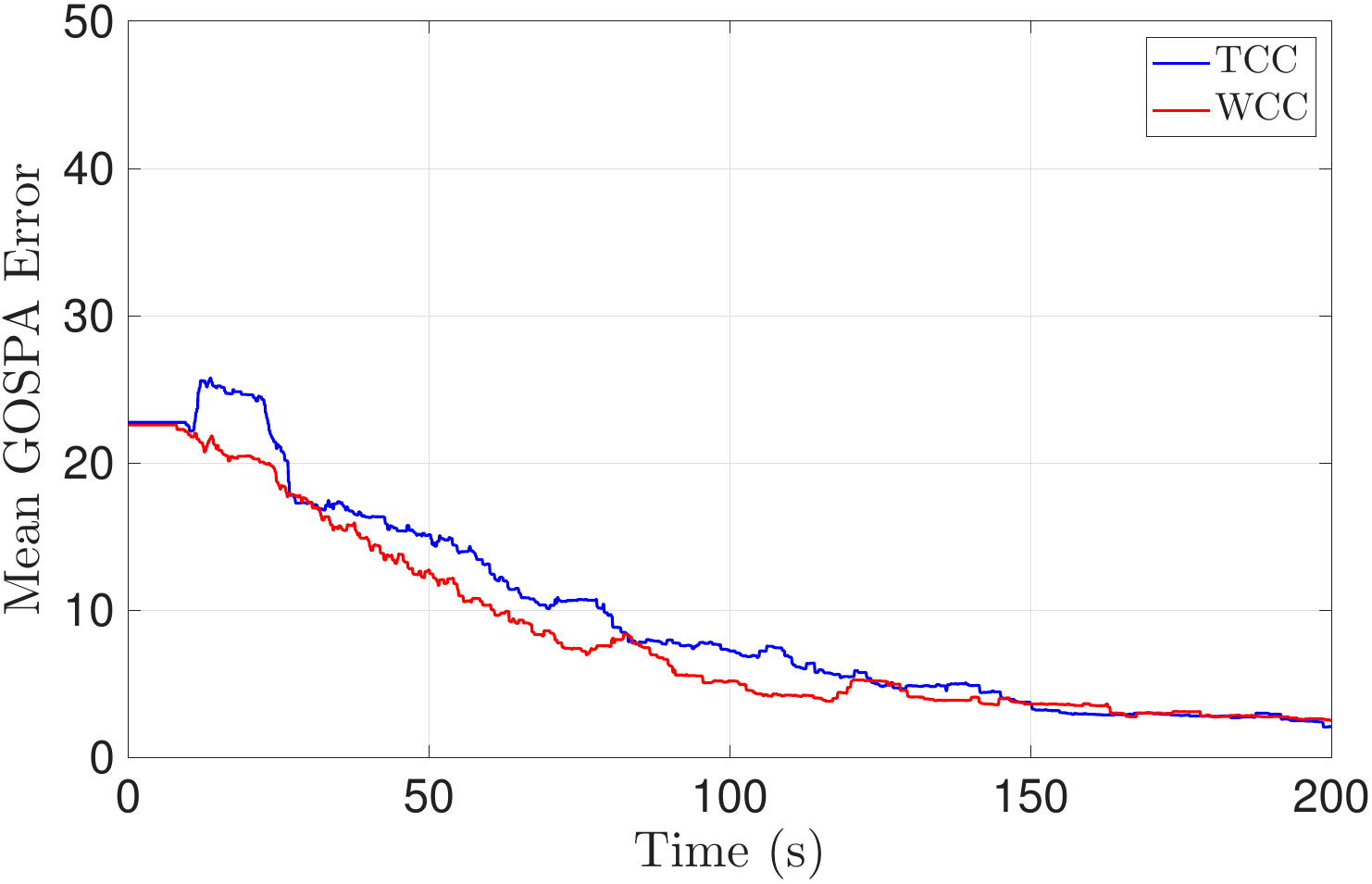}
        \caption{$n=$ 6 and $M_{\max}=2$}
        \label{fig:gospaNoA6Ns2}
    \end{subfigure}\\
    \begin{subfigure}{0.28\linewidth}\centering
        \includegraphics[width=\linewidth]{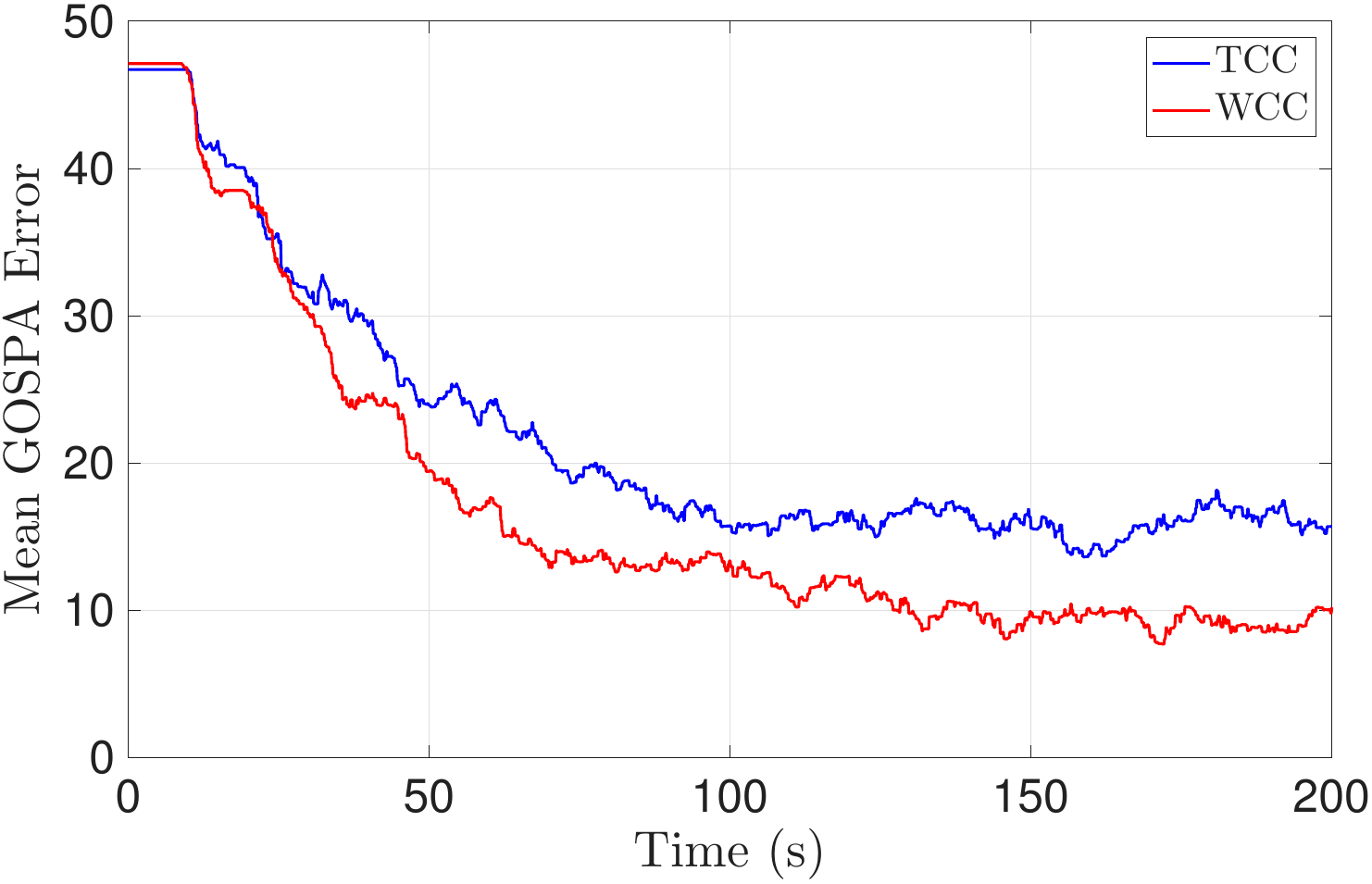}
        \caption{$n=$ 5 and $M_{\max}=4$}
        \label{fig:gospaNoA5Ns4}
    \end{subfigure}
    \begin{subfigure}{0.28\linewidth}\centering
        \includegraphics[width=\linewidth]{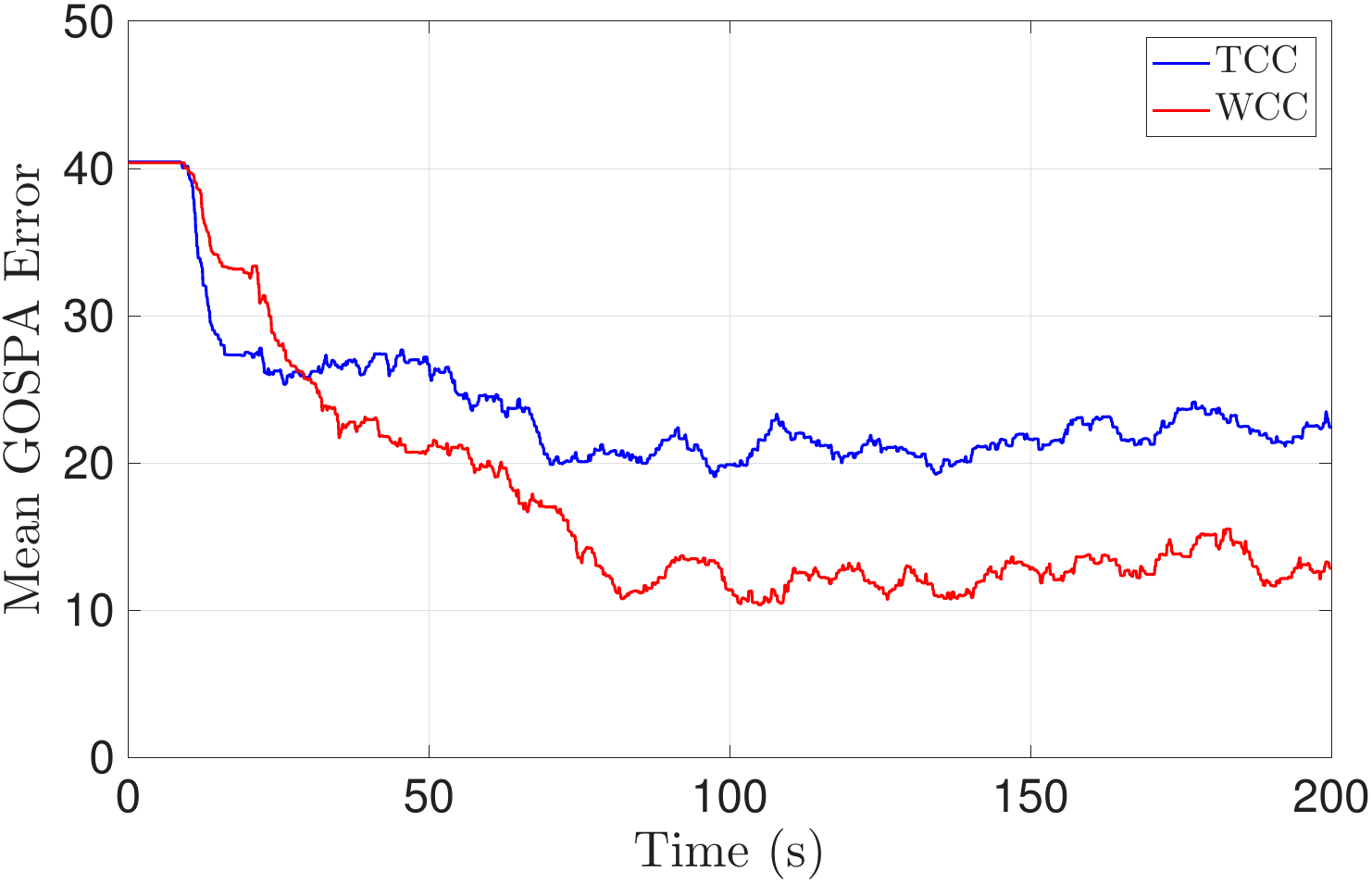}
        \caption{$n=$ 5 and $M_{\max}=3$}
        \label{fig:gospaNoA5Ns3}
    \end{subfigure}
    \begin{subfigure}{0.28\linewidth}\centering
        \includegraphics[width=\linewidth]{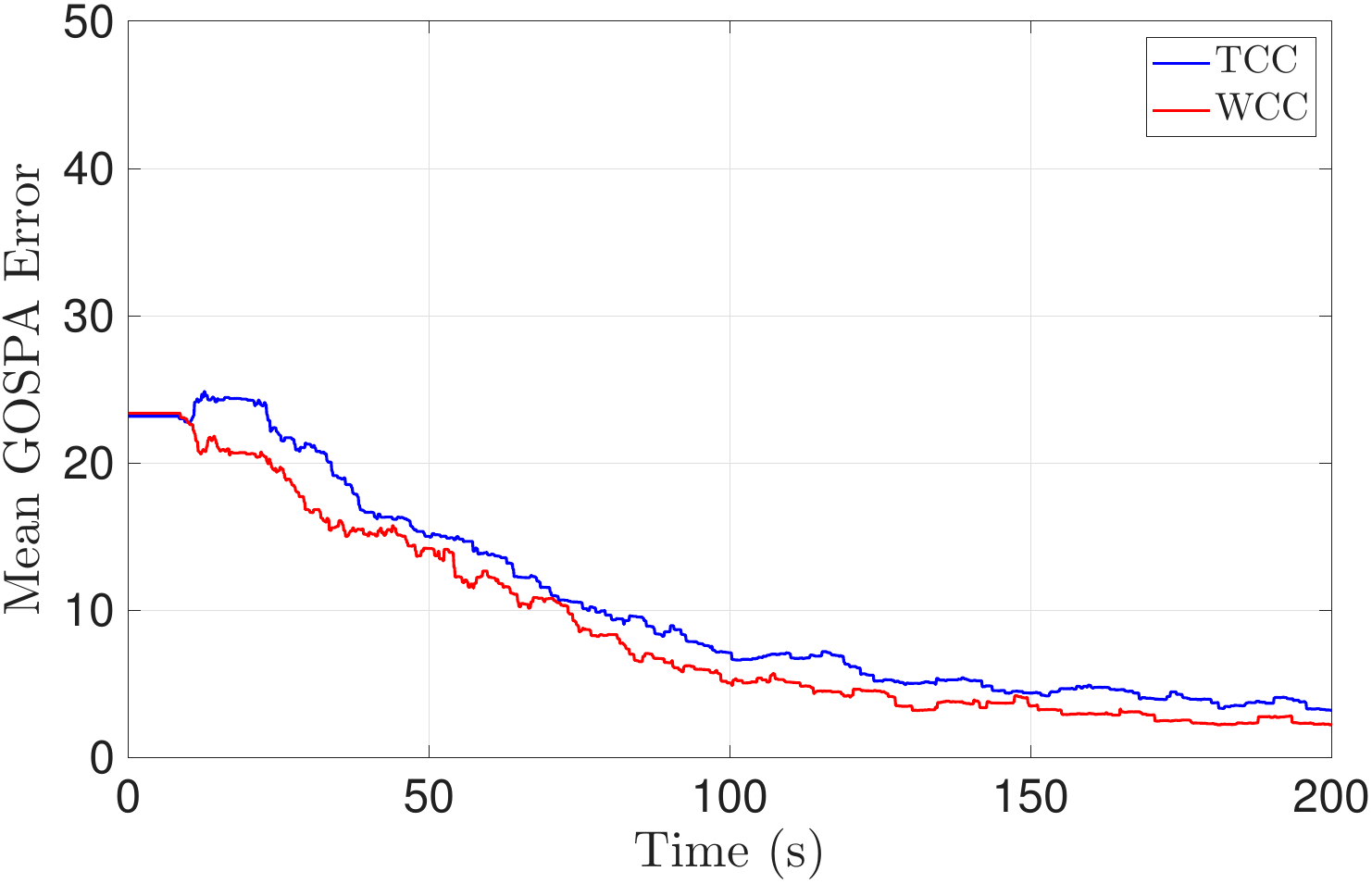}
        \caption{$n=$ 5 and $M_{\max}=2$}
        \label{fig:gospaNoA5Ns2}
    \end{subfigure}\\
    \begin{subfigure}{0.28\linewidth}\centering
        \includegraphics[width=\linewidth]{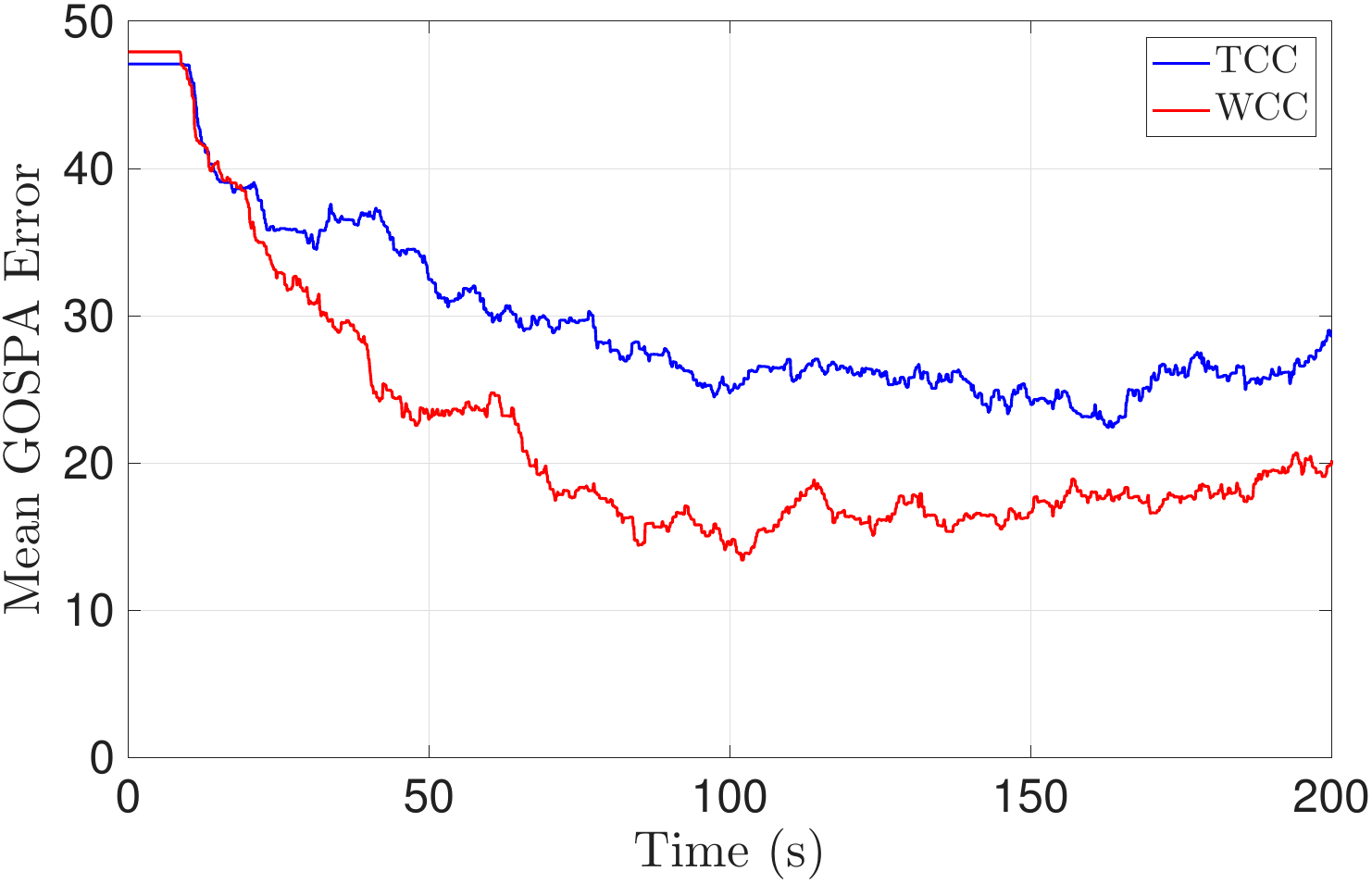}
        \caption{$n=$ 4 and $M_{\max}=4$}
        \label{fig:gospaNoA4Ns4}
    \end{subfigure}
    \begin{subfigure}{0.28\linewidth}\centering
        \includegraphics[width=\linewidth]{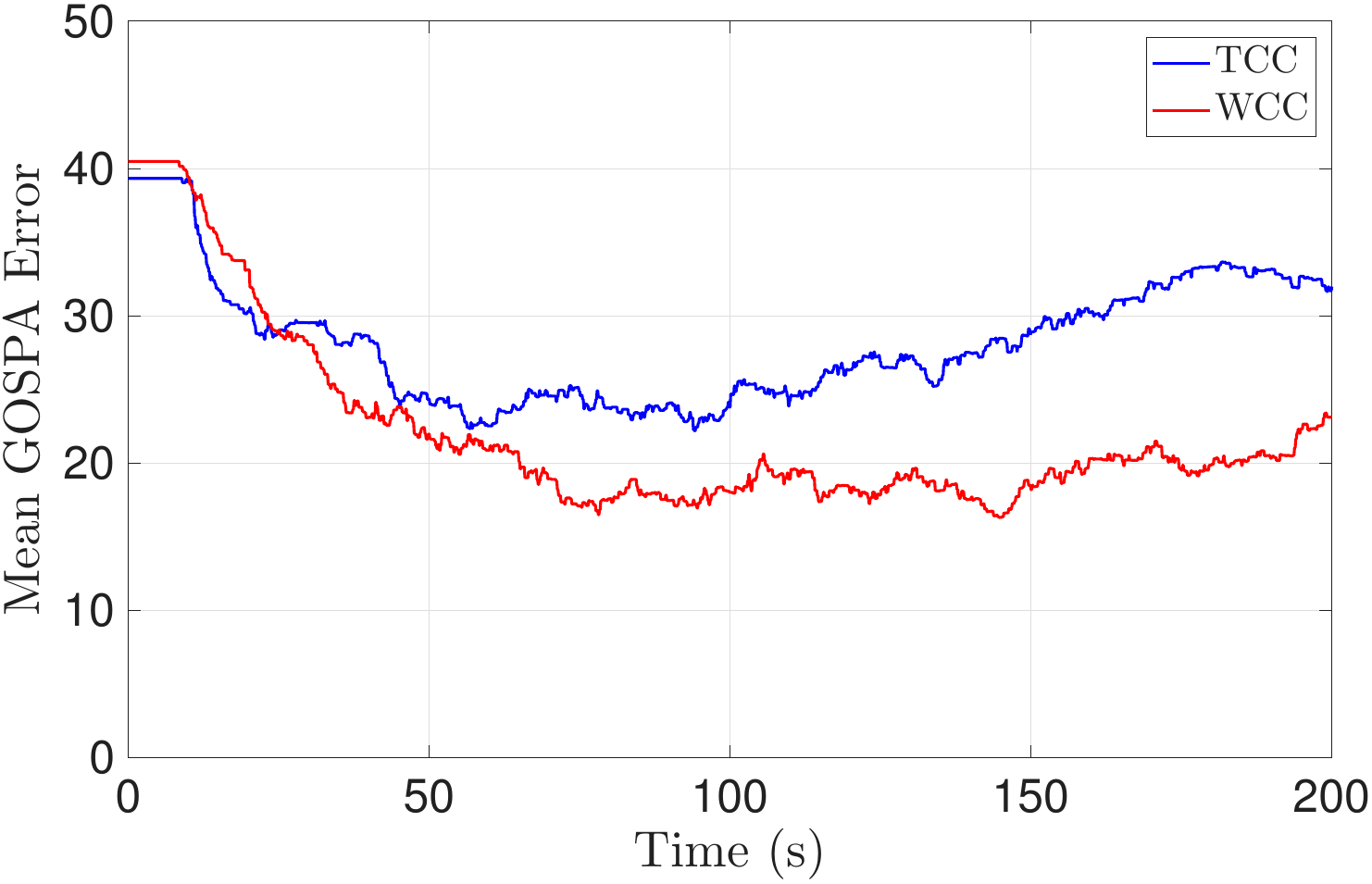}
        \caption{$n=$ 4 and $M_{\max}=3$}
        \label{fig:gospaNoA4Ns3}
    \end{subfigure}
    \begin{subfigure}{0.28\linewidth}\centering
        \includegraphics[width=\linewidth]{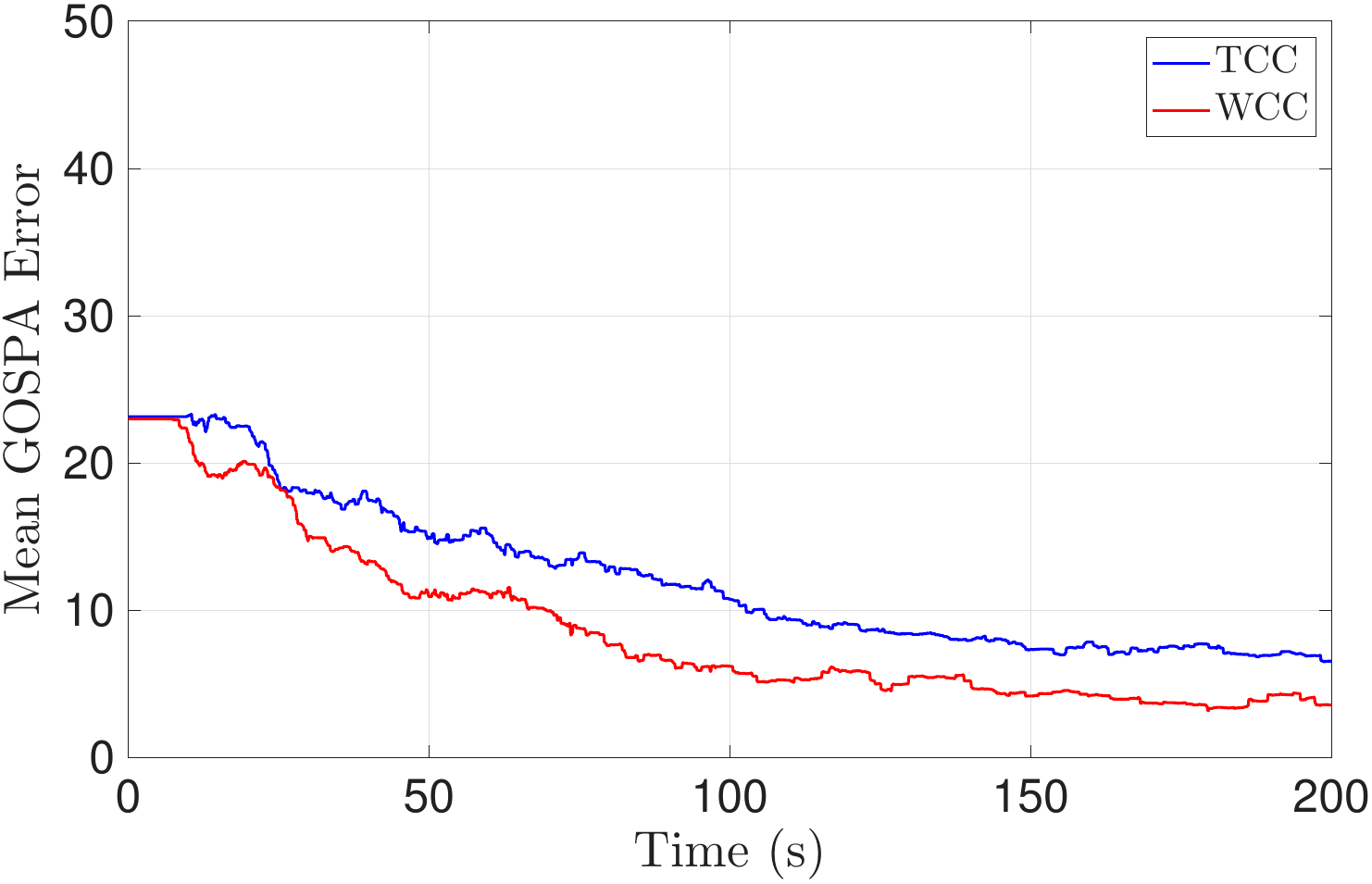}
        \caption{$n=$ 4 and $M_{\max}=2$}
        \label{fig:gospaNoA4Ns2}
    \end{subfigure}\\
    \begin{subfigure}{0.28\linewidth}\centering
        \includegraphics[width=\linewidth]{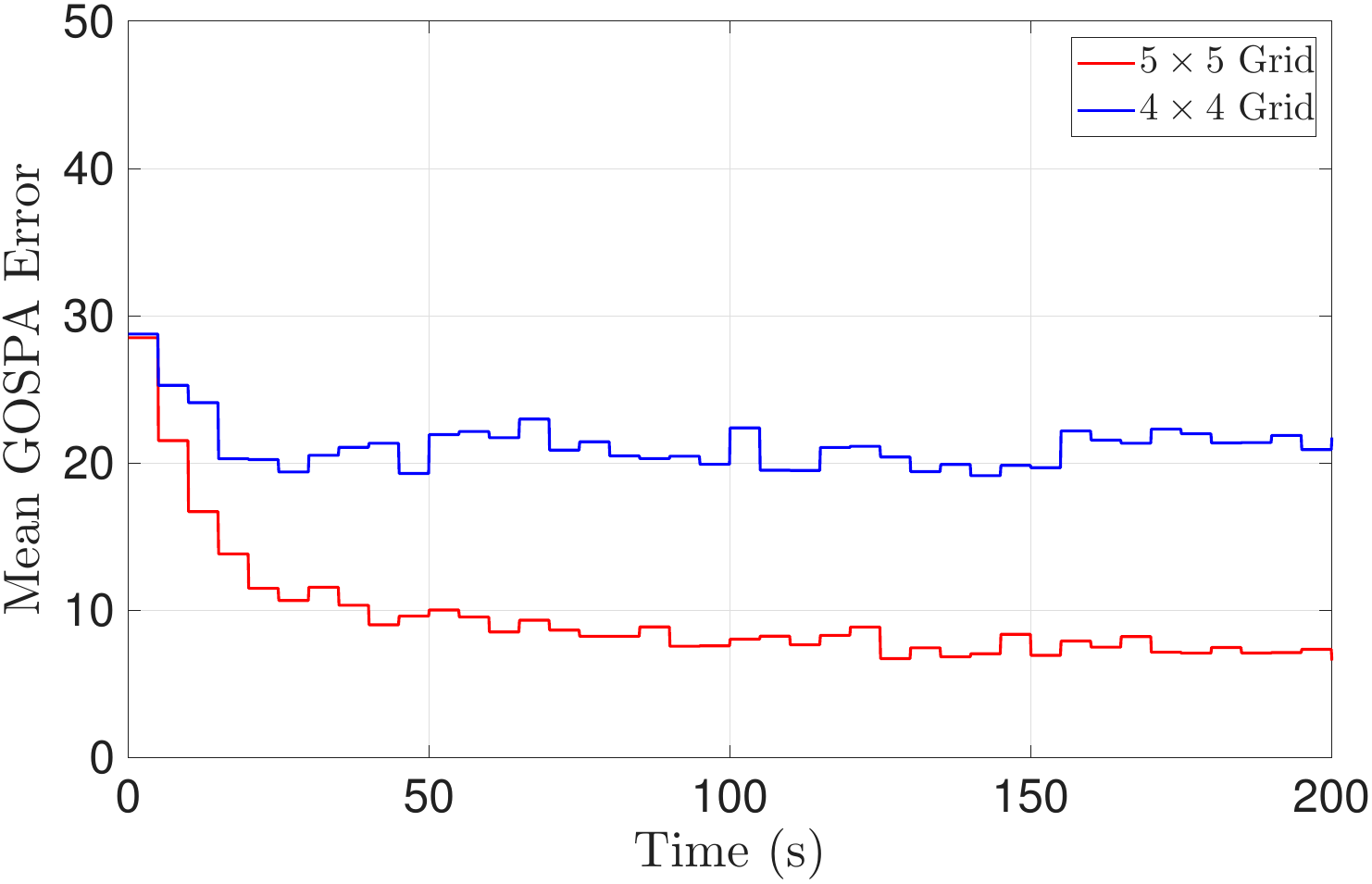}
        \caption{Static grid of sensors and $M_{\max}=4$}
        \label{fig:gospaGridNs4}
    \end{subfigure}
    \begin{subfigure}{0.28\linewidth}\centering
        \includegraphics[width=\linewidth]{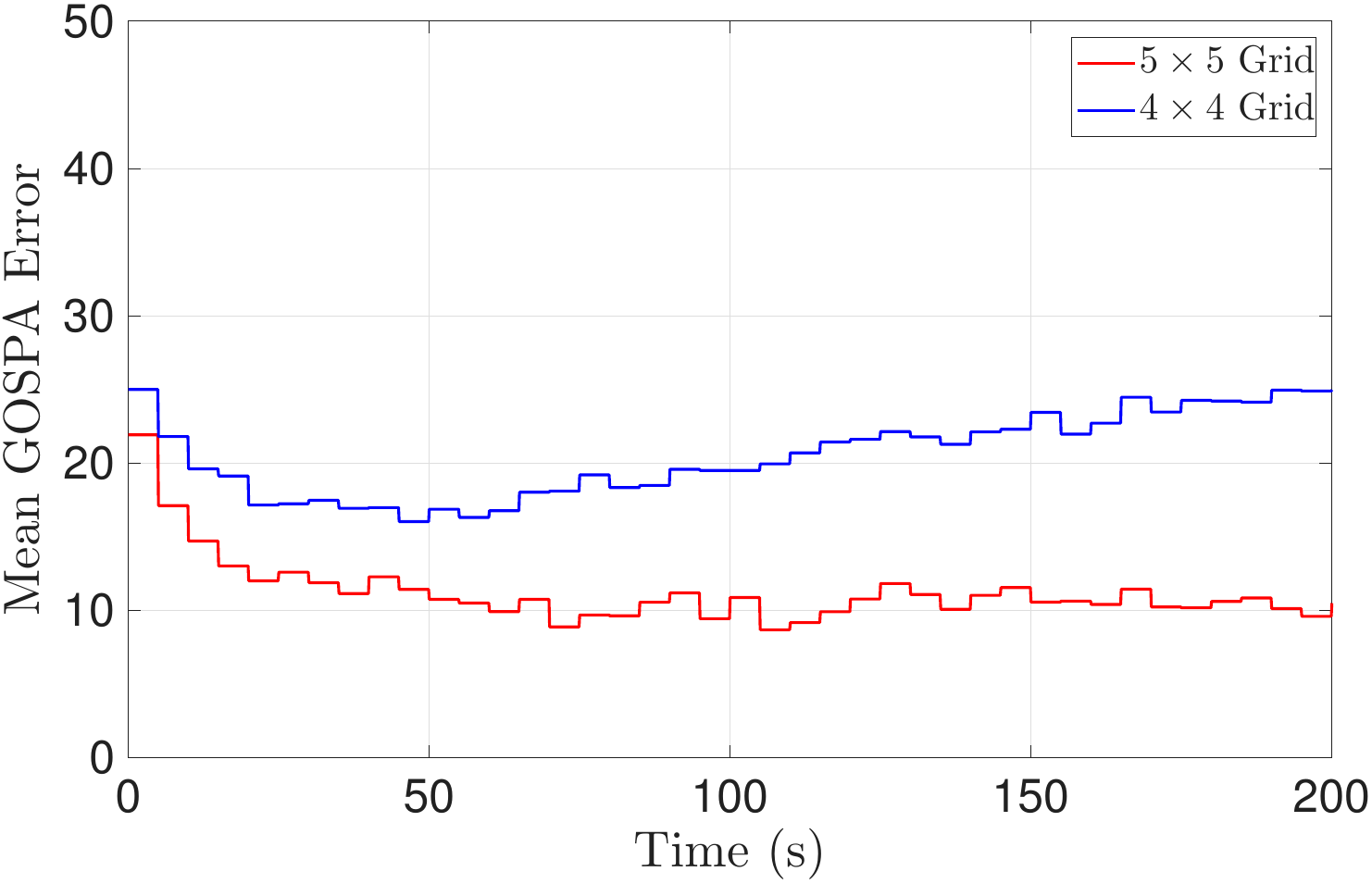}
        \caption{Static grid of sensors and $M_{\max}=3$}
        \label{fig:gospaGridNs3}
    \end{subfigure}
    \begin{subfigure}{0.28\linewidth}\centering
        \includegraphics[width=\linewidth]{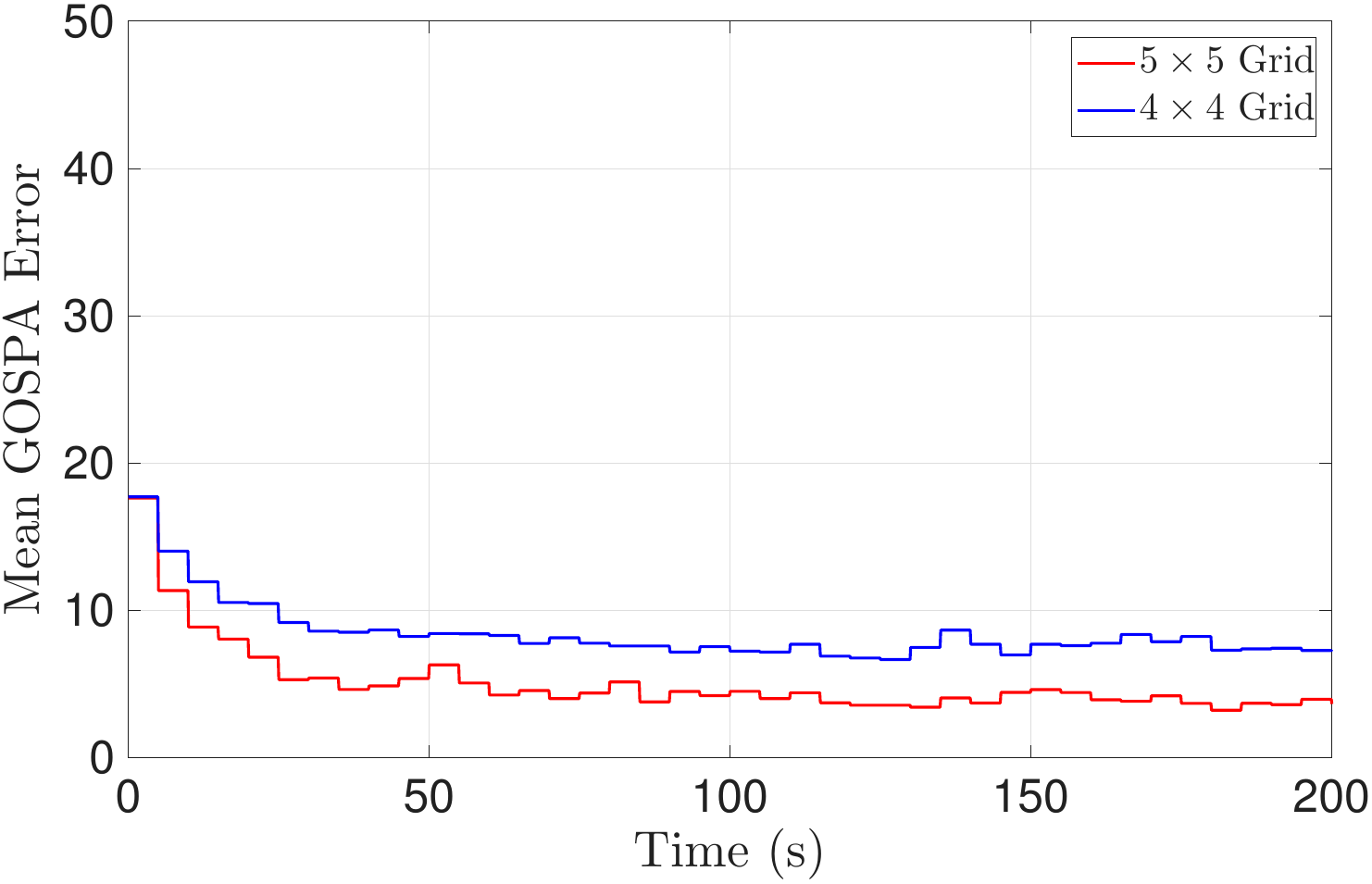}
        \caption{Static grid of sensors and $M_{\max}=2$}
        \label{fig:gospaGridNs2}
    \end{subfigure}
    \caption{Average GOSPA error across 80 Monte Carlo runs for different simulation settings.}
    \label{fig:MCrunsGOSPA}
\end{figure*}

\begin{figure*}[!h]
    \centering
    \captionsetup[subfigure]{justification=centering}
    \begin{subfigure}{0.28\linewidth}\centering
        \includegraphics[width=\linewidth]{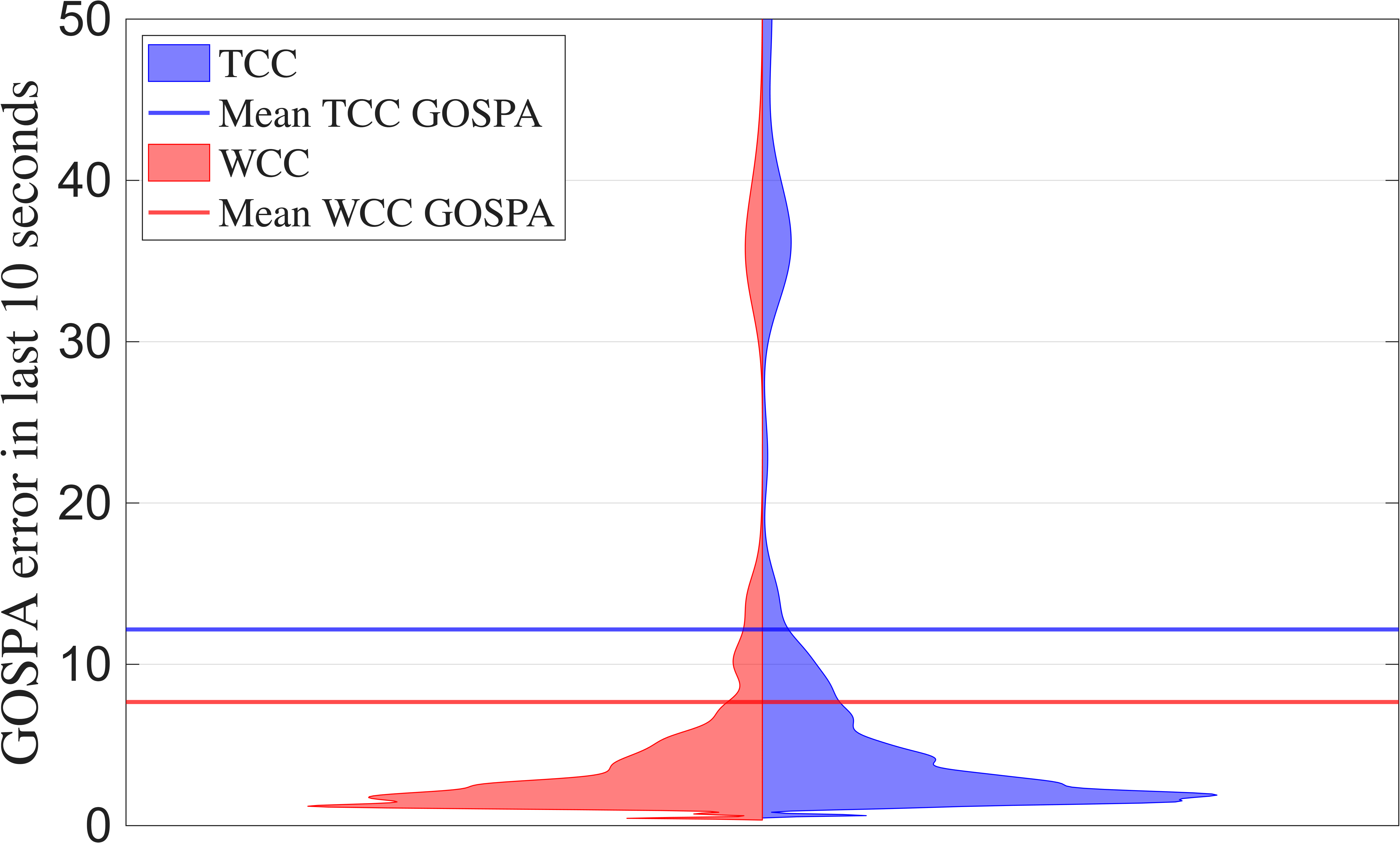}
        \caption{$n=$ 6 and $M_{\max}=4$}
        \label{fig:violin_gospaNoA6Ns4}
    \end{subfigure}
    \begin{subfigure}{0.28\linewidth}\centering
        \includegraphics[width=\linewidth]{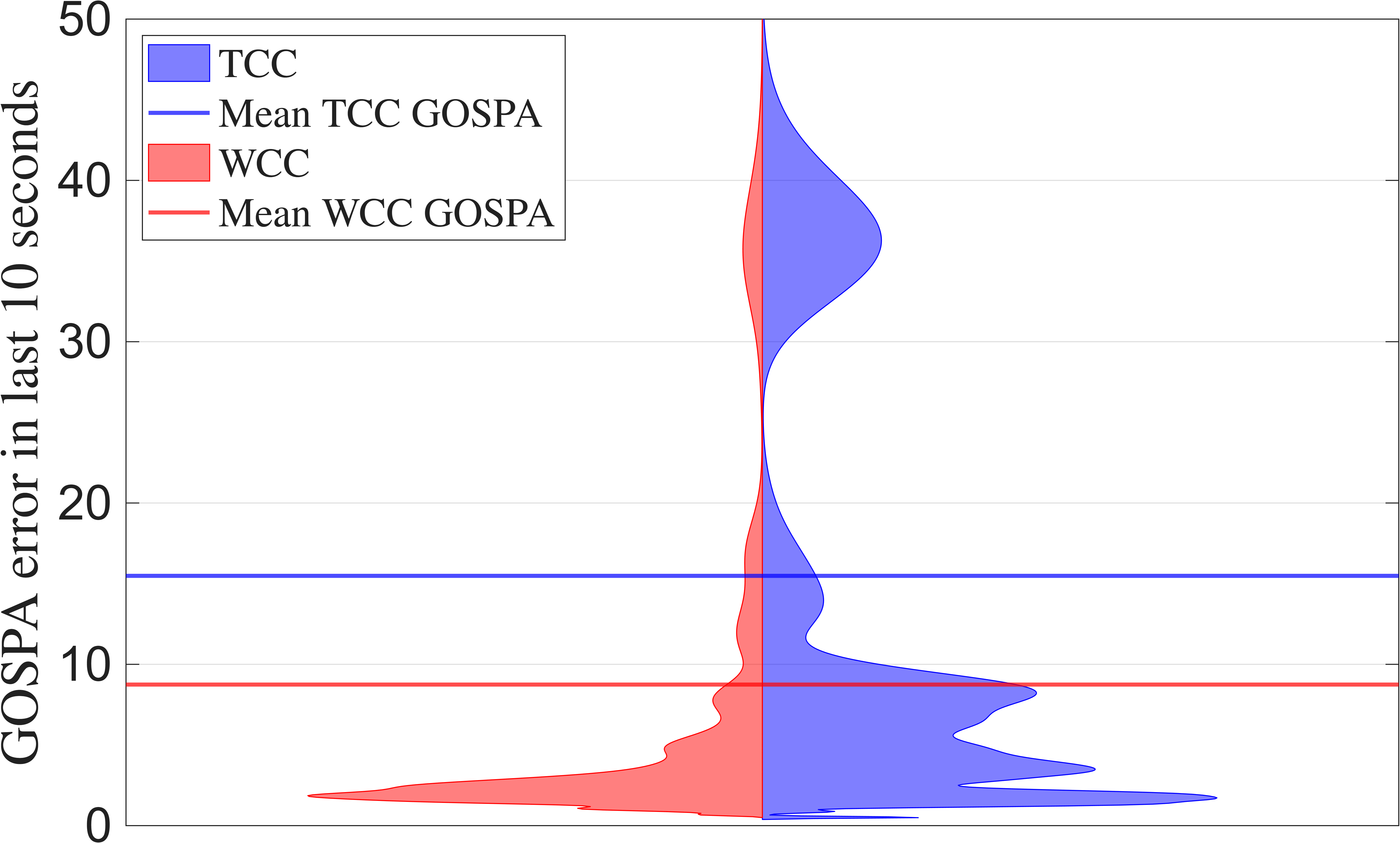}
        \caption{$n=$ 6 and $M_{\max}=3$}
        \label{fig:violin_gospaNoA6Ns3}
    \end{subfigure}
    \begin{subfigure}{0.28\linewidth}\centering
        \includegraphics[width=\linewidth]{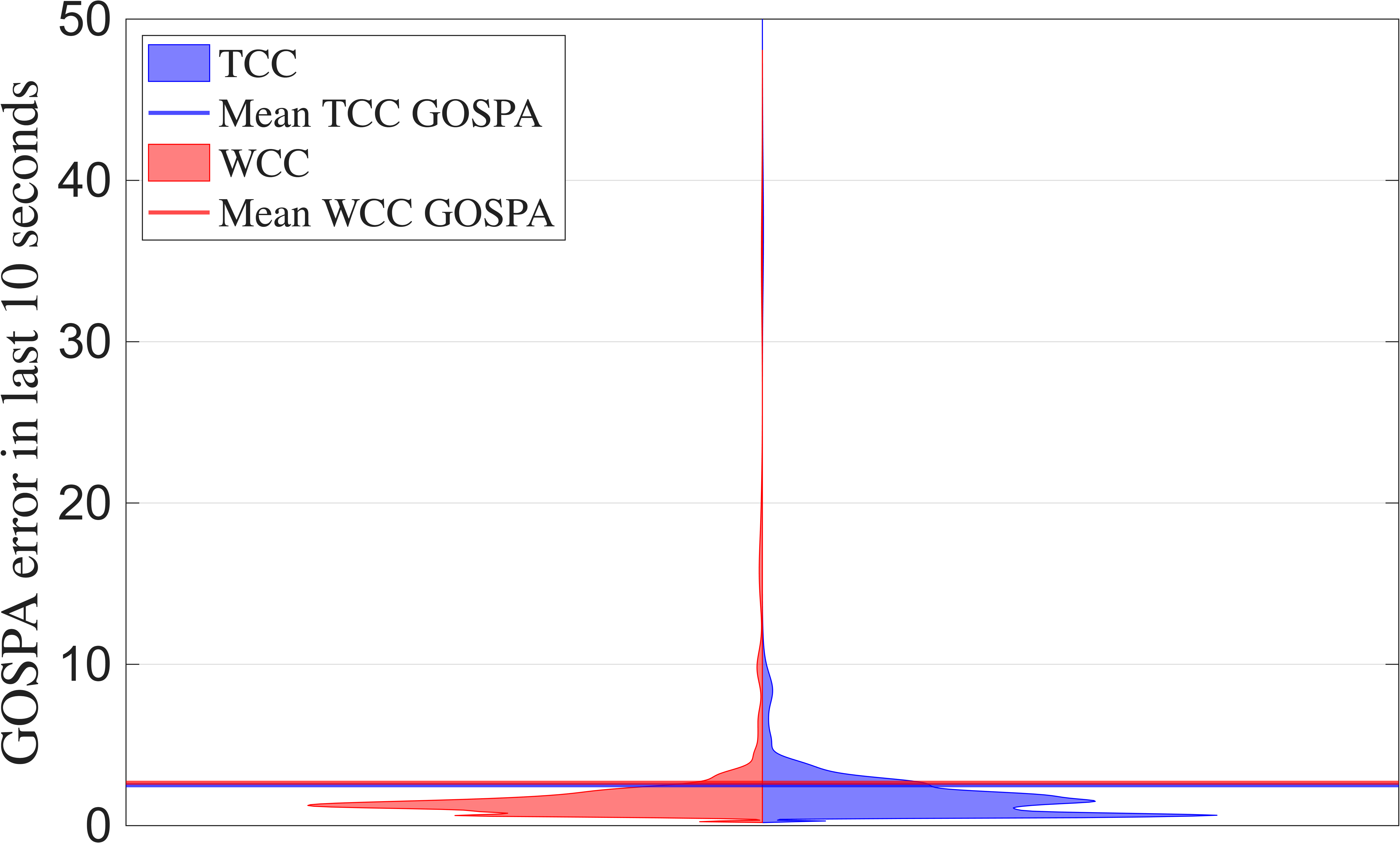}
        \caption{$n=$ 6 and $M_{\max}=2$}
        \label{fig:violin_gospaNoA6Ns2}
    \end{subfigure}\\
    \begin{subfigure}{0.28\linewidth}\centering
        \includegraphics[width=1\linewidth]{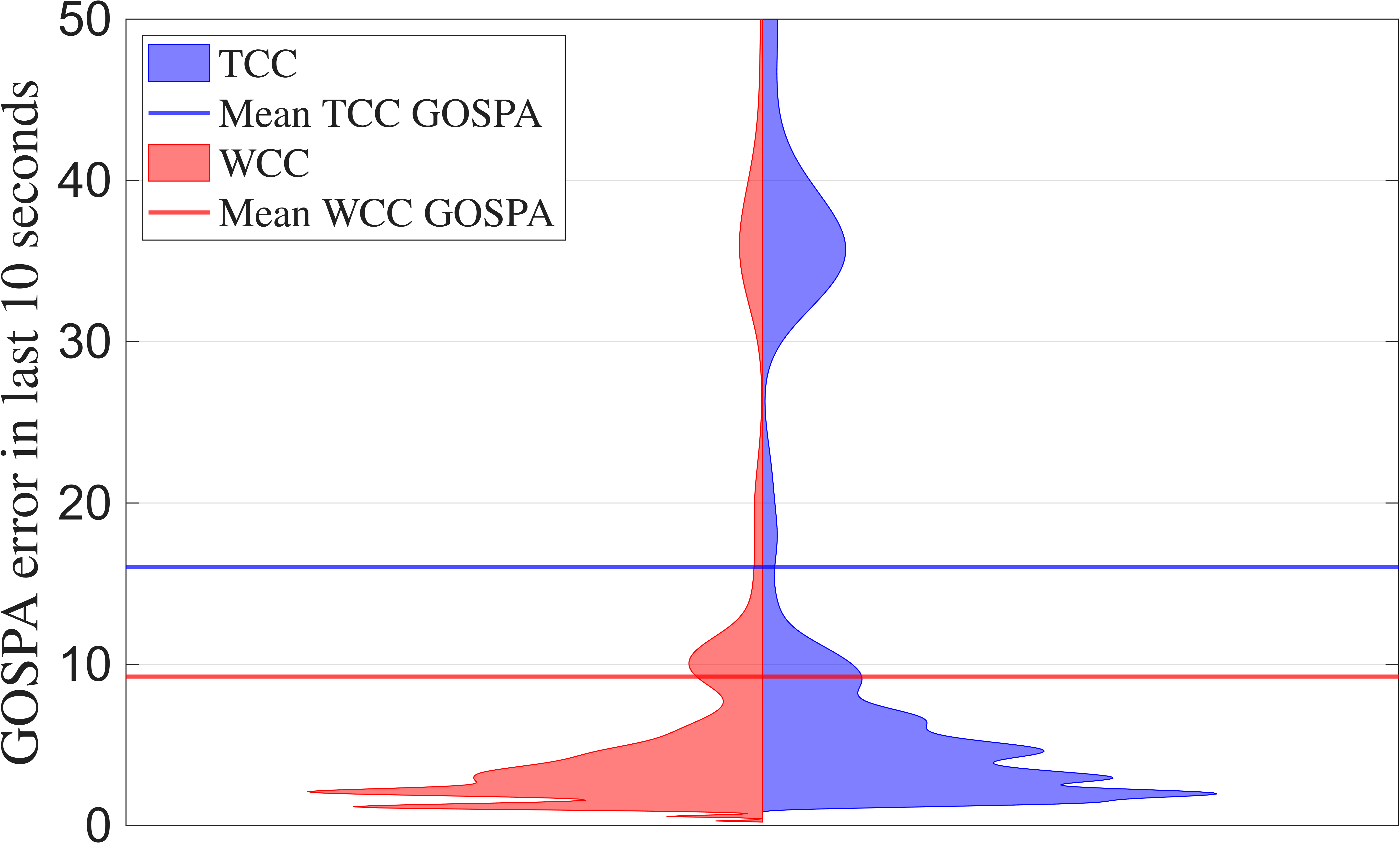}
        \caption{$n=$ 5 and $M_{\max}=4$}
        \label{fig:violin_gospaNoA5Ns4}
    \end{subfigure}
    \begin{subfigure}{0.28\linewidth}\centering
        \includegraphics[width=1\linewidth]{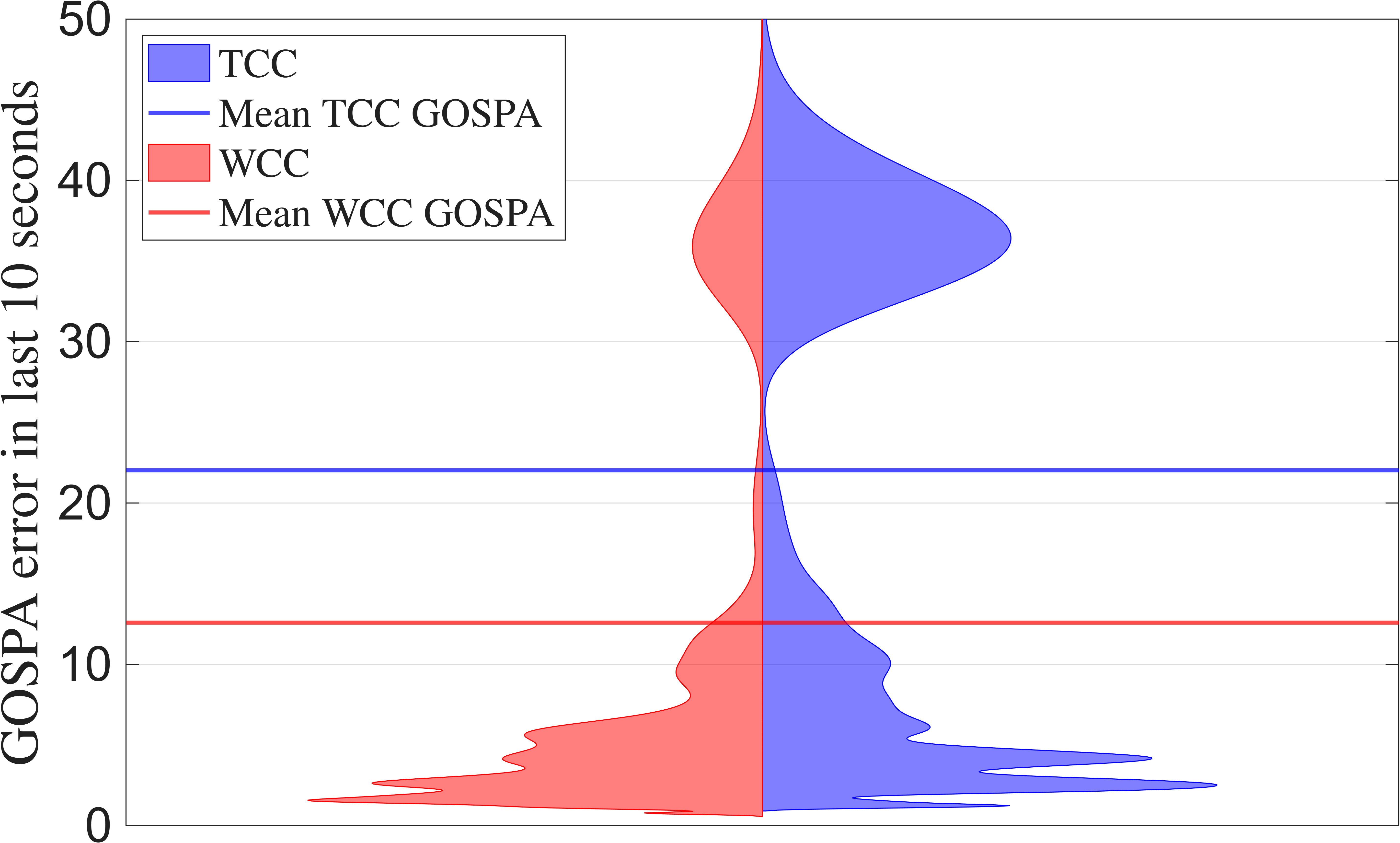}
        \caption{$n=$ 5 and $M_{\max}=3$}
        \label{fig:violin_gospaNoA5Ns3}
    \end{subfigure}
    \begin{subfigure}{0.28\linewidth}\centering
        \includegraphics[width=1\linewidth]{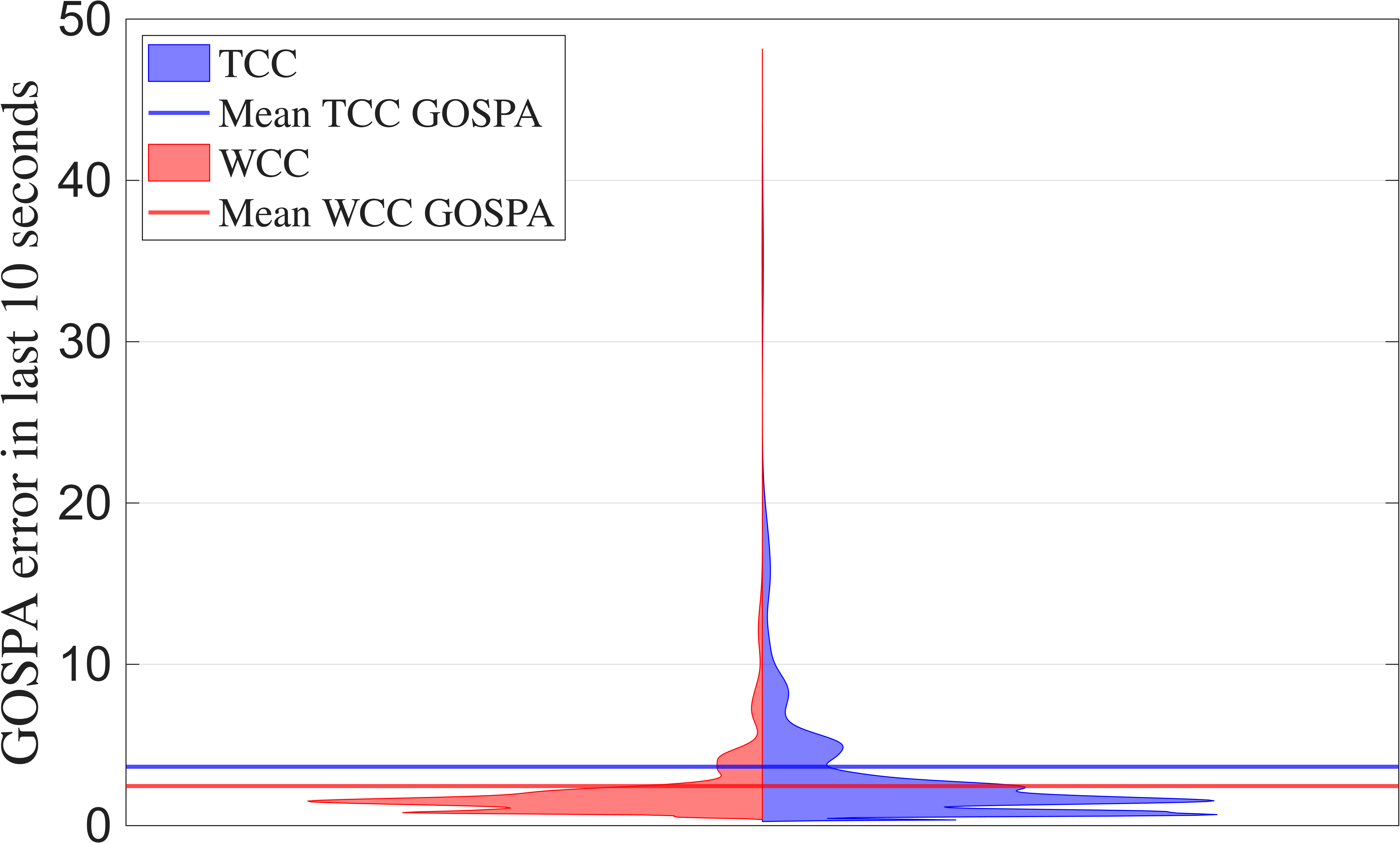}
        \caption{$n=$ 5 and $M_{\max}=2$}
        \label{fig:violin_gospaNoA5Ns2}
    \end{subfigure}\\
    \begin{subfigure}{0.28\linewidth}\centering
        \includegraphics[width=1\linewidth]{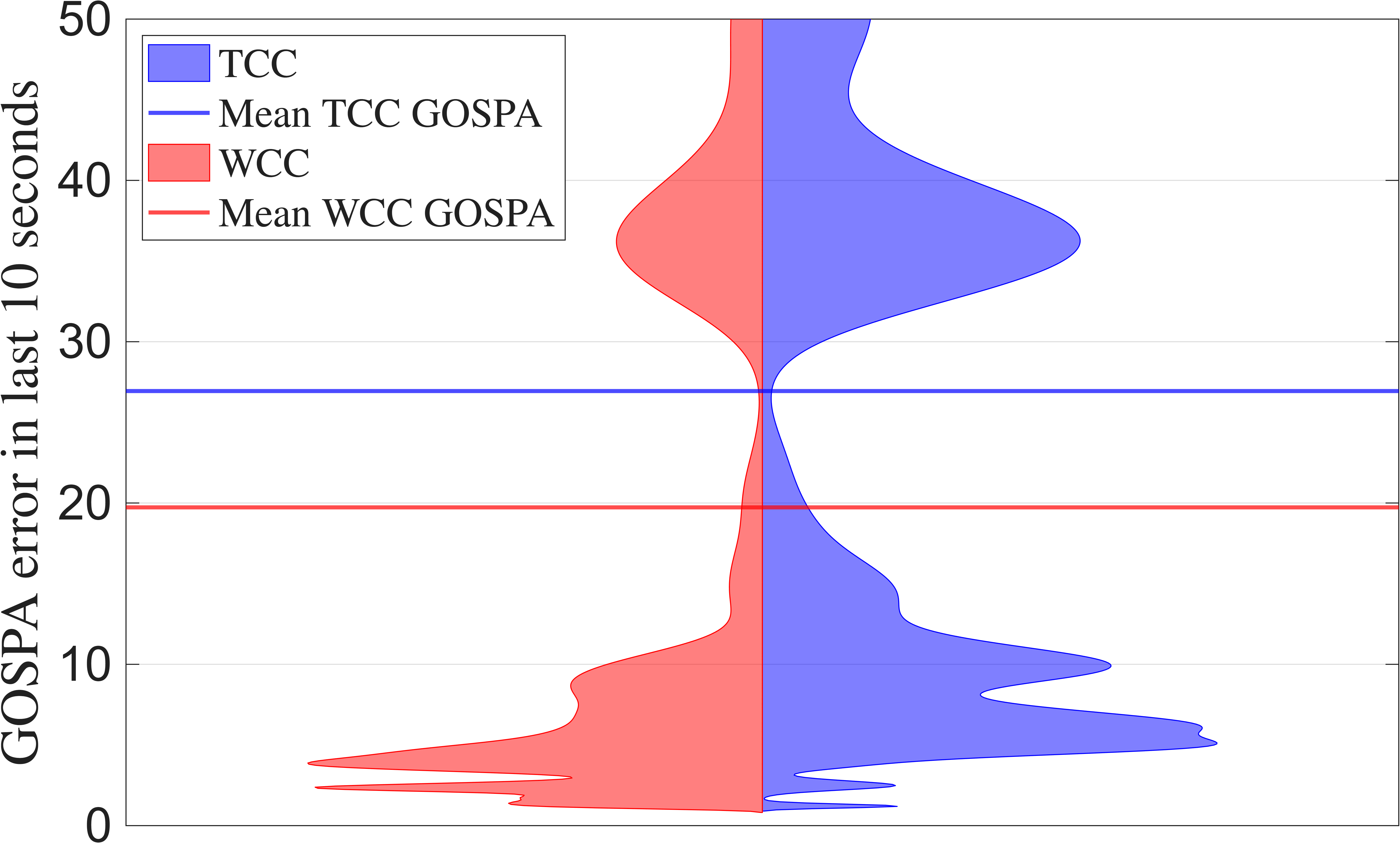}
        \caption{$n=$ 4 and $M_{\max}=4$}
        \label{fig:violin_gospaNoA4Ns4}
    \end{subfigure}
    \begin{subfigure}{0.28\linewidth}\centering
        \includegraphics[width=1\linewidth]{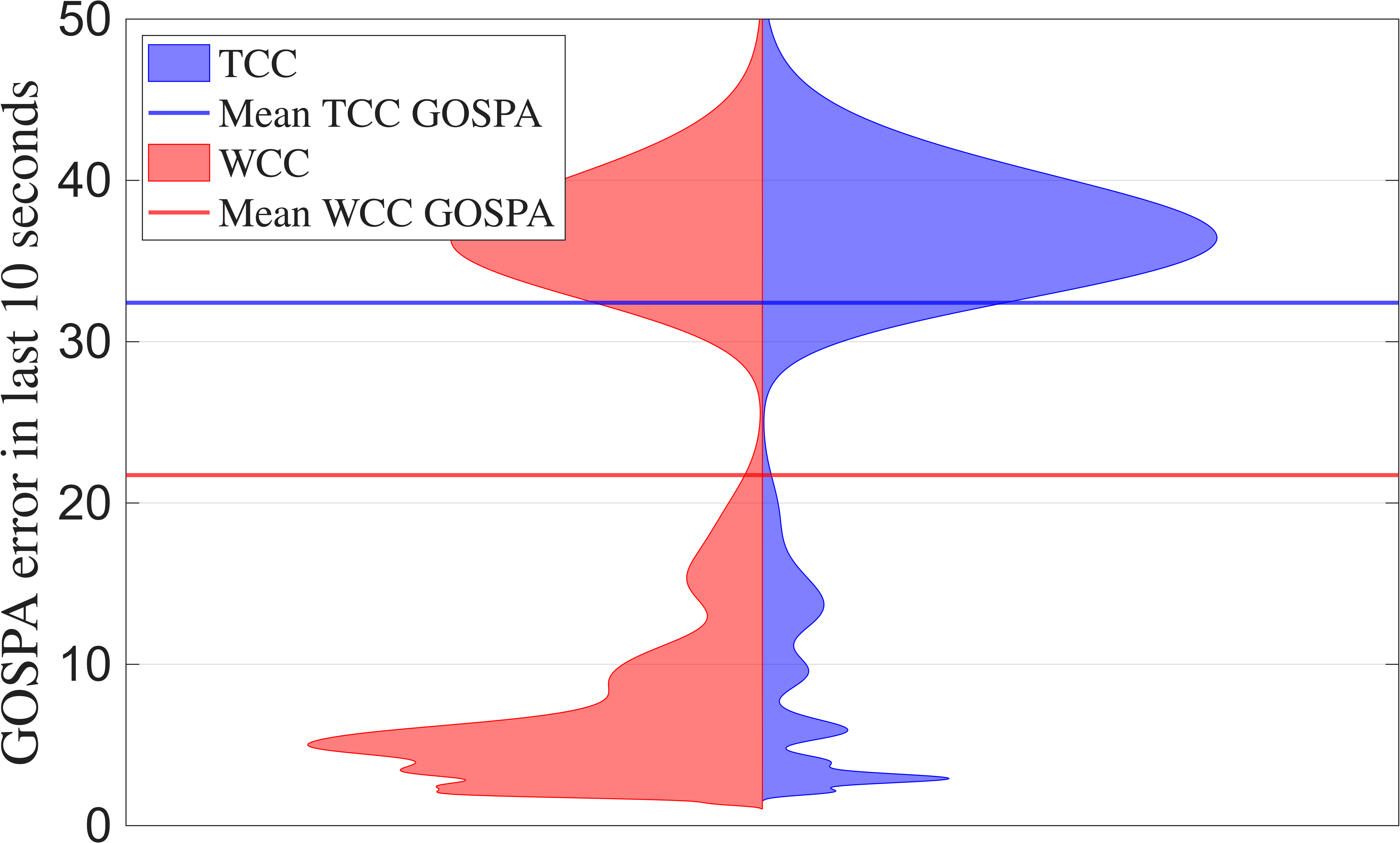}
        \caption{$n=$ 4 and $M_{\max}=3$}
        \label{fig:violin_gospaNoA4Ns3}
    \end{subfigure}
    \begin{subfigure}{0.28\linewidth}\centering
        \includegraphics[width=1\linewidth]{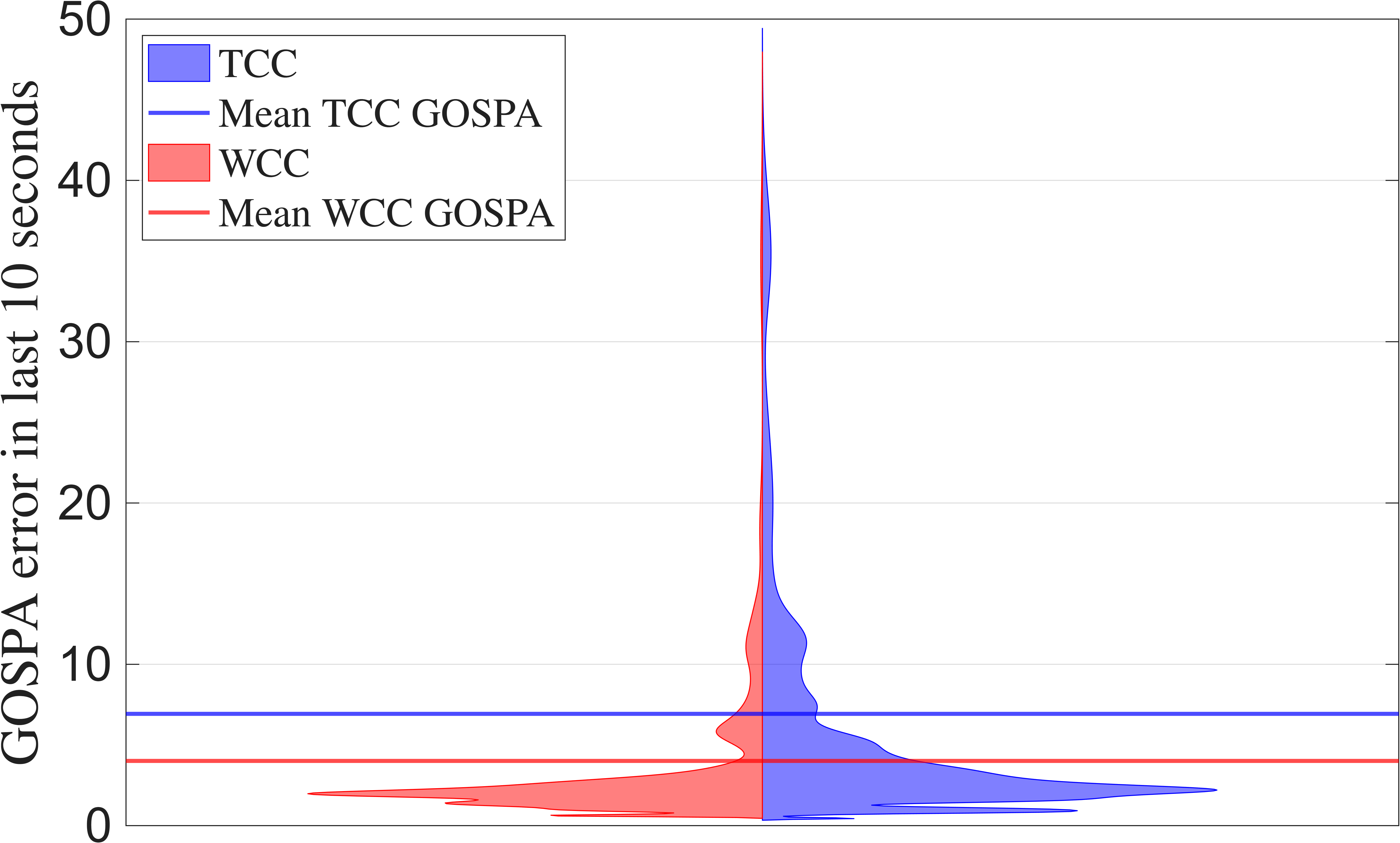}
        \caption{$n=$ 4 and $M_{\max}=2$}
        \label{fig:violin_gospaNoA4Ns2}
    \end{subfigure}\\
    \begin{subfigure}{0.28\linewidth}\centering
        \includegraphics[width=1\linewidth]{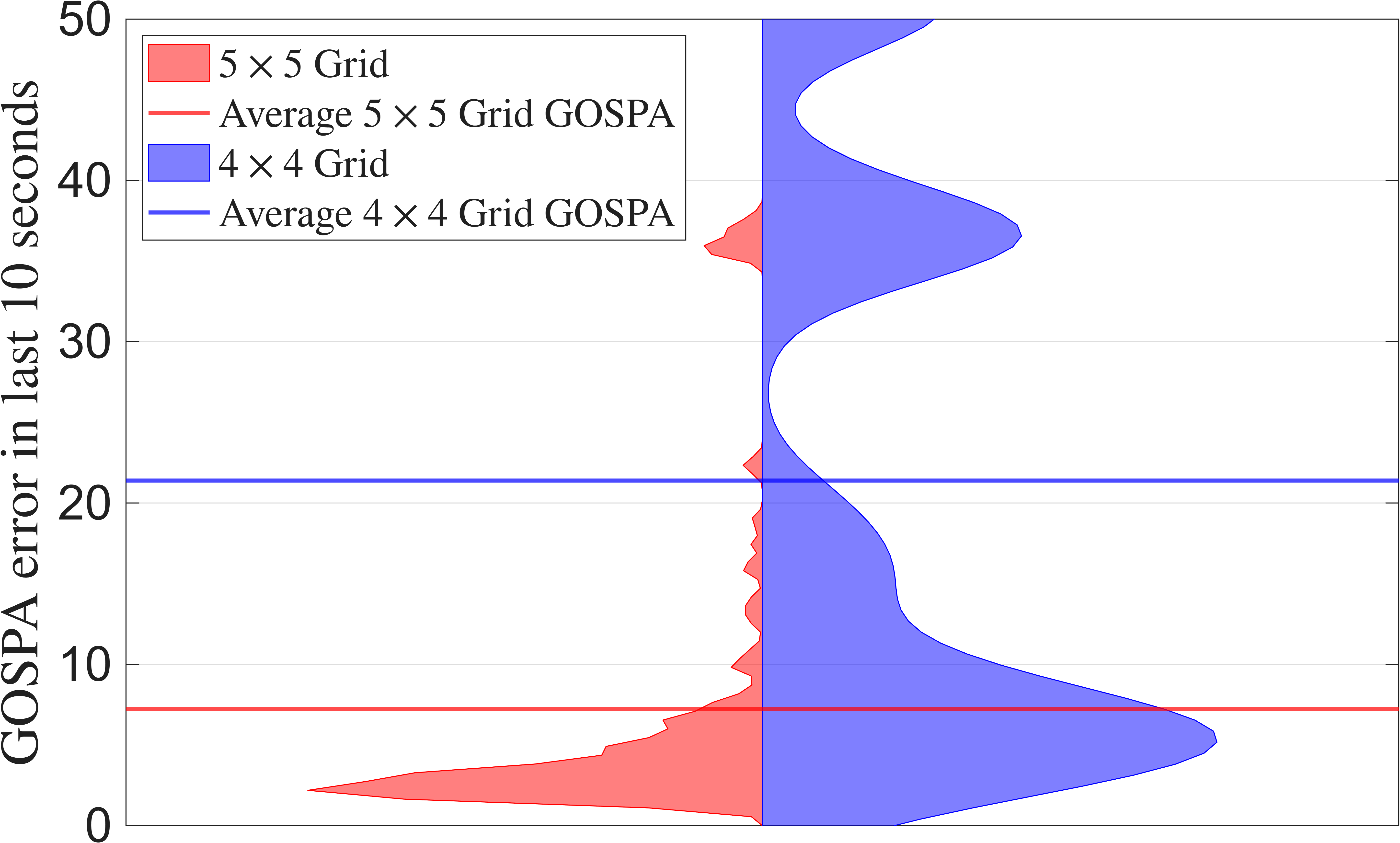}
        \caption{Static grid of sensors and $M_{\max}=4$}
        \label{fig:violin_gospaGridNs4}
    \end{subfigure}
    \begin{subfigure}{0.28\linewidth}\centering
        \includegraphics[width=1\linewidth]{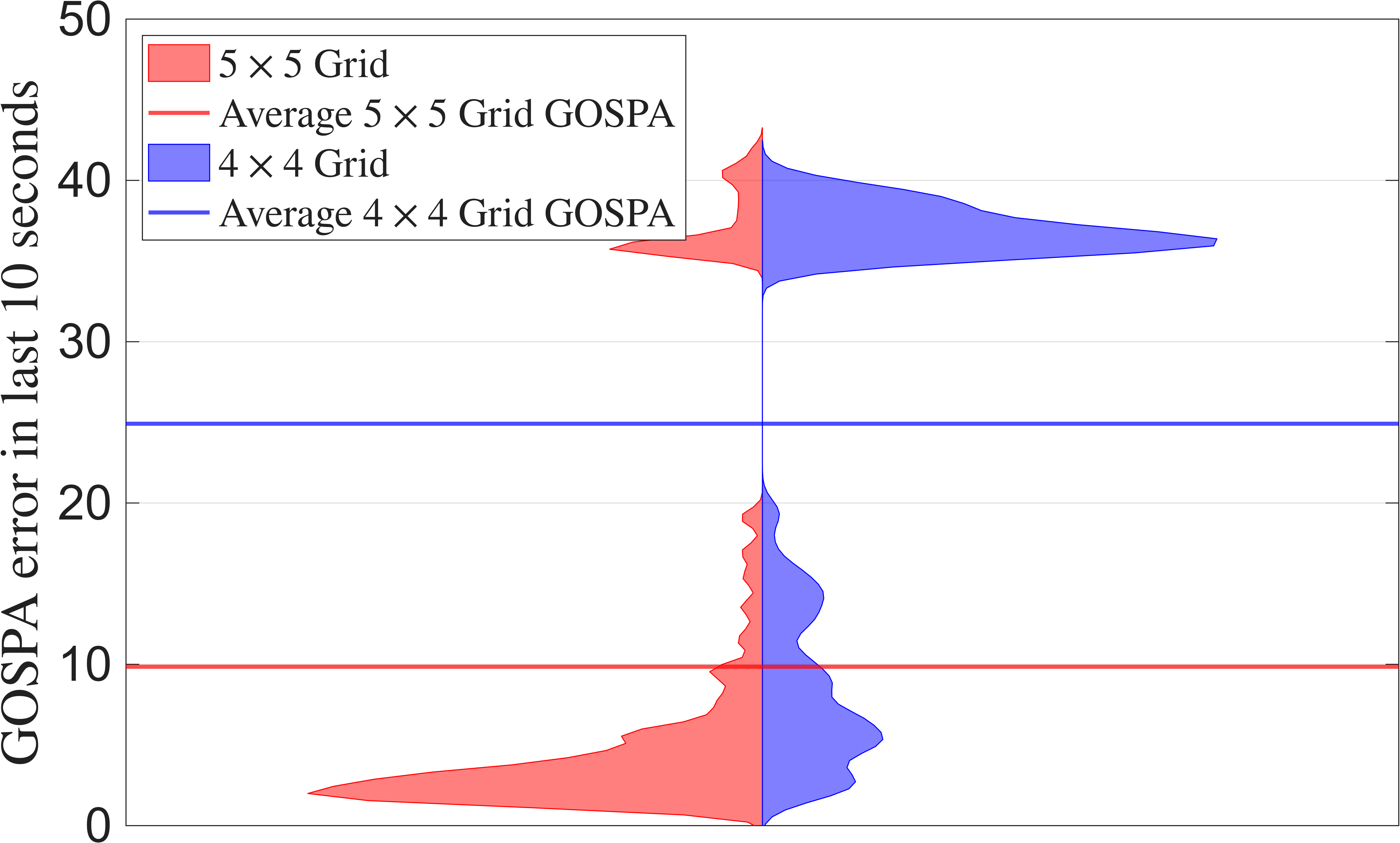}
        \caption{Static grid of sensors and $M_{\max}=3$}
        \label{fig:violin_gospaGridNs3}
    \end{subfigure}
    \begin{subfigure}{0.28\linewidth}\centering
        \includegraphics[width=1\linewidth]{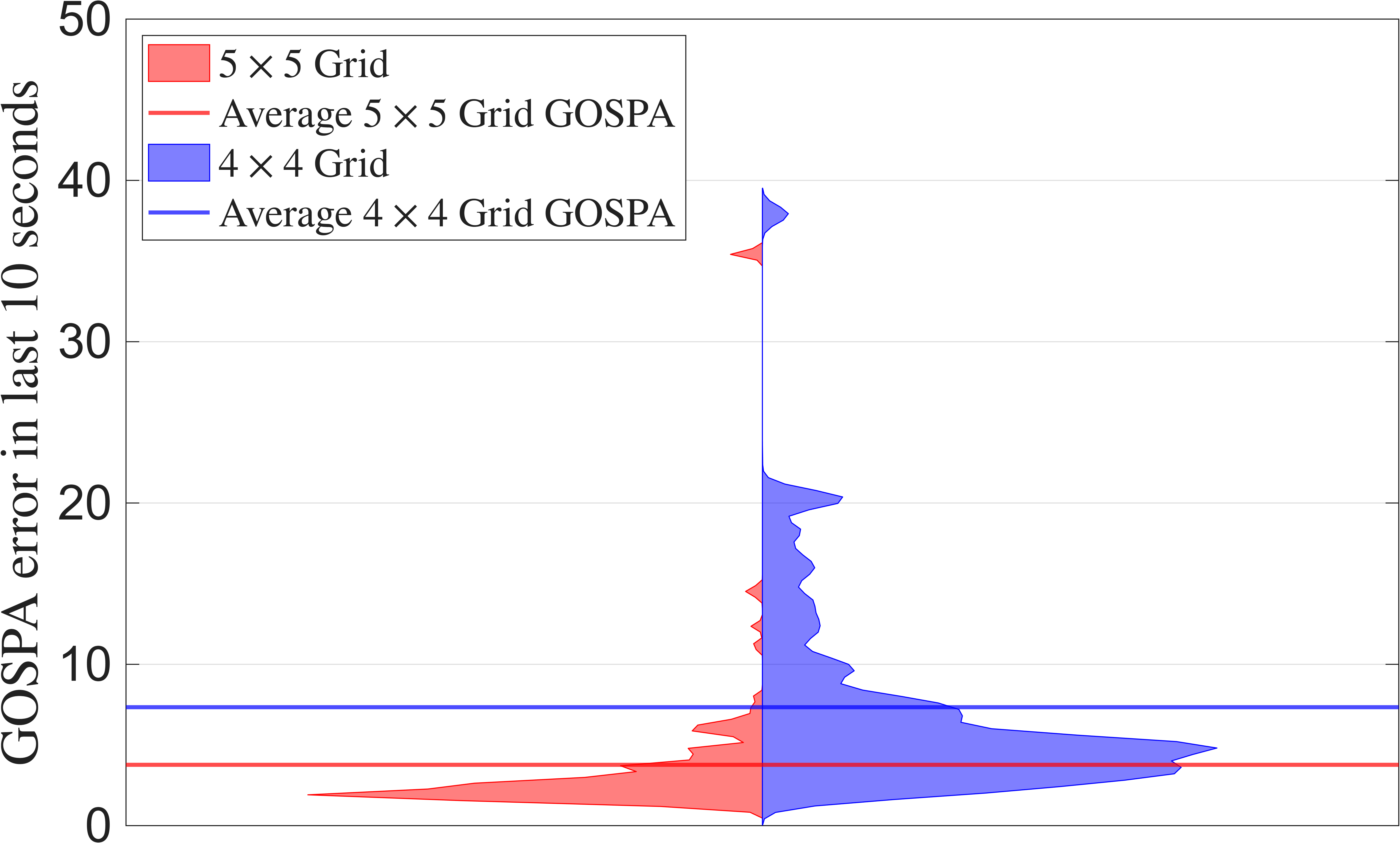}
        \caption{Static grid of sensors and $M_{\max}=2$}
        \label{fig:violin_gospaGridNs2}
    \end{subfigure}
    \caption{Voilin plot of the GOSPA error in the last 10 seconds across 80 Monte Carlo runs.}
    \label{fig:violin_MCrunsGOSPA}
\end{figure*}

\tb{\begin{table}[!h]
    \small\sf\centering
    \caption{Monte Carlo study source configurations}
    \label{tab:MCSourceConfig}
    \begin{tabular}{|c|c|c|c|c|}
        \toprule
          \textbf{Config. No.} & $\vect{s}_1$ (m) & $\vect{s}_2$ (m) & $Q_1$ ($g/s$)  & $Q_2$ ($g/s$) \\
         \midrule
         $\mathbb{GT}_1$ & (15,40) & (40,30) & 6 & 7\\
         $\mathbb{GT}_2$ & (40,15) &(30,45) & 5 & 9\\
         $\mathbb{GT}_3$ & (5,30) & (30,10) & 7 & 7\\
         $\mathbb{GT}_4$ & (10,25) & (40,25) & 9 & 4\\
         $\mathbb{GT}_5$ & (20,30) & (40,30) & 4 & 6\\
         \bottomrule
    \end{tabular}
\end{table}}

\subsection{Monte Carlo Study}
\tb{In this subsection, we present a Monte Carlo (MC) study comparing the MR-MSTE performance of the proposed wind-aware Coverage control (WCC) based coordination policy to traditional coverage control in different settings. A traditional coverage control (TCC) based MR-MSTE is a special case of the proposed strategy for $\alpha=0$. We also include a comparison with two static sensor network settings, - $4\times4$ and $5\times5$ grid.}

\tb{The MC study consist of five two-source configurations, as listed in Table \ref{tab:MCSourceConfig} and two wind directions $\psi = \{-\pi/2,\pi/2\}$. The two wind directions essentially refer to an upwind or downwind start scenarios. Each combination of source configuration and wind direction is simulated for 8 runs, resulting in an overall 80 runs for both WCC and TCC based coverage control. Moreover, $|\alpha|$ is chosen to be 0.75 and the remaining simulation parameters are the same as the illustrative case. Furthermore, each MC run is simulated for a maximum of 200 seconds. The results of the MC study, presented in Figs. \ref{fig:MCrunsGOSPA} and \ref{fig:violin_MCrunsGOSPA}, consist of 
\begin{enumerate}
    \item[(a)] 
    Estimation error profile computed using an GOSPA distance metric between the estimated source terms and the ground truth computed as per Def. \ref{def:GOSPA}. 
    \item[(b)] Aggregated GOSPA in the last 10 seconds of across all MC runs to quantify the terminal estimation error and consequently the success likelihood of the proposed strategy. 
\end{enumerate}}

\tb{The MC study is performed using three configurations of robot numbers and the maximum number of sources assumption each. Each column in  Figs. \ref{fig:MCrunsGOSPA} and \ref{fig:violin_MCrunsGOSPA} denote MC study corresponding to a different maximum source number hypothesis, that is, columns 1, 2 and 3 correspond to $M_{\max}$ of 4, 3 and 2, respectively. Furthermore, rows 1, 2 and 3 of the aforementioned figures denote MC study results for robot number $n=6,~5,$ and 4, respectively. Additionally, the final row in each figure consist of the MC study corresponding filter performance for static sensor networks. }

\tb{It can be seen from Fig. \ref{fig:MCrunsGOSPA} that the proposed WCC based MR-MSTE algorithm results in lower GOSPA-error in estimating the source term as compared to the TCC based MR-MSTE. Additionally, the proposed strategy reduces the overall GOSPA error at a quicker rate than the TCC based strategy. It is important to note that such a trend is observed for all configurations of varying number of robots and maximum source number hypothesis. Moreover, it can also be noted from the first three rows of Fig. \ref{fig:MCrunsGOSPA} that as the number of robots decrease the magnitude of overall GOSPA error with time also increases. Thus, once can conclude that as the number of robots increase, the terminal GOSPA error can be reduced as one would expect intuitively. In general, it can also be noted that when the maximum no. of source hypothesis $M_{\max}$ is chosen to be closer to the actual number of sources, the overall estimation performance improves. This can be interpreted as providing the algorithm with an well-informed prior, would result in lower chances of false detection of a source and consequently a lower GOSPA error. Moreover, it can also be noted that with a well-informed prior on the maximum number of sources in the environment, the performance advantages of the proposed strategy over the traditional coverage control reduces. Thus, one can conclude that the WCC based strategy outperformed the TCC based strategy throughout the search window. }

\tb{From fig. \ref{fig:violin_MCrunsGOSPA}, it can be observed that, under the WCC based strategy, the PDF of the GOSPA error in the last 10 seconds across all MC run was narrower and resulted in a lower average value as compared to the TCC based strategy. This metric can be used to quantify the likelihood of the terminal GOSPA error within the last 10 seconds of the mission, and consequently the likelihood of a run being successful. As the WCC based strategy resulted in a PDF weighted relative closer to the origin as compared to TCC based strategy, we can conclude that the chances of successfully estimating the number of leaks and accurately characterize their source term are larger under the proposed wind aware strategy.}

\tb{We also evaluate the performance of the proposed multi-source filter in Section \ref{sec:BayesianEstimation} using multiple static sensor networks with different spatial sparsity. The last rows of Figs. \ref{fig:MCrunsGOSPA} and \ref{fig:violin_MCrunsGOSPA} demonstrate the estimation performance with a $4\times4$ (16 sensors) and $5\times5$ (25 sensors) sensor grid. 
It is important to note from comparing their first and last rows of Figs. \ref{fig:MCrunsGOSPA} and \ref{fig:violin_MCrunsGOSPA} that the sensing performance achieved by an intelligently placed dynamic sensor network with 6 sensors performed similar to the performance with a static sensor network with 25 sensor measurements per sampling instant. This emphasises the advantage of a dynamic sensor network over a static sensor network.  }

\tb{While the numerical MC study demonstrates the effectiveness and advantage of the proposed algorithm, we further evaluate the robustness and applicability of the overall MRMSTE framework in a more realistic experimental setting as presented in the next subsection.}

\section{Experimental Study}
The experimental study was conducted in a confined space, using three Turtlebot3-burger robots with a VICON motion tracking system for robot positioning. Each robot was equipped with CozIR-A sensors in front of them, as shown in Fig. \ref{fig:mobilesetup}, to sense the concentration of CO$_2$ gas. The release mechanism consisted of a pressurised gas bottle (see fig. \ref{fig:Gas_Bottle}) connected to a multi-stage regulator (see fig. \ref{fig:Gas_Regulator}) that ensured near constant release pressure at its outlet, which was further split into two hose outlets to create two release sources in the search domain as seen in Fig. \ref{fig:expsetupImage}. The CozIR-A sensors provide an absolute concentration reading of CO$_2$ at the sampling location in parts-per-million (PPM), which is converted to milligram per cubic meter using as $\scalnot{z}{k}{:}(mg/m^3) = 1.8\times\scalnot{z}{k}{:}(\text{PPM})$. However, this sensor has a slow response time to any change of the gas concentration, which presents a challenges to the experiment study. 

\begin{figure}[ht]
    \centering
    \includegraphics[width=0.5\linewidth]{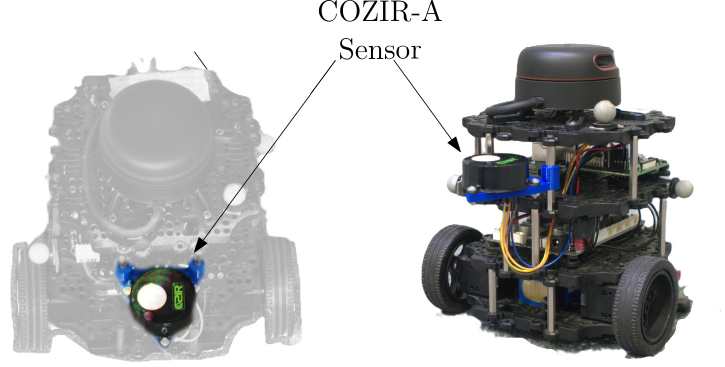}
    \caption{Turtlebot3-Burger setup with COZIR-AX-1 sensor.}
    \label{fig:mobilesetup}
\end{figure}

\begin{figure}[ht]
    \centering
    \captionsetup[subfigure]{justification=centering}
    \begin{subfigure}[t]{0.245\linewidth}\centering
        \includegraphics[width=0.9\linewidth]{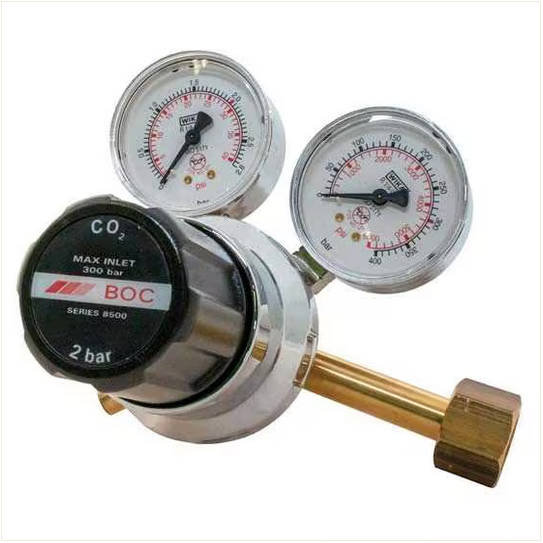}
        \caption{}
        \label{fig:Gas_Regulator}
    \end{subfigure}
    \begin{subfigure}[t]{0.245\linewidth}\centering
        \includegraphics[width=0.9\linewidth]{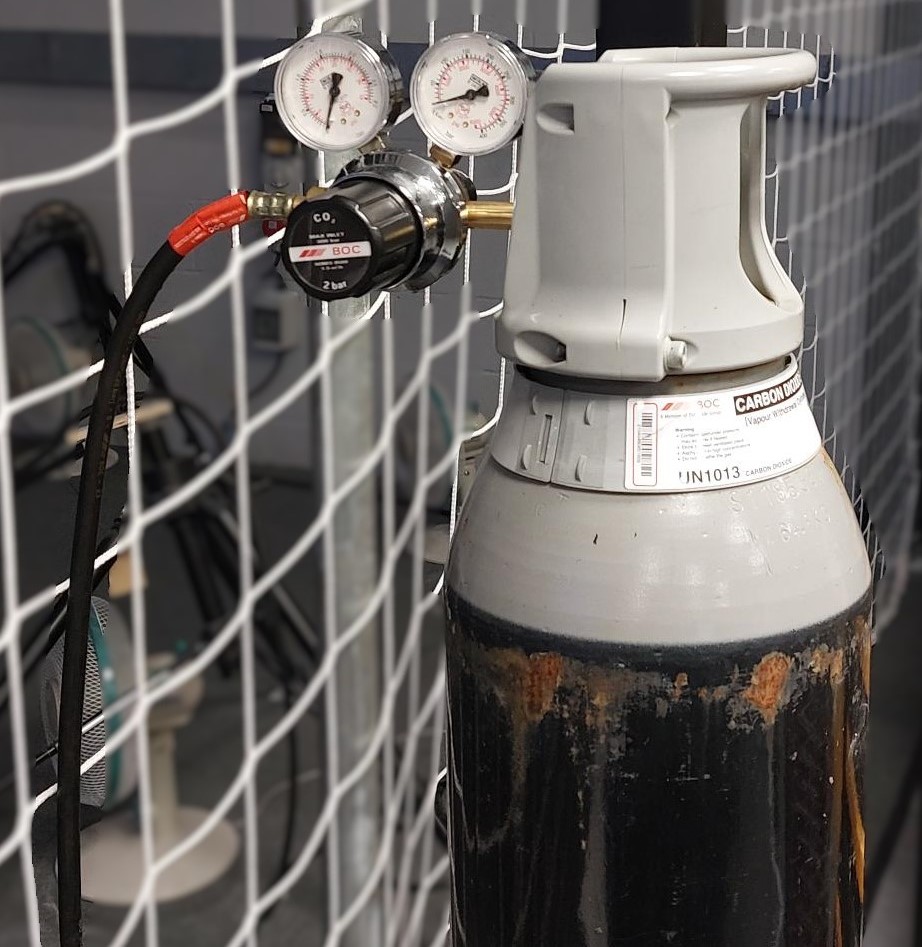}
        \caption{}
        \label{fig:Gas_Bottle}
    \end{subfigure}
    \begin{subfigure}[t]{0.45\linewidth}
        \includegraphics[width=\linewidth]{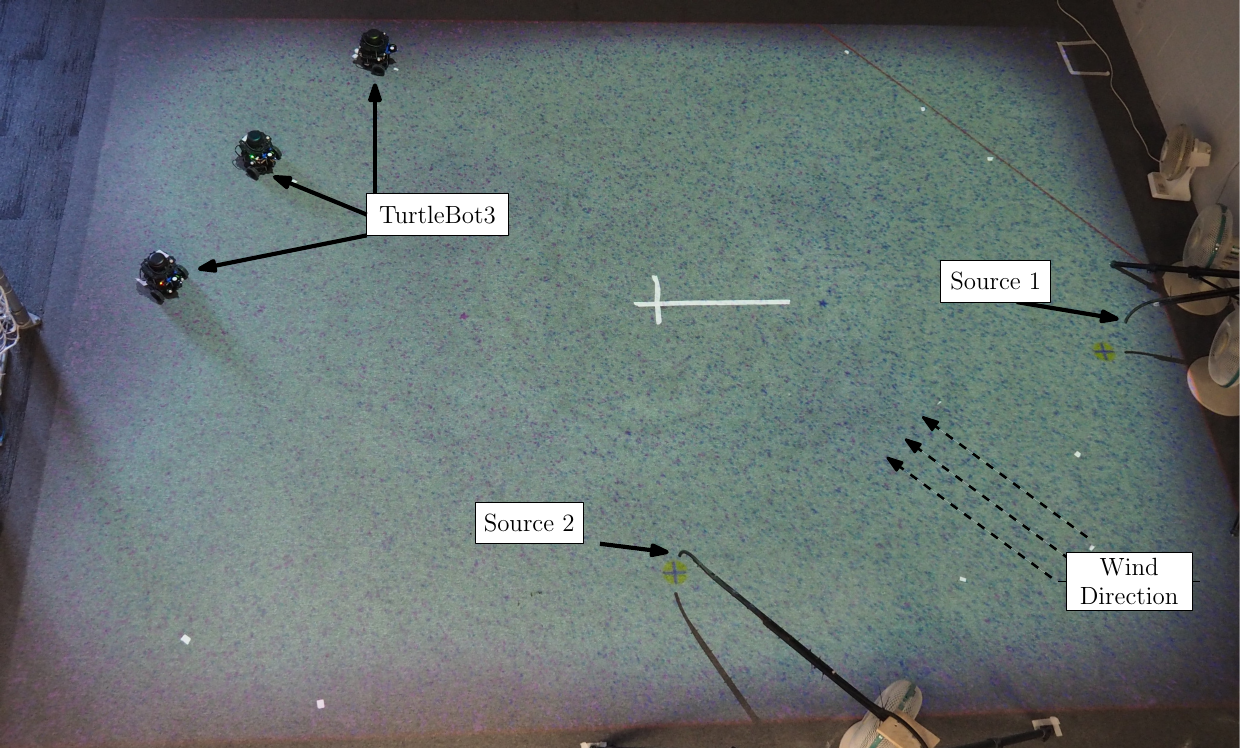}
        \caption{}
        \label{fig:expsetupImage}
    \end{subfigure}
    \caption{Multi-source gas release setup with \cite{BOCRegulator} in (\subref{fig:Gas_Regulator}) attached to a gas cylinder shown in (\subref{fig:Gas_Bottle}). The search area containing the TurtleBot3 robots and two sources is shown in (\subref{fig:expsetupImage}). }
    \label{fig:GasReleaseSetup}
\end{figure}

The experimental setup and system architecture is given Fig. \ref{fig:sysdiagram}. While the CO\textsubscript{2} gas cylinder provides the gas sources via two outlets, several small electric fans are used to encourage the plume creation in the direction from the bottom right of the environment to the top left. Each TurtleBot is equipped with a CozIR CO\textsubscript{2} sensor providing gas concentration measurements at 1 Hz, which are sent via ROS to an offboard PC running the filter and planning algorithms. A VICON motion capture system provides states for each TurtleBot3 robot, which are subsequently controlled via linear and angular speed commands issued from the offboard PC. Real-time visualisation of historical TurtleBot3 trajectories and source location estimates is provided by an overhead projector that is spatially calibrated with the VICON system, as can be seen in the figure.

\begin{figure*}[htb]
    \centering
    \includegraphics[width=0.8\linewidth]{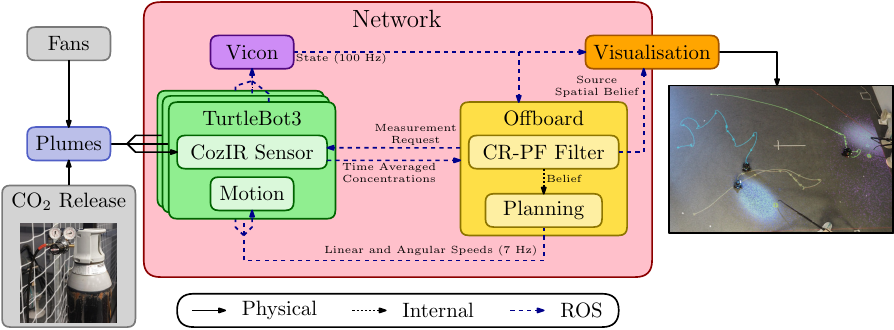}
    \caption{System diagram for the experimental setup.}
    \label{fig:sysdiagram}
\end{figure*} 
 
 For the experimental study, a strong prior is assumed about the number of sources (i.e., $M_{\mathrm{max}}=2$), the prevailing wind and diffusion parameters. An illustrative run of the multi-source search and estimation is shown in Fig. \ref{fig:sampleExpsearch}. It can be seen from Fig. \ref{fig:sampleExpsearch} that the robots are able to localise and distinguish between the two sources as the search progresses.  

 \begin{figure*}[htb]
     \centering
     \captionsetup[subfigure]{justification=centering}
     \begin{subfigure}{0.45\linewidth}
          \includegraphics[width=0.9\linewidth]{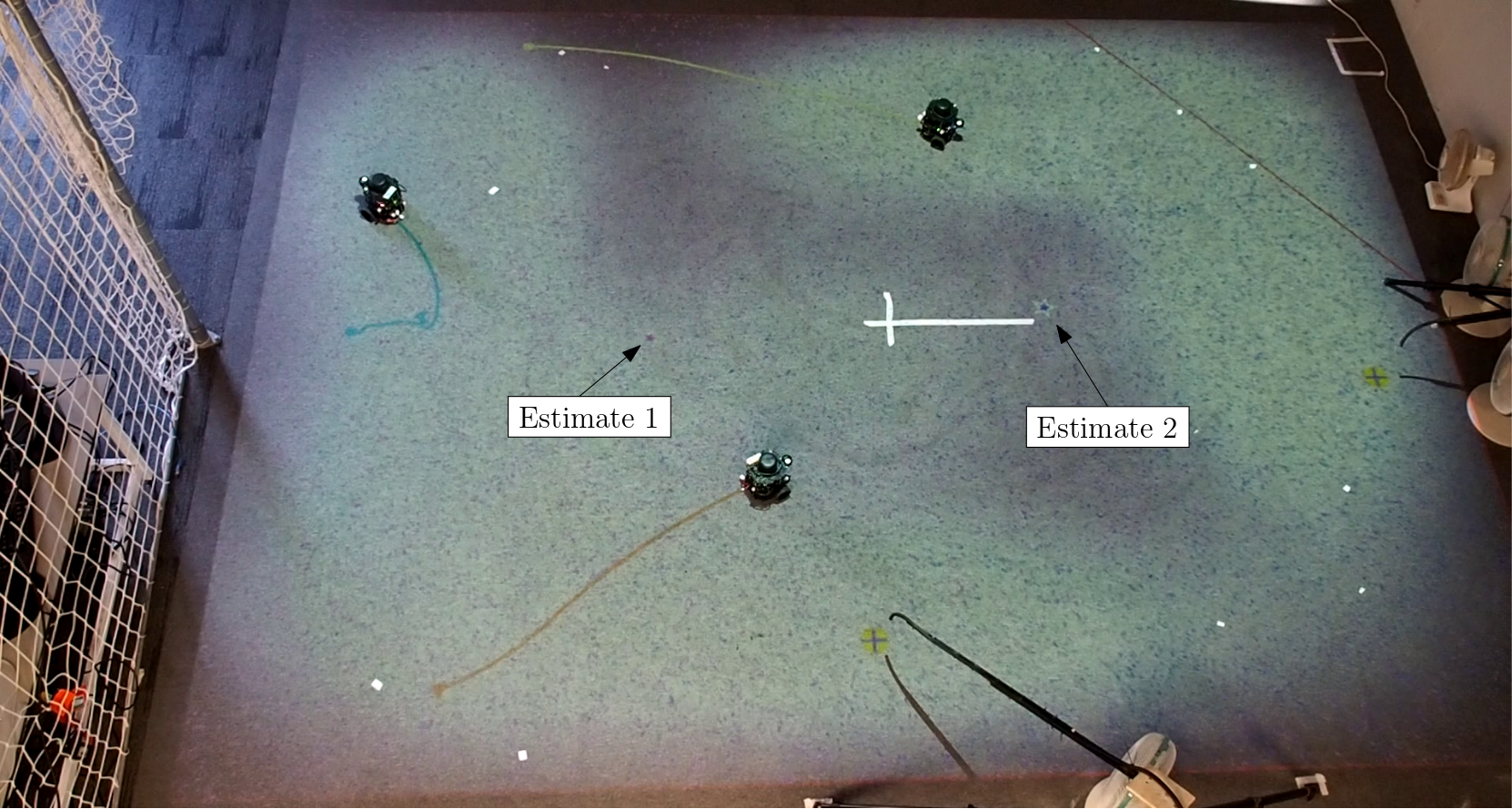}
          \caption{After 10 seconds.}
          \label{fig:ExpIllustrativeRunT1}
     \end{subfigure}
     \begin{subfigure}{0.45\linewidth}
          \includegraphics[width=0.9\linewidth]{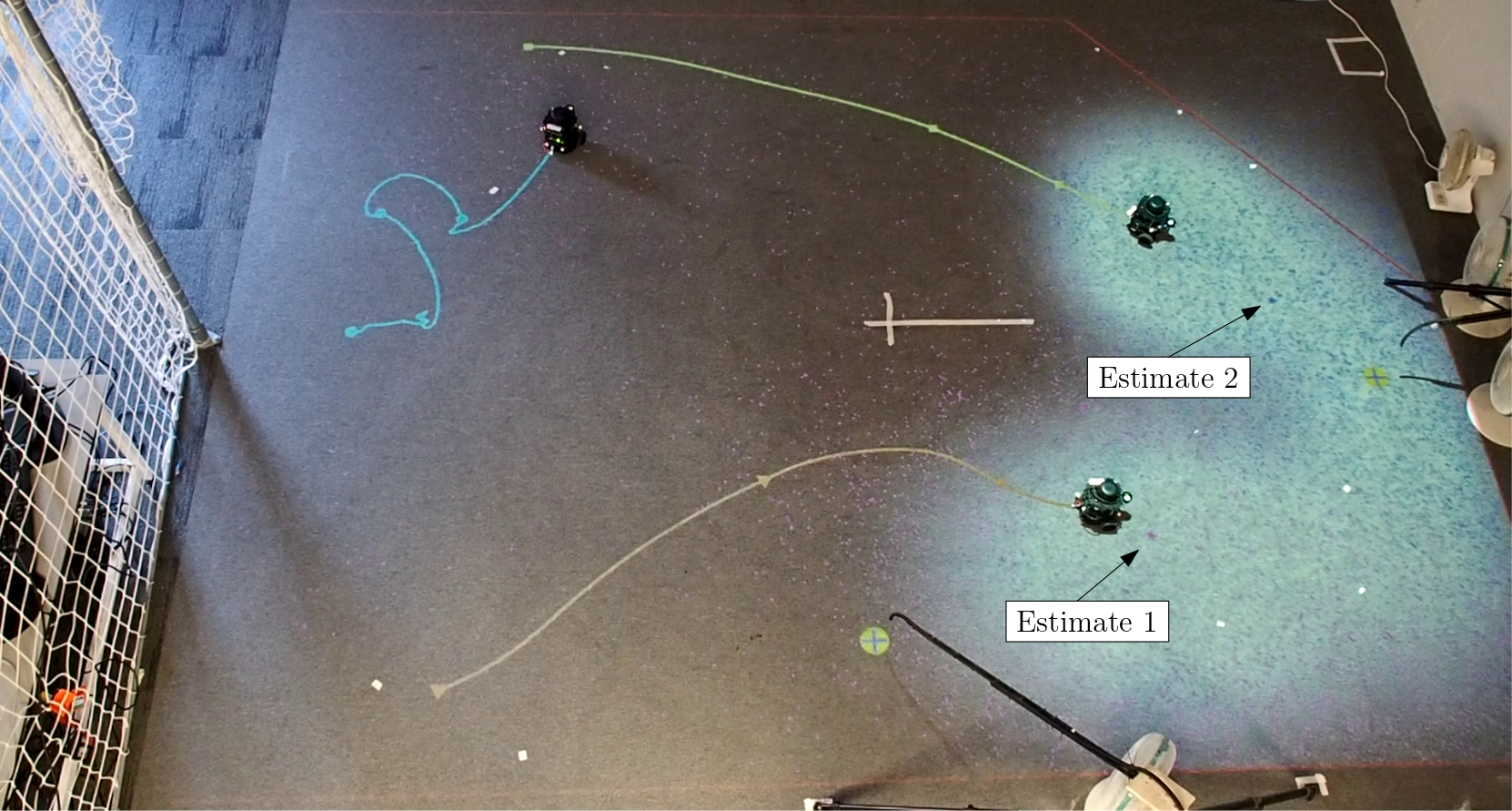}
          \caption{After 76 seconds.}
          \label{fig:ExpIllustrativeRunT2}
     \end{subfigure}
     \begin{subfigure}{0.45\linewidth}
          \includegraphics[width=0.9\linewidth]{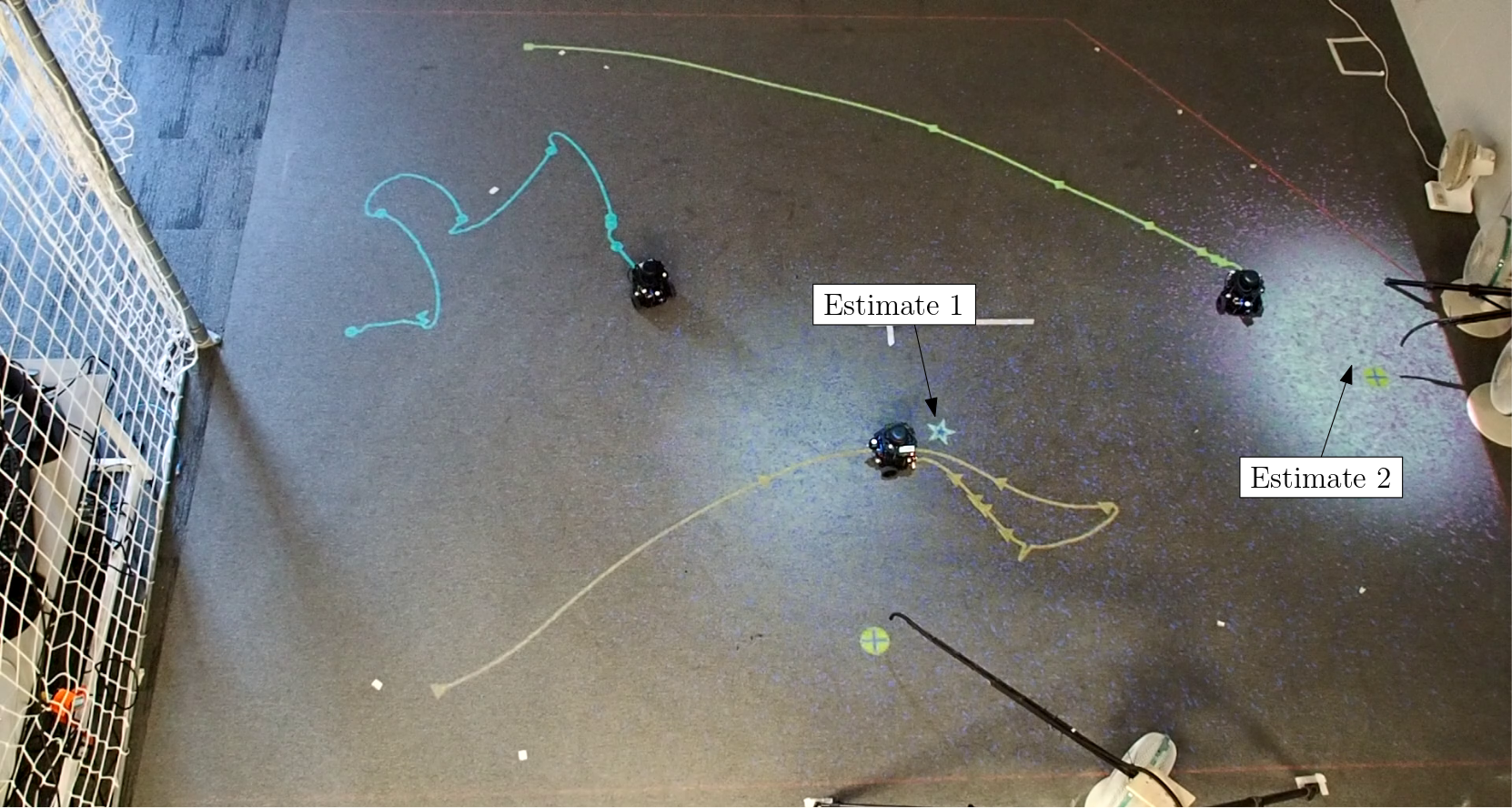}
          \caption{After 206 seconds.}
          \label{fig:ExpIllustrativeRunT3}
     \end{subfigure}
     \begin{subfigure}{0.45\linewidth}
          \includegraphics[width=0.9\linewidth]{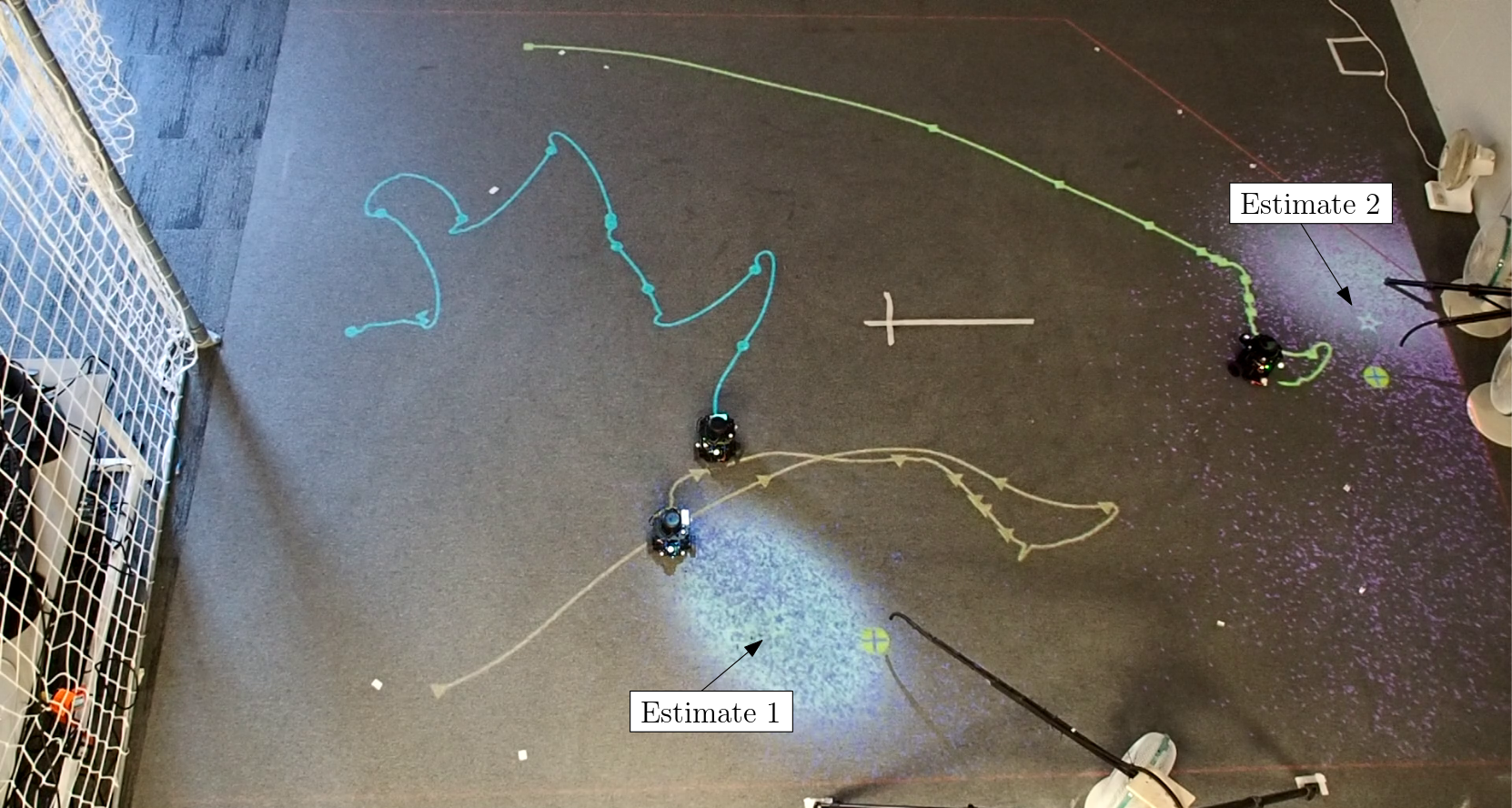}
          \caption{After 300 seconds.}
          \label{fig:ExpIllustrativeRunT4}
     \end{subfigure}
     \caption{Sample run of the proposed wind-aware coverage control based MR-MSTE algorithm with $\alpha=-0.75$. The reader can refer here for a \href{https://youtu.be/Wxf5Jsr_dxI}{video} demonstrating the search. }
     \label{fig:sampleExpsearch}
 \end{figure*}

In order to analyse the performance of the proposed path planning strategies against traditional coverage control, we conducted 5 experimental run per algorithm with three TurtleBot3 robots starting their search at the bottom left corner of the search domain. The search began 1 minute after the gas release was started to ensure sufficient dispersion of the gas in the search domain. The search for each path planner was run for a total of 300 seconds, with an average of 10 seconds per sampling instant. This stop-and-sample interval allows the robot to take a time-averaged measurement to account for the slow response of the sensor. Moreover, it was ensured that the base CO$_2$ concentrations were back to the nominal amount of 550 PPM before every run. Furthermore, the $z_{thr}$ was heuristically chosen at 950 PPM to account for sensor dynamics, noise and minimize miss-detection cases. The results from the experiments corresponding to each path planner are tabulated in Table \ref{tab:expresults}. The performance metrics to compare different path planners focus on the accuracy and precision of source localisation as the ground truth associated with release rates are difficult to obtain. 
In the experimental study, we will be focusing on three performance indexes, i.e., (a) success ratio (SR) defined as the final localisational GOSPA with $\sigma_{th} = 0.85 \mathrm{m}$, (b) final localisation error, and (c) final localisation uncertainty.  
\begin{table}[htb]
    \centering
    \caption{Results from the 5-Run experimental study.}
    \label{tab:expresults}
    \begin{tabular}{|c|c|c|c|}
        \toprule
          & Success & \multicolumn{2}{c|}{Final Localisation} \\\cline{3-4}
          & Ratio (\%) & Error (m) & Uncertainty (m) \\
         \midrule
         WCC & 80 & 0.74 & 0.43 \\
         TCC & 60 & 0.91 & 1.42 \\
         \bottomrule
    \end{tabular}
\end{table} 

The final localization error with respect to both source in each run for the WCC and TCC are plotted in Fig. \ref{fig:PolarPlot}. It can be noted from Fig. \ref{fig:PolarPlotWCC} that the average localisation error (ALE) of WCC is within a radius of 0.85m for runs 2 to 5, resulting in a 80\% SR. Whereas, for TCC, the ALE is below $\sigma_{th}$ for runs 2,4 and 5 resulting in a 60\% SR. The results are further summarised and tabulated in Table \ref{tab:expresults}. It can be seen from Table \ref{tab:expresults} that the proposed wind-aware strategy (with $\alpha=-0.75$) algorithm resulted in 17 cm lower average localisation error as compared the TCC algorithm across 5 runs. Additionally, WCC also resulted in \emph{significantly} lower estimate uncertainty as compared to TCC algorithms as can be noted from the fourth column of Table \ref{tab:expresults}. Therefore, in the light of the above discussion, it can be concluded that the proposed multi-robot multi-source term estimation strategy provides significant advantages in localisation of multiple sources as compared to the traditional coverage control.
\begin{figure}
    \centering
    \captionsetup[subfigure]{justification=centering}
    \begin{subfigure}{0.4\linewidth}
        \includegraphics[width=\linewidth]{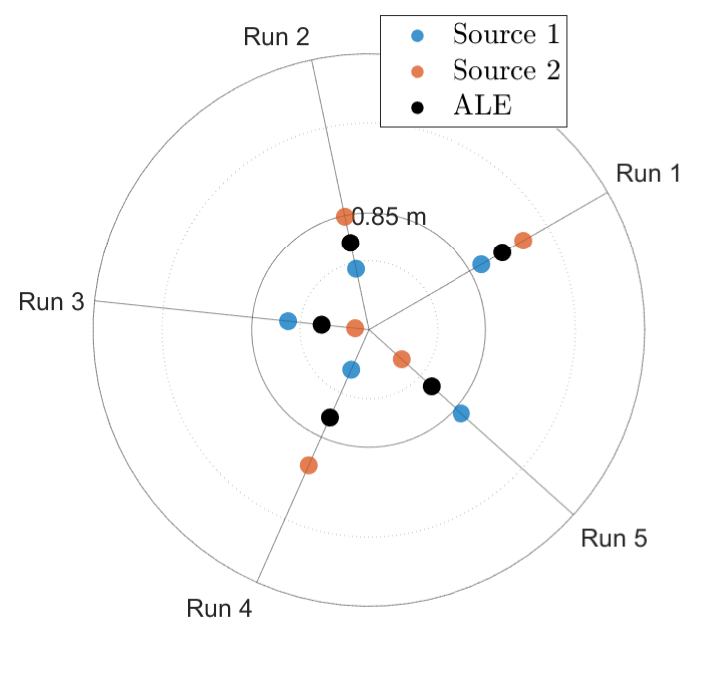}
        \caption{Runs of WCC}
        \label{fig:PolarPlotWCC}
    \end{subfigure}
    \begin{subfigure}{0.4\linewidth}
        \includegraphics[width=\linewidth]{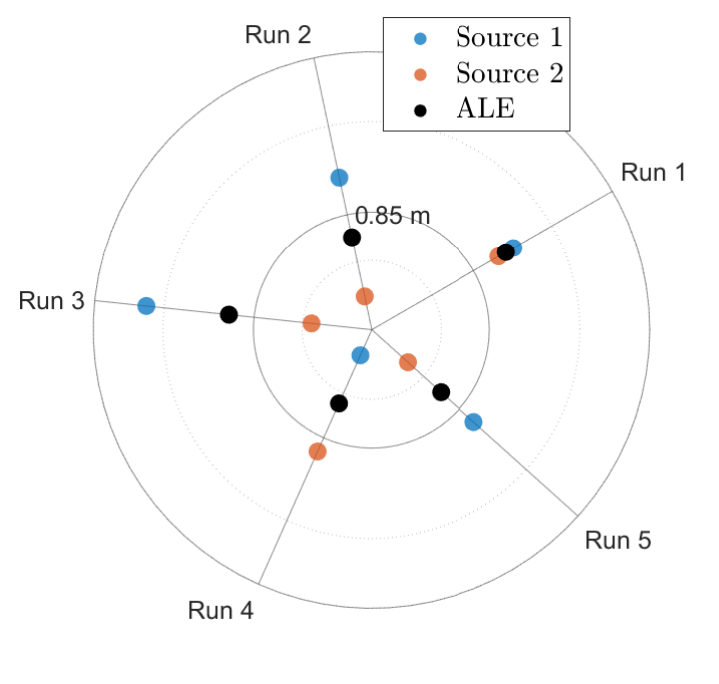}
        \caption{Runs of TCC}
        \label{fig:PolarPlotTCC}
    \end{subfigure}
    \caption{Final estimation error corresponding to both source along with the average localisation error (ALE) per run for different path planning algorithms within 300 seconds.}
    \label{fig:PolarPlot}
\end{figure}

\section{Conclusion}
\tb{In this paper, we developed, evaluated and demonstrated a multi-robot planning and estimation framework for characterizing multiple airborne release sources using sparse mobile sensing. The proposed wind-aware coverage control–based planning enables robots to collaboratively sample gas concentrations, allowing the Bayesian filter to progressively resolve multiple sources whose cardinality and parameters are initially unknown. }

\tb{A key feature of the approach is the use of a dedicated particle filter that exploits physics-informed state transitions and permutation symmetry of JMPD to extract meaningful multi-source estimates from superposition-based measurements. The proposed wind-aware coverage control strategy further enhances estimation performance by guiding sensing platforms toward upwind sampling locations while preserving domain coverage. The effectiveness of the proposed framework was demonstrated through both simulation and real-world experiments.  Comparative Monte Carlo studies showed that the proposed hybrid particle filter enables stable tracking of multiple source hypotheses while maintaining clear source separation and that the wind-aware spatial coordination leads to improved accuracy and efficiency over traditional coverage control and static sensor deployments. Experimental results using controlled CO$_2$ releases and TurtleBot3 platforms confirmed that the framework can operate reliably under realistic sensing noise and environmental uncertainty, supporting its applicability beyond purely simulated settings.}

\tb{Overall, this work demonstrates that combining physics-informed Bayesian inference with coverage-based multi-robot coordination provides a practical and scalable solution to multi-source term estimation. The results highlight the robustness and readiness of mobile robots in complex and time-critical gas release sensing scenarios, paving the way for wider adoption in environmental monitoring, emergency response, and industrial safety related applications. Future work will focus on extending the framework to larger and more dynamic environments, incorporating adaptive estimation of environmental parameters, and developing distributed implementations to support fully autonomous multi-robot deployments.}

\bibliographystyle{ieeetr}
\bibliography{References.bib}

\end{document}